\documentclass[11pt,a4paper]{article}

\usepackage[utf8]{inputenc}
\usepackage[T1]{fontenc}
\usepackage{mathptmx}
\usepackage{amsmath,amssymb,amsthm}
\usepackage{booktabs,array,tabularx}
\usepackage{enumitem}
\usepackage{algorithm}
\usepackage{algpseudocode}
\usepackage{listings}
\usepackage{xcolor}
\usepackage{xspace}
\usepackage{longtable}
\usepackage{tikz}
\usetikzlibrary{arrows.meta,positioning}
\usepackage{microtype}
\usepackage[a4paper,left=2.5cm,right=2.5cm,top=2.7cm,bottom=2.7cm]{geometry}
\usepackage[hidelinks]{hyperref}
\usepackage[font=small,labelfont=bf]{caption}

\newcolumntype{P}[1]{>{\raggedright\arraybackslash}p{#1}}
\setlist{nosep,leftmargin=*}
\mathchardef\mathcommaorig="613B
\mathcode`\,="8000
\begingroup\catcode`\,=\active
\gdef,{\mathcommaorig\penalty0\hspace{0pt}}
\endgroup

\theoremstyle{plain}
\newtheorem{theorem}{Theorem}[section]
\newtheorem{lemma}[theorem]{Lemma}
\newtheorem{proposition}[theorem]{Proposition}
\newtheorem{corollary}[theorem]{Corollary}
\theoremstyle{definition}
\newtheorem{definition}[theorem]{Definition}
\newtheorem{example}[theorem]{Example}
\theoremstyle{remark}
\newtheorem{remark}[theorem]{Remark}

\newcommand{\SROIQ}{\mathcal{SROIQ}}
\newcommand{\TB}{\mathcal{T}}
\newcommand{\AB}{\mathcal{A}}
\newcommand{\KB}{\mathcal{K}}
\newcommand{\RS}{\mathcal{R}}

\newcommand{\DDA}{\textsc{DDA}}
\newcommand{\APT}{\textsc{APT}}
\newcommand{\Ind}{\mathrm{Ind}}
\newcommand{\Step}{\mathrm{Step}}
\newcommand{\Dec}{\mathrm{Dec}}
\newcommand{\Sugg}{\mathrm{Suggest}}

\newcommand{\Brave}{\mathrm{Brave}}
\newcommand{\Proj}{\mathrm{Proj}}
\newcommand{\op}[1]{\textsc{#1}}
\newcommand{\nentails}{\not\models}
\newcommand{\sn}[1]{\textit{#1}}   

\newcommand{\TurnKey}{\sn{Act\_TurnKey}\xspace}
\newcommand{\OperateHandle}{\sn{Act\_OperateHandle}\xspace}
\newcommand{\ImportWheat}{\sn{ImportWheatAction}\xspace}
\newcommand{\ImportCar}{\sn{ImportCarAction}\xspace}
\newcommand{\Renounce}{\sn{RenounceExemptionAction}\xspace}
\newcommand{\IssueBadge}{\sn{Act\_IssueBadge}\xspace}
\newcommand{\EnterZone}{\sn{Act\_EnterZone}\xspace}
\newcommand{\Certify}{\sn{Act\_Certify}\xspace}
\newcommand{\Train}{\sn{Act\_Train}\xspace}
\newcommand{\IssuePass}{\sn{Act\_IssuePass}\xspace}
\newcommand{\EnterLab}{\sn{Act\_EnterLab}\xspace}

\algrenewcommand\algorithmicrequire{\textbf{Input:}}
\algrenewcommand\algorithmicensure{\textbf{Output:}}
\algrenewcommand{\alglinenumber}[1]{\scriptsize #1}

\definecolor{backcolour}{rgb}{0.97,0.97,0.95}
\definecolor{framegray}{rgb}{0.6,0.6,0.6}
\lstdefinestyle{plainbox}{basicstyle=\ttfamily\scriptsize,backgroundcolor=\color{backcolour},
  showstringspaces=false,breaklines=true,frame=single,rulecolor=\color{framegray},
  xleftmargin=0pt,xrightmargin=0pt,columns=flexible,keepspaces=true,aboveskip=6pt,belowskip=6pt}
\title{\bfseries\Large Alternating Abduction and Planning over\\
Temporal Dynamic Knowledge Graphs\thanks{This research was in part supported by Noor Avaran Jelvehaye Maanaei Najm Co. and a grant from IPM (No.~1404030025).}}

\author{Amir Hossein Sharafi$^{1,2,3}$ \and Alireza Shahbazi$^{1}$\\[6pt]
\small $^{1}$Noor Avaran Jelvehaye Maanaei Najm Co., Tehran, Iran\\
\small $^{2}$Department of Mathematics, Tafresh University, 39518-79611, Tafresh, Iran\\
\small $^{3}$School of Mathematics, Institute for Research in Fundamental Sciences (IPM),\\
\small P.\,O.\ Box 19395-5746, Tehran, Iran\\[2pt]
\small \texttt{am\_sharafi@tafreshu.ac.ir}, \texttt{shahbazi@najm.ir}}
\date{}
\renewcommand{\thefootnote}{\fnsymbol{footnote}}

\begin{document}
\maketitle
\renewcommand{\thefootnote}{\arabic{footnote}}\setcounter{footnote}{0}
\thispagestyle{plain}

\begin{abstract}
Knowledge graphs change over time, but the effects of actions on them are only partly stated: the ontology entails further consequences, facts that an action needs are often unknown, and an action may make the graph inconsistent. We study these problems for Temporal Dynamic Knowledge Graphs (TDKGs), that is, $\SROIQ$ knowledge bases whose actions are rules with delete/insert templates and whose events lie on an intrinsic timeline. We present three tools: temporal projection, which plays the events along the timeline and is the exact test of plans; \DDA{} (Direct-Derivation Abduction), which explains an observation by assuming its negation and negating what the reasoner newly derives; and \APT{} (Abduction and Planning over TDKGs), which plans backwards from a goal, assumes the facts that the plan needs but nobody stated, and lets \DDA{} explain the assumptions, each explanation becoming a new goal for a sub-plan, up to a depth bound. The same machinery synthesises action sequences that end in a chosen contradiction. The tools use a $\SROIQ$ reasoner only through a small contract, so they depend neither on a particular reasoner nor on a particular data structure. They are designed for a reasoner that processes added and removed axioms incrementally, from which large ontologies should profit most; a reasoner that also reports the consequences of an update or the derivation of an axiom makes them cheaper, and one whose copies are layers over the parent state needs less memory, and with a reasoner that only checks consistency and entailment they remain correct. Change is treated without modal operators, on one linear timeline and without possible worlds, so the decidability of projection follows from that of $\SROIQ$. We prove that the indirect effects of an action are exactly the entailments of the updated state; that \DDA{} is sound, terminates and finds every single-assertion explanation; that \APT{} is sound and terminates; that bounded planning is $2$\textsc{NExpTime}-hard and in $3$\textsc{ExpTime}; and that it has no compact translation into classical planning. We compare the tools analytically with hitting-set abduction and classical planning, compare \DDA{} with the abduction solver CATS on shared problems, and run the tools on six benchmarks. Five reasoners (HermiT, Konclude, Pellet, Openllet and JFact) return identical results. The incremental modes that we could test answered some checks wrongly after removals of assertions, and parallel calls shorten the run time only when a single call is expensive; both findings concern the tested reasoners and small inputs.
\end{abstract}


\section{Introduction}
\label{sec:intro}

Knowledge graphs on the Semantic Web are written in OWL~2~DL, whose logical basis is the description logic $\SROIQ$ \cite{horrocks2006sroiq,motik2012owl2}. Real graphs are not static: facts change, actions have consequences that the ontology should propagate, and what is not stated is \emph{unknown} rather than false. We call a graph together with a set of rules that define actions, and with events placed on a timeline, a \emph{Temporal Dynamic Knowledge Graph} (TDKG) \cite{shahbazi2024strategy}. Three questions arise over such graphs.

\begin{itemize}
\item \emph{What happens?} If the events of the timeline are executed one after another, which state does the graph reach? An action may have far more consequences than its rules say, because the ontology entails more (an unlocked, installed lock on a secure door lets the door be opened). Writing all such \emph{indirect effects} by hand is the classical \emph{ramification problem} \cite{thielscher1997ramification,lin1995embracing,mccain1995causal}.
\item \emph{What is missing?} Under the open-world assumption, a plan may need facts that nobody has stated, or even individuals that do not exist yet. We want the planner to say \emph{which} facts it had to assume, to explain them from the ontology, and, where possible, to \emph{achieve} them by further actions.
\item \emph{What goes wrong?} A sequence of actions can make the graph inconsistent. Finding such sequences deliberately is a way to debug the domain model.
\end{itemize}

\paragraph{Approach.} Our tools are built on one idea: \emph{let the reasoner do the logic, and keep the tools thin}. They never look inside the knowledge base. They create copies of it, add and remove assertions, ask whether it is consistent, and ask whether it entails an assertion. Section~\ref{sec:backend} states which operations are needed, so the tools are independent of any particular reasoner and of any particular data structure for the graph: every representation and every $\SROIQ$ reasoner that supplies these operations qualifies. The tools are designed for an \emph{incremental} reasoner, which processes added and removed axioms without starting over; on large ontologies, where a single check is expensive, this is the ability from which they should profit most. A reasoner that reports what an update newly derives, or that stores how it derived each axiom (a \emph{derivation record}), makes them cheaper still. With a plainer reasoner everything remains correct, only slower. We do not tie the tools to one query language either: SPARQL is the realisation we use, and only the abilities listed in Section~\ref{sec:backend} are needed.

\begin{enumerate}[label=\textbf{T\arabic*.}]
\item \emph{Temporal projection} answers the first question. A TDKG carries an intrinsic timeline, the sequence of its time instants. Projection executes every action that is triggered along the timeline, one instant after another, through the reasoner, and returns the whole trajectory of states and in particular the \emph{last state}. It is useful on its own: it predicts what a TDKG will look like, and it finds the first instant at which the graph becomes inconsistent. It is also the exact test of plans: a plan is scheduled into the timeline as new events, and the plan is valid iff the projection of the result runs through and reaches the goal.
\item \DDA{} (\emph{Direct-Derivation Abduction}) answers the second question for \emph{axioms}. To explain an observation $o$, add $\neg o$ to the knowledge base, let the reasoner report the consequences that are newly derived, and negate each of them. These negated consequences are the \emph{suggestions} for $o$, and every suggestion, together with the knowledge base, entails $o$ (Lemma~\ref{lem:sugg}); no list of candidate hypotheses is needed.
\item \APT{} (\emph{Abduction and Planning over TDKGs}) answers the second and third question together, by \emph{alternating between planning and abduction}. Abduction is about assertions: what should be true so that something else follows. Planning is about actions: what to do so that something becomes true. Neither can replace the other, and \APT{} lets each call the other (Figure~\ref{fig:loop}):
\begin{enumerate}[label=(\roman*),nosep]
\item \emph{plan:} regress the goal through the rules to a sequence of actions; the preconditions that no action produces and that do not hold are \emph{assumed}, say $\Gamma_0$;
\item \emph{abduce:} ask \DDA{} for explanations of the assumptions; each explanation $H$ is an alternative set of assertions that would make $\Gamma_0$ true;
\item \emph{plan again:} treat each $H$ as a new goal and find a \emph{sub-plan} for it, scheduled before the original actions; the sub-plan may again need assumptions $\Gamma_1$;
\item \emph{abduce again, plan again, \dots,} until nothing is left to assume, nothing can be explained, or a depth bound $R$ is reached; the bound makes the recursion terminate;
\item \emph{test:} schedule the assembled plan and decide it exactly by temporal projection.
\end{enumerate}
\end{enumerate}

\begin{figure}[t]
\centering
\begin{tikzpicture}[
  pl/.style={draw,rounded corners=2pt,align=center,font=\footnotesize,minimum height=1.05cm,text width=3.35cm,inner sep=2pt},
  ab/.style={draw,rounded corners=2pt,align=center,font=\footnotesize,minimum height=1.05cm,text width=3.35cm,inner sep=2pt,fill=gray!16},
  io/.style={draw,dashed,align=center,font=\footnotesize,inner sep=3pt,minimum height=1.05cm},
  arr/.style={-{Stealth[length=2.2mm]},thick},
  lab/.style={font=\scriptsize,inner sep=1pt}]
\node[io] (in) at (-0.2,0) {goal $G$,\\$\KB_0$, rules};
\node[pl] (g0) at (3.1,0)  {\textbf{level 0: plan}\\ (\textsc{Generate}) regress $G$:\\ plan $\pi_0$, assumes $\Gamma_0$};
\node[ab] (a0) at (7.3,0)  {\textbf{level 0: abduce}\\ (\DDA) explain $\Gamma_0$:\\ alternatives $H$};
\node[pl] (g1) at (7.3,-1.75) {\textbf{level 1: plan} for each $H$\\ sub-plan $\sigma_1$,\\ assumes $\Gamma_1$};
\node[ab] (a1) at (11.5,-1.75) {\textbf{level 1: abduce}\\ (\DDA) explain $\Gamma_1$:\\ alternatives $H'$};
\node[pl] (g2) at (11.5,-3.5) {\textbf{level 2: plan} for each $H'$\\ sub-plan $\sigma_2$,\\ assumes $\Gamma_2$};
\node[font=\small,align=center] (dots) at (7.3,-3.5) {$\cdots$ \ until $\Gamma_j=\emptyset$, no explanation,\\ or level $R$ is reached};
\node[pl,text width=7.2cm] (tp) at (3.3,-5.35) {\textbf{schedule} $\sigma_R\cdots\sigma_1\,\pi_0$ and \textbf{project} along the timeline\\ (exact test: \textsc{valid}, \textsc{story}, \textsc{fail}($m$), \textsc{reject})};
\node[io] (out) at (9.6,-5.35) {plans / stories};
\draw[arr] (in)--(g0); \draw[arr] (g0)--(a0);
\draw[arr] (a0.south)--(g1.north);
\draw[arr] (g1)--(a1);
\draw[arr] (a1.south)--(g2.north);
\draw[arr] (g2.west)--(dots.east);
\draw[arr] (dots.south) -- ++(0,-0.35) -| (tp.north east);
\draw[arr] (tp)--(out);
\draw[arr] ([xshift=-1.1cm]tp.north) -- node[lab,right,align=left]{\textsc{Repair}:\\deletion goal} ([xshift=-1.1cm]tp.north |- g0.south);
\end{tikzpicture}

\caption{The recursion of \APT. Planning (white boxes) and abduction (grey boxes) call each other: the assumptions of a plan at level $j$ are explained by \DDA{}; every alternative explanation becomes the goal of a sub-plan at level $j{+}1$, whose assumptions are explained again, up to the depth bound $R$. The sub-plans are scheduled \emph{before} the plans that need them, and the assembled candidate is decided by temporal projection (Section~\ref{sec:projection}); if the projection fails because an earlier fact would have to be deleted, \textsc{Repair} starts a planning step for a deletion goal. Every box calls only the reasoner operations of Section~\ref{sec:backend}.}
\label{fig:loop}
\end{figure}

\paragraph{A running example, in two levels.} A lab admits a person to the \sn{InLab} state by the action \EnterLab, whose rule requires $\sn{Cleared}(p)$. In the ontology $\sn{Cleared}\equiv\sn{Trained}\sqcap\sn{HasPass}$, so no action produces $\sn{Cleared}$: it is a \emph{derived} class. Amir is only known to be \sn{Registered}. The planner regresses $\sn{InLab}(\sn{amir})$ through \EnterLab{} and has to \emph{assume} $\sn{Cleared}(\sn{amir})$ (level 0). \DDA{} explains it by $\{\sn{Trained}(\sn{amir}),\sn{HasPass}(\sn{amir})\}$, and a sub-plan for this goal is $[\Train,\IssuePass]$; but \Train{} requires $\sn{Qualified}(p)$, again a derived class, which is assumed (level 1). \DDA{} explains $\sn{Qualified}(\sn{amir})$ by $\sn{Certified}(\sn{amir})$, and the sub-sub-plan $[\Certify]$ needs only $\sn{Registered}(\sn{amir})$, which holds. Nothing is left to assume, and the assembled plan $[\Certify,\Train,\IssuePass,\EnterLab]$ is checked by temporal projection. The example needs two abductions and two sub-plannings (Section~\ref{sec:examples}); a planner that does only one round of each would stop with the assumption $\sn{Qualified}(\sn{amir})$ unresolved.

\paragraph{Terminology.} In this paper abduction and planning have separate duties: abduction concerns \emph{assertions} (what should hold so that something else follows), planning concerns \emph{actions} (what to do so that something holds). A sequence of actions together with the assumptions it takes for granted is a \emph{hypothetical plan}; a hypothetical plan without assumptions is simply a \emph{plan}. We call the alternation of Figure~\ref{fig:loop} \emph{supported planning}, because each assumption of a plan is supported, where possible, by a sub-plan that is found with the help of abduction.

\paragraph{Parallelism and caching.} Almost all reasoner calls of the tools are independent of each other: the suggestions for different atoms, the checks of different hypotheses, the sub-plans for different explanations, the projections of different candidates. The implementation runs the checks of one round in parallel on a pool of reasoner processes, and it caches suggestions and consistency checks, with single-flight deduplication of concurrent identical requests (Section~\ref{sec:parcache}). Correctness does not depend on either device.

\paragraph{Contributions.}
\begin{enumerate}[label=\textbf{C\arabic*.}]
\item A small \emph{contract} for the reasoner (Section~\ref{sec:backend}) with three required abilities, two derived abilities (pattern evaluation, new consequences of an update) and four optional abilities (incremental updates of added and removed assertions, the derivation record, concurrency, layered copies), so that the tools are independent of the reasoner and of the data structure; a semantics for rules with \texttt{DELETE}, \texttt{INSERT} and the creation of new individuals that does not depend on a query language; and \emph{temporal projection} as an independent tool and as the exact test of plans, with decidability (Section~\ref{sec:dynamics}, Theorem~\ref{thm:transitions}).
\item A treatment of change without modal operators, on one linear timeline and without possible worlds, with decidability as a corollary and with the costs of these choices stated (Section~\ref{sec:logic-position}).
\item A proof that the update solves the ramification problem in a precise sense: indirect effects are exactly the entailments of the updated state, nothing changes without a cause, and violated integrity constraints are reported as inconsistencies (Theorems~\ref{thm:ramif} and~\ref{thm:definitional}), and what this does \emph{not} cover (Proposition~\ref{prop:scope}).
\item \DDA: soundness, termination, exact minimality, joint treatment of several observations, a completeness result, and minimality checking through the derivation record without reasoner calls (Section~\ref{sec:dda}).
\item \APT: the recursion between planning and abduction, its termination and soundness, a repair step for plans that need deletions, the exact characterisation of what the heuristic generator can miss, and the complexity of bounded supported planning, which is $2$\textsc{NExpTime}-hard and in $3$\textsc{ExpTime} (Section~\ref{sec:apt}, Theorems~\ref{thm:apt-hard} and~\ref{thm:apt-upper}).
\item Synthesis of action sequences that end in a \emph{chosen} contradiction, for \emph{any} contradiction the reasoner detects, illustrated on the Taxes benchmark, where the target is a contradiction (Section~\ref{sec:adversarial}).
\item Parallelisation and caching of the tools, with a proof that they do not change the results (Section~\ref{sec:parcache}).
\item A proof that, outside a stated fragment, planning over a TDKG has no compact translation into classical planning, while inside the fragment the two coincide (Theorem~\ref{thm:no-compile}, Proposition~\ref{prop:coincide}); an analytical comparison of \DDA{} with hitting-set abduction and of \APT{} with classical planning; a comparison of \DDA{} with the abduction solver CATS on shared problems, with the reasons why other abduction systems cannot be compared; a discussion of scalability (Section~\ref{sec:comparison}); and six benchmarks, among them BankAccount (an action with two rules and the creation of individuals), DerivedGate (a rule precondition obtained by reasoning in a state) and LabAccess (a recursion of two abductions and two sub-plans) (Section~\ref{sec:examples}).
\item An evaluation of the implementation on five reasoners, which checks that the results do not depend on the reasoner, tests the incremental modes of three reasoners and measures parallel calls (Section~\ref{sec:eval}).
\end{enumerate}

\paragraph{Organisation.} Section~\ref{sec:related} discusses related work, Section~\ref{sec:prelim} the logical background and Section~\ref{sec:backend} the reasoner contract. Section~\ref{sec:dynamics} defines actions, temporal projection, plans and the ramification results, and states in Section~\ref{sec:logic-position} the logical position of the approach: no modal operators, one timeline, no possible worlds. Sections~\ref{sec:dda}, \ref{sec:apt} and~\ref{sec:adversarial} present \DDA, \APT{} and adversarial synthesis, Section~\ref{sec:parcache} parallelism and caching, Section~\ref{sec:comparison} the analytical comparison, Section~\ref{sec:examples} the benchmarks and Section~\ref{sec:eval} the implementation and its evaluation on five reasoners. Section~\ref{sec:discussion} concludes. Appendix~\ref{app:traces} contains the plans that the implementation computes for the benchmarks, and Appendix~\ref{app:formal} the formal statements and proofs on hitting-set abduction and classical planning that support Section~\ref{sec:comparison}.

\section{Related Work}
\label{sec:related}

\paragraph{Actions in description logics.} Baader et al.\ \cite{baader2005integrating} describe an atomic action by a precondition, an effect and an \emph{occlusion} set naming what may change; Chang, Lin and Shi \cite{chang2007dynamic} and Shi et al.\ \cite{shi2005logical} add action operators to the logic itself and study decidability of the resulting dynamic logic. Model-based update semantics for ABoxes \cite{winslett1988reasoning,liu2006updating} define the successor state as the closest models. Our update is simpler and syntactic: actions delete and insert explicit assertions, and everything else is left to the reasoner. Section~\ref{sec:ramification} states what this gives and what it does not, and Section~\ref{sec:logic-position} compares the logical choices (no modal operator, one linear timeline, no possible worlds) with these formalisms.

\paragraph{Ontology-mediated planning.} Knowledge and action bases \cite{bagheri2013dlkab,calvanese2016ekab} and the planners of John and Koopmann \cite{john2023towards,john2024planning} and Borgwardt et al.\ \cite{borgwardt2022expressivity,borgwardt2025automated} combine an ontology with a planner by compilation or by dedicated update semantics. They usually fix the set of objects and restrict the ontology language to keep planning decidable. We allow full $\SROIQ$ and new individuals, and obtain decidability by bounding the plan length (Theorem~\ref{thm:apt-decid}).

\paragraph{ABox abduction.} Calvanese et al.\ \cite{calvanese2013explanations} use abduction in $\mathit{DL\text{-}Lite}$ to explain why an expected tuple is missing from a query answer and analyse the complexity of the explanation tasks. Symbolic solvers for expressive logics compute explanations by Reiter's minimal hitting sets over a declared set of abducible assertions \cite{reiter1987theory,pukancova2020,homola2023}, recently combined into hybrid algorithms \cite{kloc2025cats}; Klarman et al.\ \cite{klarman2011abox} give calculi for $\mathcal{ALC}$, Du et al.\ \cite{du2011practical} reduce abduction in $\mathcal{SHOIQ}$ to datalog and Prolog, and Koopmann et al.\ \cite{koopmann2020signature,koopmann2021signature} and Del-Pinto and Schmidt \cite{delpinto2019} study abduction with fresh individuals, complex concepts and forgetting; the hardness of logic-based abduction is classical \cite{eiter1995complexity}. Neural approaches generate hypotheses statistically \cite{bai2024advancing,gao2025controllable,gao2026unifying,zheng2025logidynamics}. \DDA{} differs in where candidates come from (the consequences of the negated observation, not an enumerated list); Section~\ref{sec:price} compares what the approaches guarantee, and Section~\ref{sec:cmp-systems} runs the solver CATS on shared problems and states why the other systems cannot be compared. The explanation of an entailment by its justifications \cite{kalyanpur2007justifications} is what we use, when the reasoner offers it, to check minimality cheaply; contrastive explanations of ABox entailments \cite{koopmann2025contrastive} are a recent, related way of explaining why one assertion is entailed and another is not.

\paragraph{Abduction and planning together.} In abductive event calculus \cite{shanahan2000abductive,kakas1992abductive} the abduced hypotheses are \emph{events}, that is, actions, in a logic-programming setting. Ontology-mediated planners treat the ontology as a fixed background and plan with it. \APT{} separates the duties: abduction explains \emph{assertions} from the ontology, planning supplies the \emph{actions}, and each is a subroutine of the other. In the approaches we are aware of, the assumptions of a plan are not explained by ontology abduction and then planned for, recursively, over a $\SROIQ$ knowledge base; this alternation is the central idea of \APT{}. The literature on planning with incomplete knowledge is large, and we do not claim to have covered it.

Table~\ref{tab:novelty} summarises how the lines of work differ from ours in the problems they treat. The entries describe the main problem of each line, not everything that its members do.

\begin{table}[t]
\centering\small
\caption{Lines of related work and the problems that they treat as their main problem. Ours is the only entry that explains the facts a plan has to assume by abduction in the ontology and then plans for them, over a $\SROIQ$ knowledge base.}
\label{tab:novelty}
\newcolumntype{P}[1]{>{\raggedright\arraybackslash}p{#1}}
\begin{tabular}{@{}P{3.7cm}P{2.9cm}P{2.6cm}P{2.6cm}P{2.1cm}@{}}
\toprule
Line of work & Ontology & Effects of actions through the ontology & Missing facts explained by abduction & Plan search\\ \midrule
Actions in description logics \cite{baader2005integrating,chang2007dynamic,shi2005logical} & description logics of the cited papers & yes & no & not the main problem\\
Knowledge and action bases \cite{bagheri2013dlkab,hariri2013verification,calvanese2016ekab} & restricted languages ($\mathit{DL\text{-}Lite}$ variants) & yes & no & yes (plan synthesis)\\
Ontology-mediated planners \cite{john2023towards,john2024planning,borgwardt2022expressivity,borgwardt2025automated} & restricted language or fixed objects & yes & no & yes\\
ABox abduction solvers \cite{pukancova2020,kloc2025cats,koopmann2020signature,klarman2011abox,calvanese2013explanations} & $\mathit{DL\text{-}Lite}$ up to $\SROIQ$ & no & yes & no\\
Abductive event calculus and logic programs \cite{shanahan2000abductive,kakas1992abductive} & logic programs, closed world & by program & yes, as events & yes\\
This paper & $\SROIQ$, new individuals & yes, by entailment & yes, by \DDA{}, recursively & yes, bounded\\
\bottomrule
\end{tabular}
\end{table}

\paragraph{Inconsistency and incrementality.} Inconsistency-tolerant semantics \cite{lembo2010inconsistency} reason \emph{despite} contradictions; we instead \emph{construct} them on purpose. Incremental reasoning reuses earlier work after small changes \cite{grau2007history,cuenca2010incremental,kazakov2013incremental}, and recent work extends it to inconsistency-resilient reasoning over streams of ABoxes \cite{proost2026incremental,bonte2024grounding}; we need from it only that updated states answer like fresh ones, and we exploit its speed.

\section{Preliminaries}
\label{sec:prelim}

\subsection{Knowledge bases}

A knowledge base (KB) is a pair $\KB=\langle\TB,\AB\rangle$ in $\SROIQ$ \cite{horrocks2006sroiq}. The \emph{terminology} $\TB$ contains the TBox and RBox axioms (concept inclusions, role inclusions, transitivity, disjointness, and so on); the \emph{ABox} $\AB$ is a finite set of \emph{assertions}:
\[
C(a),\qquad R(a,b),\qquad \neg R(a,b),\qquad a\approx b,\qquad a\not\approx b ,
\]
where $a,b$ are individual names, $R$ a role and $C$ a (possibly complex) concept. We write $\Ind(\KB)$ for the individual names of $\KB$. As usual, $\KB\models\alpha$ means that every model of $\KB$ satisfies the assertion $\alpha$, and $\KB$ is \emph{consistent} if it has a model.

\paragraph{States.} In a TDKG the terminology $\TB$ stays fixed and only the ABox changes with time. A \emph{state} is the knowledge base $\langle\TB,\AB_t\rangle$ that holds at one time instant $t$; the states of a TDKG form a sequence $\KB_0,\KB_1,\dots$, one for each step of Section~\ref{sec:dynamics}. A state is \emph{consistent} or \emph{inconsistent} as a knowledge base is. An inconsistent KB entails everything; hence \emph{we never ask what an inconsistent state entails}. An inconsistent state is only detected and, for adversarial synthesis (Section~\ref{sec:adversarial}), analysed through its consistent parts.

\begin{theorem}[\cite{horrocks2006sroiq,kazakov2008riq}]\label{thm:sroiq}
Consistency of a $\SROIQ$ KB is decidable, and $2$\textsc{NExpTime}-complete.
\end{theorem}

\paragraph{Open world.} If $\KB$ is consistent, an assertion $\alpha$ can be entailed, contradicted ($\KB\models\neg\alpha$), or neither. The last case, \emph{unknown}, is the room in which abduction works.

\subsection{Negating assertions}

The engines create the negation of assertions themselves; nobody has to supply them. The negation of $C(a)$ is $(\neg C)(a)$ (and $\neg\neg C$ is simplified to $C$), the negation of $R(a,b)$ is $\neg R(a,b)$ and vice versa, and $a\approx b$ is negated to $a\not\approx b$. In this paper observations, hypotheses and goals are assertions, and the engines negate them in this way.\footnote{OWL~2~DL allows negative role assertions only for \emph{simple} roles (roughly, roles that are not transitive or defined by chains). Observations on non-simple roles are therefore outside the scope of \DDA; this restriction plays no further role.}

\begin{lemma}\label{lem:neg}
For assertions $\alpha,\beta$ and a KB $\KB$:
(a) $\KB\models\alpha$ iff $\KB\cup\{\neg\alpha\}$ is inconsistent;
(b) if $\KB\cup\{\alpha\}\models\beta$ then $\KB\cup\{\neg\beta\}\models\neg\alpha$.
\end{lemma}
\begin{proof}
A model satisfies $\neg\alpha$ exactly when it does not satisfy $\alpha$. (a) follows at once. For (b), let $\mathcal I\models\KB\cup\{\neg\beta\}$. If $\mathcal I\models\alpha$ then $\mathcal I\models\beta$, which contradicts $\mathcal I\models\neg\beta$; so $\mathcal I\not\models\alpha$, i.e.\ $\mathcal I\models\neg\alpha$.
\end{proof}

\begin{corollary}\label{cor:entail}
Entailment of an assertion $\alpha$ by a $\SROIQ$ KB is decidable.
\end{corollary}
\begin{proof}
By Lemma~\ref{lem:neg}(a), $\KB\models\alpha$ iff $\KB\cup\{\neg\alpha\}$ is inconsistent, which is decidable by Theorem~\ref{thm:sroiq}.
\end{proof}

Part (b) of Lemma~\ref{lem:neg} is the contraposition on which \DDA{} rests: if $\alpha$ would conclude $\beta$, then the negation of $\beta$ concludes the negation of $\alpha$. Read in the other direction, whatever follows from $\neg\beta$ is, once negated, something that concludes $\beta$.

\subsection{Patterns}

Rule bodies, goals and contradiction patterns are written as \emph{patterns}, which in SPARQL are basic graph patterns. A \emph{pattern} $P$ is a finite set of atoms $C(x)$ and $R(x,y)$ (and negated atoms) whose arguments are variables or individual names. An \emph{answer} to $P$ in a consistent KB $\KB$ is a map $\mu$ from the variables of $P$ to individual names of $\KB$ with $\KB\models\mu(\sigma)$ for each atom $\sigma\in P$ (so the answer sees \emph{entailed} facts, as under the SPARQL entailment regimes \cite{glimm2013entailment}); $\mathrm{Ans}(\KB,P)$ is the set of answers. A pattern \emph{holds} in $\KB$ if it has an answer.

\begin{proposition}\label{prop:select}
$\mathrm{Ans}(\KB,P)$ is computable with at most $|\Ind(\KB)|^k\cdot|P|$ entailment checks, $k$ the number of variables of $P$.
\end{proposition}
\begin{proof}
Try all assignments of individuals to the variables and check each atom.
\end{proof}

\subsection{The finite space of consequences}
\label{sec:universe}

\DDA{} reads off consequences of a KB. To keep this set finite, it only considers assertions built from the individuals of the KB and of the observations, from roles occurring in $\TB$ and in the observations, and from the concept expressions occurring in $\TB$ and in the observations, together with their negations. Call this finite set $\Omega$. It is closed under negation, and its size is polynomial in the size of the input. This is the natural vocabulary of tableau-style reasoners, whose node labels range over exactly such expressions. For example, if $\TB$ contains $\sn{Mother}\equiv\sn{Parent}\sqcap\sn{Woman}$, then $\Omega$ contains $\sn{Mother}(a)$, $\neg\sn{Mother}(a)$, $(\sn{Parent}\sqcap\sn{Woman})(a)$ and $\neg(\sn{Parent}\sqcap\sn{Woman})(a)$ for every individual $a$.

\section{What the Tools Need from a Reasoner}
\label{sec:backend}

A \emph{reasoner} is a system that holds a KB $\KB=\langle\TB,\AB\rangle$, derives its consequences by logical inference rules that are sound and complete for the language (for example tableau rules, consequence-based saturation rules, or rules over triples), and offers the five operations below so that (B1)--(B3) hold. The tools treat it as a black box; the terminology $\TB$ never changes, only the ABox does. We distinguish three kinds of ability, and every ability of this section has a number.
\begin{itemize}[nosep]
\item \emph{Required} abilities, (B1)--(B3): the reasoner must provide them, and they decide correctness. A system that lacks one cannot be used, except that the tools can provide (B1) themselves by rebuilding a KB from its stored ABox, which is slower.
\item \emph{Derived} abilities, (B4) and (B5): the reasoner need not provide them, because they are computed from (B1)--(B3) by a stated number of calls; a reasoner that provides them natively makes the tools faster and changes nothing else.
\item \emph{Optional} abilities, (B6)--(B9): incremental updates, the derivation record, concurrency and layered copies. They change the cost, never the results.
\end{itemize}
Table~\ref{tab:capabilities} at the end of the section summarises all nine, and Table~\ref{tab:abilities} in Section~\ref{sec:considered} states which of them each tested reasoner provides. Two things that may look like abilities are not: caching of results is done by the tools and needs nothing from the reasoner, and the query language (SPARQL) is one way of providing (B4); both are discussed at the end of this section.

\subsection{The five operations and three requirements}

The reasoner offers five operations.
\begin{center}\small
\begin{tabular}{@{}ll@{}}
\toprule
\textbf{Operation} & \textbf{Meaning}\\ \midrule
$\op{Copy}(\KB)$ & a new, independent KB equal to $\KB$\\
$\op{Add}(\KB,S)$, $\op{Remove}(\KB,S)$ & add / remove a finite set $S$ of assertions to / from the ABox\\
$\op{Consistent}(\KB)$ & true iff $\KB$ has a model\\
$\op{Entails}(\KB,\alpha)$ & true iff $\KB\models\alpha$ (only asked if $\KB$ is consistent)\\
\bottomrule
\end{tabular}
\end{center}
They must satisfy three requirements.
\begin{description}
\item[(B1) Isolation.] Changing a copy never changes the original.
\item[(B2) Explicit update.] $\op{Add}$ and $\op{Remove}$ act on the \emph{explicit} ABox. In particular, $\op{Remove}$ deletes an assertion only if it is stated; it cannot delete one that is merely entailed. $\op{Add}$ accepts assertions about individuals that do not occur in $\KB$, such as the fresh names that rules create (Section~\ref{sec:dynamics}); a reasoner needs no ability to create individuals itself.
\item[(B3) Faithfulness.] After any sequence of operations, $\op{Consistent}$ and $\op{Entails}$ give the answers that a reasoner started from scratch on the same KB would give; the answers are correct also for complex concepts in $\alpha$.
\end{description}
(B3) is the only requirement about logic. It says that however the reasoner is implemented, its answers are exact; speed-ups such as incremental updates (B6, below) are invisible in the results.

\subsection{Two derived operations}

The tools also use two operations that are not required of the reasoner: they are computed from the five operations, and a reasoner that provides them natively makes the tools faster, and nothing else changes.

\paragraph{(B4) Pattern evaluation.} $\op{Select}(\KB,P)$ returns the answers $\mathrm{Ans}(\KB,P)$ of a pattern $P$ in $\KB$ (Section~\ref{sec:prelim}). By Proposition~\ref{prop:select} it needs at most $|\Ind(\KB)|^k\cdot|P|$ calls of $\op{Entails}$, where $k$ is the number of variables of $P$. It is used to evaluate rule bodies and goals. Reasoners that evaluate SPARQL under an entailment regime \cite{glimm2013entailment} provide it natively.

\paragraph{(B5) New consequences.} For two KBs $\KB\subseteq\KB'$ with the same terminology and $\AB\subseteq\AB'$, and a finite set $\Omega$ of assertions, the \emph{new consequences} of the step from $\KB$ to $\KB'$ within $\Omega$ are
\[
\mathrm{New}(\KB,\KB',\Omega)=\{\delta\in\Omega\mid\KB'\models\delta\}\setminus\{\delta\in\Omega\mid\KB\models\delta\}.
\]
\DDA{} uses it to read what the reasoner derives after the negated observation has been added (Section~\ref{sec:suggest}).

\begin{proposition}[Computing the derived operations]\label{prop:derived}
(a) $\op{Select}(\KB,P)$ is computable with the five operations, with at most $|\Ind(\KB)|^k\,|P|$ calls of $\op{Entails}$. (b) $\mathrm{New}(\KB,\KB',\Omega)$ is computable with at most $2|\Omega|$ calls of $\op{Entails}$, and with none if the reasoner reports what an update derives.
\end{proposition}
\begin{proof}
(a) is Proposition~\ref{prop:select}: try all assignments of individuals to the $k$ variables and check each of the $|P|$ atoms. (b) For every $\delta\in\Omega$ ask $\op{Entails}(\KB',\delta)$ and $\op{Entails}(\KB,\delta)$ and keep the $\delta$ for which the first answer is true and the second false. If the reasoner already knows which assertions an update newly derived, as an incremental reasoner does, it reports them directly.
\end{proof}

\subsection{Optional abilities}

\paragraph{(B6) Incremental updates.} A reasoner supports incremental updates if it processes $\op{Add}$ and $\op{Remove}$ without starting over, so that the work after an update is roughly proportional to what the update changes. Incremental reasoning has been developed for classification \cite{grau2007history,cuenca2010incremental}, for $\mathcal{EL}$ ontologies \cite{kazakov2013incremental}, and, more recently, for streams of ABoxes whose windows change by small steps \cite{proost2026incremental}; the broader field of stream reasoning is surveyed in \cite{bonte2024grounding}. Both tools make many small changes to one large knowledge base: \DDA{} adds a single negated observation and looks at what changes; \APT{} adds one step of a plan and looks at what changes. Every search node therefore holds a \emph{copy} of its parent's reasoner state and applies one small update to it. Correctness never depends on (B6), because of (B3); only the running time does. On large ontologies, where a single check is expensive, incremental updates should be the most valuable of the optional abilities; we have not measured this, and Section~\ref{sec:eval-incr} tests the incremental modes that exist.

\paragraph{(B7) Derivation record.} For a consistent KB $\KB$ and an assertion $\alpha$ with $\KB\models\alpha$, the reasoner may return a set $\mathrm{Der}(\KB,\alpha)$ of \emph{derivations} of $\alpha$: the stored graph of inference steps that led to $\alpha$. A derivation $d$ uses a set $\mathrm{Prem}(d)\subseteq\AB$ of explicit assertions, its \emph{minor premises}; the axioms of $\TB$ are the \emph{major premises}. We require \emph{soundness}: $\langle\TB,\mathrm{Prem}(d)\rangle\models\alpha$ for every derivation $d$. The record is \emph{complete} if moreover every \emph{justification} of $\alpha$, that is, every subset-minimal $J\subseteq\AB$ with $\langle\TB,J\rangle\models\alpha$, equals $\mathrm{Prem}(d)$ for some $d$. (We say \emph{derivation record} and \emph{derivation} here, and keep the words \emph{suggestion}, \emph{hypothesis} and \emph{explanation} for abduction in Section~\ref{sec:dda}; the two notions are different, the first describes how the reasoner \emph{inferred} an assertion, the second what we \emph{assume} in order to make an observation true.) \DDA{} uses the record to check minimality without further reasoner calls (Section~\ref{sec:fastmin}).

\paragraph{(B8) Concurrency.} Calls on \emph{different} copies of a KB may be issued at the same time (the reasoner is thread-safe per copy, or runs several workers). The tools use this in Section~\ref{sec:parcache}. Running the reasoner calls of the tools on several copies at once is an application of this ability and not a further one.

\paragraph{(B9) Layered copies.} A copy that costs time and space in proportion to what has been inserted into it, and not to the size of the KB. A \emph{state} is a pair $(r,G)$ of a reasoner $r$ and a set $G$ of assertions, its \emph{delta graph}, over a \emph{main model} $\KB_0$ that many states share. The reasoner $r$ reads $\KB_0\cup G$. Everything that is added to the state, by the tools or derived by $r$, is written to $G$ only; $\KB_0$ is never changed, and a state never deletes. $\op{Copy}$ of a state creates a new reasoner $r'$ and a copy $G'$ of $G$, so (B1) holds because $r'$ writes only to $G'$. A triple store realises a state with named graphs: $\KB_0$ is one graph, $G$ another, a query is evaluated over both (in SPARQL, the \texttt{FROM} clause names both graphs) and writes go to $G$.
\begin{itemize}[nosep]
\item \emph{Planning.} Let $\KB_0$ be the initial state. For every rule $a$ that is applicable, a state $(r_a,G_a)$ with $G_a$ empty is created, and $r_a$ derives the effects of $a$ into $G_a$. If the preconditions of two rules $b$ and $c$ then hold, two states $(r_b,G_b)$ and $(r_c,G_c)$ are created whose graphs are copies of $G_a$, and each continues without changing $G_a$ or the other. The state of a branch is dropped when the branch ends. Preconditions are checked over $\KB_0\cup G$.
\item \emph{Abduction.} To find the suggestions for an observation $o$ (Section~\ref{sec:suggest}), a state $(r,G)$ with $G=\{\neg o\}$ is created; $r$ derives the consequences of $\KB_0\cup G$ into $G$, and the negations of the new consequences are the suggestions. For every suggestion $a$ a state $(r_a,G_a)$ with $G_a=\{\neg a\}$ is created in the same way, and so on for the next round.
\end{itemize}
Deletion by rules happens only in temporal projection (Section~\ref{sec:dynamics}), and projection runs on the main model and creates no state; the restriction of a copy to a subset of its ABox in minimisation (Section~\ref{sec:fastmin}) is likewise a removal and does not use a state. With $b$ live states whose delta graphs hold $\delta$ assertions each, the memory is $|\KB_0|+b\delta$ and not $b(|\KB_0|+\delta)$. The price is paid in reads, which take the union of the two graphs. Correctness never depends on (B9), because the algorithms only use $\op{Copy}$, $\op{Add}$ and the reads; (B9) changes what a copy costs (Section~\ref{sec:dda-algs}). We have not measured the effect of (B9) on time or memory.

\paragraph{Two kinds of deletion.} Deleting by the rules of a TDKG is different from deleting in incremental reasoning. Let a \emph{justification} of an assertion $\alpha$ be a subset-minimal $J\subseteq\AB$ with $\langle\TB,J\rangle\models\alpha$ (as in (B7)).
\begin{itemize}[nosep]
\item \emph{Deletion by an action} removes the \emph{stated} assertions $D$ of a rule (by (B2), $\op{Remove}$ acts on the explicit ABox) and nothing else; the consequences of the new ABox are what the new ABox entails. An assertion of $D$ that is still entailed by the remaining ABox, because it has another justification, remains true, and a reasoner that stores derivations records it again as derived. A derived assertion cannot be deleted by naming it (Section~\ref{sec:ramification}(ii)); it disappears when all its justifications are cut. The rules of a TDKG delete only in this way, in projection.
\item \emph{Deletion in incremental reasoning} (B6) retracts an assertion $\alpha$ from a reasoner that has already derived its consequences. The consequences of $\alpha$ are withdrawn, except those that have a justification without $\alpha$. If the assertion to be retracted is itself derived, then at least one assertion of each of its justifications must be retracted, and which one is a choice that the reasoner leaves to the user.
\end{itemize}
The second kind needs the justifications of the derived assertions, so it is provided by reasoners that can find them, most directly by a reasoner with a derivation record (B7), where the justifications of a consequence are read from the record; a reasoner without it must recompute them. The tools need only the first kind: a rule names only stated assertions, so they never ask a user which assertion to retract.

\paragraph{Not reasoner abilities.} \emph{Caching.} The tools remember the answers to the checks that they have already made (Section~\ref{sec:parcache}). This is done in the tools, on top of (B1)--(B3), and needs nothing from the reasoner, so it has no number. \emph{Query language.} SPARQL is a way of providing (B4): a reasoner or triple store that evaluates SPARQL under an entailment regime provides it natively, and the tools use SPARQL only through (B4). A reasoner without SPARQL is used through $|\Ind(\KB)|^k\cdot|P|$ entailment checks.

\subsection{Summary}

\begin{table}[t]
\centering\small
\caption{Reasoner abilities used by the tools: what is required, derived or optional, who uses it, and what happens if it is missing.}
\label{tab:capabilities}
\begin{tabularx}{\textwidth}{@{}P{4.0cm}P{1.9cm}P{2.9cm}X@{}}
\toprule
\textbf{Ability} & \textbf{Kind} & \textbf{Used by} & \textbf{If the reasoner lacks it}\\ \midrule
(B1) Isolated copies & required & search, suggestions, projection & rebuild the KB from a stored ABox (slower)\\
(B2) Add and remove explicit assertions & required & rules, hypotheses & ---\\
(B3) Consistency and entailment, also for complex concepts & required & all & ---\\
(B4) Pattern evaluation under entailment & derived & rule bodies, goals & $|\Ind|^k|P|$ entailment checks\\
(B5) New consequences of an update & derived & \DDA{} & $2|\Omega|$ entailment checks\\
(B6) Incremental updates & optional & all & every update is a full reasoning run\\
(B7) Derivation record & optional & \DDA{} (minimality) & redundancy tested by entailment calls\\
(B8) Concurrency & optional & all (Section~\ref{sec:parcache}) & calls are issued one after another\\
(B9) Layered copies (insertions only) & optional & search branches, suggestions & every copy duplicates the KB, or the KB is rebuilt from a stored ABox\\
\bottomrule
\end{tabularx}
\end{table}

\paragraph{Which reasoners qualify.} Any reasoner that is sound and complete for the language in use, such as HermiT, Pellet, Openllet, JFact (behind the OWL API \cite{horridge2011owlapi}) and Konclude through OWLlink \cite{steigmiller2014konclude}; copies and updates are ontology copies and changes. Section~\ref{sec:eval} runs the tools on these five and lists the reasoners that could not be used. For a reasoner that is complete only for a fragment (for instance $\mathcal{EL}$ or $\mathcal{RL}$), everything below holds when the ontology, rules, goals and observations are restricted to that fragment. A reasoner that cannot reliably detect inconsistency cannot be used, since validity of plans depends on it. The derived abilities (B4) and (B5) and the optional abilities (B6)--(B9) are not required; which of them each tested reasoner provides is stated in Table~\ref{tab:abilities} and discussed in Section~\ref{sec:price}.

\begin{proposition}[Reasoner independence]\label{prop:independence}
The tools of Sections~\ref{sec:dda}--\ref{sec:adversarial} access a KB only through the operations of this section. Hence all theorems below hold for every reasoner satisfying (B1)--(B3); the derived and optional abilities change only the number or the order of calls, never their results.
\end{proposition}
\begin{proof}
Every algorithm below is written with $\op{Copy}$, $\op{Add}$, $\op{Remove}$, $\op{Consistent}$, $\op{Entails}$ and the derived operations of Proposition~\ref{prop:derived}; the proofs use only the exactness (B3) of the answers. A native implementation of (B4) or (B5), a derivation record, concurrent calls and incremental updates return the same answers as the computations given above.
\end{proof}

\section{Actions, Temporal Projection and Plans}
\label{sec:dynamics}

\subsection{Actions and their rules}
\label{sec:rules}

An \emph{action} is a class of events, such as \sn{OpenAccount} in the bank benchmark. An event $e$ of that class happens at a time instant $t$ ($\sn{hasTime}(e,t)$) and has participants (an agent, an undergoer, \dots). What the action does is given by one or more \emph{rules} attached to the class. A rule has a \emph{body}, which says when it fires (an event of the class at the current instant, whose participants satisfy the preconditions), and \emph{templates} that say which assertions to delete and which to insert. In SPARQL~1.1 Update a rule is written
\begin{center}
\texttt{DELETE \{ $D$ \} INSERT \{ $I$ \} WHERE \{ $P$ \}}
\end{center}
where $P$ is a pattern, and $D$ and $I$ are templates of assertions to delete and to insert. A request of the form \texttt{CONSTRUCT \{ $I$ \} WHERE \{ $P$ \}} is the special case with no deletion.

\paragraph{Actions with one or several rules.} The bank benchmark, adapted from Baader et al.~\cite{baader2005integrating}, has two actions and three rules. The action \sn{OpenAccount} has \emph{two} rules (Listing~\ref{lst:bank}). When an \sn{OpenAccount} event happens, \emph{both} bodies are evaluated, and the rule whose body holds fires. Rule~1 opens an account without a card for an eligible person who has a proof of address and no letter yet; rule~2 opens an account with a card if the person has a letter. The action \sn{GetLetter} has \emph{one} rule, which creates a letter for a person (it is not shown). In each rule the variable \texttt{?ac} (respectively the letter) occurs only in the template: the rule \emph{creates} a new individual. The side condition \texttt{FILTER NOT EXISTS} of rule~1 is part of the benchmark as it was written; the formal model does not need side conditions, and tolerates them as (Q3) below explains.

\begin{lstlisting}[caption={The action \texttt{OpenAccount} with its two rules (benchmark \texttt{BankAccount}; prefixes \texttt{ba:}, \texttt{bu:}, \texttt{bt:} abbreviate the benchmark namespaces and \texttt{?\_T} is the time placeholder). Both rules only insert.},label={lst:bank},captionpos=b]
# rule_openAccount: no letter yet -> account without card
CONSTRUCT { ?ac rdf:type ba:BankAccountWithoutCard .  ?human ba:has ?ac . }
WHERE { ?oa rdf:type ba:OpenAccount .  ?oa bu:EventHasAgent ?human .  ?oa bt:hasTime ?_T .
        ?human rdf:type ba:EligiblePerson .
        ?human ba:has ?pr .  ?pr rdf:type ba:ProofOfAddress .
        FILTER NOT EXISTS { ?human ba:has ?letter .  ?letter rdf:type ba:Letter . } }

# rule_openAccountwithLetter: a letter exists -> account with card
CONSTRUCT { ?ac rdf:type ba:BankAccountWithCard .  ?human ba:has ?ac . }
WHERE { ?oa rdf:type ba:OpenAccount .  ?oa bu:EventHasAgent ?human .  ?oa bt:hasTime ?_T .
        ?human rdf:type ba:EligiblePerson .
        ?human ba:has ?pr .  ?pr rdf:type ba:ProofOfAddress .
        ?human ba:has ?letter .  ?letter rdf:type ba:Letter . }
\end{lstlisting}

Listing~\ref{lst:rules} shows two rules of other benchmarks, one that only inserts and one that deletes and inserts.

\begin{lstlisting}[caption={Two further rules. \texttt{sec:} and \texttt{com:} are benchmark namespaces. \texttt{Act\_TurnKey} only inserts; \texttt{ImportWheatAction} deletes and inserts.},label={lst:rules},captionpos=b]
# action Act_TurnKey: insert-only (CONSTRUCT)
CONSTRUCT { ?lock rdf:type sec:Unlocked . }
WHERE { ?e rdf:type sec:Act_TurnKey .   ?e bt:hasTime ?_T .
        ?e bu:EventHasUndergoer ?lock . ?lock rdf:type sec:MechanicalLock . }

# action ImportWheatAction: delete and insert
DELETE { ?person rdf:type com:TaxPayer . }
INSERT { ?person rdf:type com:TaxExempt . }
WHERE  { ?a rdf:type com:ImportWheatAction .  ?a bt:hasTime ?_T .
         ?a bu:EventHasAgent ?person .        ?person rdf:type com:Trader . }
\end{lstlisting}

\begin{definition}[Rule]\label{def:rule}
A \emph{rule} $r$ of the action $A_r$ consists of a pattern $P_r$ (the body) and two finite sets $D_r$, $I_r$ of atoms (the delete and insert templates), such that
\begin{itemize}
\item $P_r$ contains an atom $A_r(?e)$ for the action class $A_r$ and an atom $\sn{hasTime}(?e,?\_T)$; atoms that mention $?e$ are the \emph{event atoms} of $r$ and the other variables occurring in them (agent, undergoer, \dots) are the \emph{participants} of $r$; the remaining atoms are the \emph{preconditions}, written $\mathit{pre}(r)$;
\item every variable of $D_r$ occurs in $P_r$; variables of $I_r$ that do not occur in $P_r$ are \emph{new} (they stand for individuals that the action creates);
\item $D_r$ and $I_r$ do not mention $\sn{hasTime}$ or $\sn{before}$: actions change the world, not the clock;
\item instantiating the templates always yields valid $\SROIQ$ assertions for $\TB$ (this is a syntactic check on the finitely many templates).
\end{itemize}
Several rules may share the same action class; we write $\RS_A$ for the rules of the action $A$ (for \sn{OpenAccount}, $|\RS_A|=2$).
\end{definition}

\paragraph{The rule interface: the setting is not tied to SPARQL.} The engines use a rule in exactly two ways: they ask which instantiations of the body hold at a given instant, and they apply the instantiated templates. Everything else is done by the reasoner. We therefore require of the system that evaluates rules only the following.
\begin{description}
\item[(Q1) Matching.] For a consistent KB $\KB$, a rule $r$ and an instant $t$, it can compute the finite set $\op{Match}(\KB,r,t)$ of \emph{solutions}, i.e.\ the answers $\mu$ of the body $P_r$ in $\KB$ with $\mu(?\_T)=t$. The solutions depend on $\KB$ only through what $\KB$ entails (a rule must see derived facts: in the benchmark \texttt{DerivedGate}, a precondition holds only because of a TBox axiom).
\item[(Q2) Effects.] It can apply a finite set of deletions and a finite set of insertions to the explicit ABox (B2), and it can name the individuals that an action creates by a function of the rule and the solution (so that every step is deterministic).
\item[(Q3) Access to the atoms.] The planner can read the atoms of $\mathit{pre}(r)$, $I_r$ and $D_r$ (regression needs them). Further conditions that the system evaluates in (Q1) but that are not atoms, for instance \texttt{FILTER NOT EXISTS}, are allowed; they are called \emph{side conditions}. The generator of Section~\ref{sec:apt} ignores them, which can make it propose plans that fail, but the exact projection evaluates them and therefore never accepts such a plan.
\end{description}
SPARQL~1.1 realises (Q1)--(Q3) as follows: the \texttt{WHERE} clause is evaluated as a \texttt{SELECT} query under an entailment regime \cite{glimm2013entailment}, and the instantiated \texttt{DELETE}/\texttt{INSERT} templates are applied with \texttt{DELETE DATA} and \texttt{INSERT DATA}. The reason for two phases is that the \texttt{WHERE} clause of a SPARQL update is, by default, evaluated on the stored graph and not under entailment, so a precondition that is only entailed would not be found. An equivalent realisation uses any other query facility over the reasoner, for example a rule language whose bodies are matched against the inferred knowledge, a graph-store API, or programmatic calls through the reasoner's interface; in all of these (Q1)--(Q3) are what is needed. In our benchmark TDKGs the rules are stored as SPARQL texts attached to the action classes \cite{shahbazi2024strategy}.

\begin{proposition}[Query-language independence]\label{prop:qli}
All results of this paper on transitions, projection, ramification, \DDA, \APT{} and adversarial synthesis hold for every rule system satisfying (Q1)--(Q3) on top of a reasoner satisfying (B1)--(B3).
\end{proposition}
\begin{proof}
The proofs use rules only through the sets of solutions of (Q1), the application of templates of (Q2), and, for \textsc{Generate}, the atoms of (Q3); none uses a property of SPARQL. Side conditions only restrict the solutions of (Q1) and are treated by the exact projection.
\end{proof}

\subsection{Steps}

\begin{definition}[Step]\label{def:step}
Let $\KB=\langle\TB,\AB\rangle$ be consistent, $t$ a time instant, and $\RS$ a finite set of rules. For $r\in\RS$ and $\mu\in\op{Match}(\KB,r,t)$ let $\mu^{+}$ extend $\mu$ by mapping each new variable of $r$ to a fresh individual name, chosen as a function of $r$, $\mu$ and the variable. Put
\[
D=\bigcup_{r\in\RS}\ \bigcup_{\mu\in\op{Match}(\KB,r,t)}\mu(D_r),\qquad I=\bigcup_{r\in\RS}\ \bigcup_{\mu\in\op{Match}(\KB,r,t)}\mu^{+}(I_r),\qquad
\Step(\KB,t)=\langle\TB,\ (\AB\setminus D)\cup I\rangle .
\]
\end{definition}

In words: all rule bodies are evaluated on the state \emph{before} the step; deletions are applied first and insertions second, so an assertion that is both deleted and inserted stays; and a rule can only delete assertions that are stated explicitly (B2). If no rule fires at $t$, then $D=I=\emptyset$ and the state does not change.

\subsection{Temporal projection}
\label{sec:projection}

A TDKG is a consistent KB together with a set of rules in which events are attached to time instants. The instants of the KB are linearly ordered by $\sn{before}$; this ordered sequence $\mathrm{Timeline}(\KB)=(t_1,\dots,t_n)$ is the \emph{intrinsic timeline} of the TDKG. It is computed with two pattern evaluations (the instants, and the entailed $\sn{before}$ relation between them). Time is \emph{rigid}: no action changes the instants or their order, which is guaranteed by the third condition of Definition~\ref{def:rule}.

\emph{Temporal projection} answers the question: \emph{what is the last state of this TDKG if every action that can be triggered along its timeline is executed?} It plays the timeline like a film: at each instant it evaluates the rules, applies their effects and moves on. It is an independent tool of a TDKG, with three uses.
\begin{enumerate}[label=(P\arabic*)]
\item \emph{Prediction.} Without any planner, projection computes the last state of the TDKG, and, because the whole trajectory is returned, every intermediate state. This is how the state of a TDKG at the end of its timeline is defined.
\item \emph{Debugging.} If some state is inconsistent, the projection stops there and reports the instant and the rules that fired; the conflict can then be extracted (Section~\ref{sec:repair}).
\item \emph{Testing.} A plan or a contradictory story is scheduled into the timeline as new events (Section~\ref{sec:plans}), and the decision whether it is valid is read off the projection of the result (Algorithm~\ref{alg:classify}). \emph{Every test of a candidate by \APT{} is a temporal projection.}
\end{enumerate}

\begin{algorithm}[t]
\caption{$\textsc{Project}(\KB,\tau)$: temporal projection of a TDKG along a timeline $\tau=(t_1,\dots,t_n)$}\label{alg:project}
\begin{algorithmic}[1]
\Require KB $\KB$ with rules $\RS$ and a timeline $\tau$ (by default $\tau=\mathrm{Timeline}(\KB)$)
\Ensure the trajectory $\KB_0,\dots,\KB_m$, the firing log $F_1,\dots,F_m$ and a flag $\mathit{done}$
\State $\KB_0\gets\op{Copy}(\KB)$
\If{$\neg\op{Consistent}(\KB_0)$} \Return $\langle(\KB_0),(),\mathsf{false}\rangle$ \EndIf
\For{$i\gets1$ \textbf{to} $n$}
  \State $F_i\gets\{(r,\mu)\mid r\in\RS,\ \mu\in\op{Match}(\KB_{i-1},r,t_i)\}$ \Comment{evaluate all bodies on the state \emph{before} the step}
  \State $D\gets\bigcup_{(r,\mu)\in F_i}\mu(D_r)$; $\ I\gets\bigcup_{(r,\mu)\in F_i}\mu^{+}(I_r)$
  \State $\KB_i\gets\op{Copy}(\KB_{i-1})$; $\op{Remove}(\KB_i,D)$; $\op{Add}(\KB_i,I)$ \Comment{deletions first, insertions second}
  \If{$\neg\op{Consistent}(\KB_i)$} \Return $\langle(\KB_0,\dots,\KB_i),(F_1,\dots,F_i),\mathsf{false}\rangle$ \EndIf \Comment{stop: no body can be evaluated on an inconsistent state}
\EndFor
\State \Return $\langle(\KB_0,\dots,\KB_n),(F_1,\dots,F_n),\mathsf{true}\rangle$
\end{algorithmic}
\end{algorithm}

Algorithm~\ref{alg:project} returns the whole \emph{trajectory}, not only the last state, because the tests of plans need to know that every intermediate state is consistent and that a given action really fired. If an intermediate state is inconsistent, the projection stops at it: rule bodies are evaluated through entailment, and an inconsistent state entails everything. The flag $\mathit{done}$ says that all states up to the last one are consistent. We write $\Proj(\KB)$ for the last state of $\textsc{Project}(\KB,\mathrm{Timeline}(\KB))$ when $\mathit{done}$ is true.

\begin{theorem}[Transitions and projection are well defined and decidable]\label{thm:transitions}
Let $\KB$ be consistent and $\RS$ a set of rules satisfying (Q1)--(Q3). (1) The set $\op{Match}(\KB,r,t)$ is finite and computable with at most $\iota^{v}\cdot a$ entailment checks, where $\iota=|\Ind(\KB)|$, $v$ is the maximal number of variables of a body and $a$ its maximal number of atoms; it depends on $\KB$ only through what it entails. (2) $\Step(\KB,t)$ is a well-defined $\SROIQ$ KB, determined by $\KB$, $\RS$ and $t$, and it is computable. (3) $\textsc{Project}(\KB,\tau)$ terminates after at most $n(|\RS|\,\iota'^{v}a+4)+1$ reasoner calls, where $\iota'$ is the largest number of individuals of a state in the trajectory, and its result is determined by $\KB$, $\RS$ and $\tau$. In particular, whether the projection of a TDKG is completed and whether its last state is consistent are decidable.
\end{theorem}
\begin{proof}
(1) Proposition~\ref{prop:select}, with $?\_T$ fixed. (2) $D$ and $I$ are finite unions of instantiated templates, and the fresh names are given by a function. The terminology is unchanged and the instantiated templates are valid assertions (Definition~\ref{def:rule}), so the result is a $\SROIQ$ KB. (3) Each iteration makes one matching, one copy, one removal, one addition and one consistency check, and there are $n$ iterations plus the initial check. All answers are exact by (B3), and each consistency check terminates by Theorem~\ref{thm:sroiq}.
\end{proof}

\begin{example}[Projection of the bank TDKG]\label{ex:bankproj}
The \sn{BankAccount} TDKG has instants $T_0<T_1<T_2$ and three events attached to them: an \sn{OpenAccount} event $\sn{oa1}$ at $T_0$, a \sn{GetLetter} event $\sn{gl1}$ at $T_1$ and a second \sn{OpenAccount} event $\sn{oa2}$ at $T_2$, all with the agent \sn{Amir}, who is an \sn{EligiblePerson} and has a proof of address $\sn{pr}$. Table~\ref{tab:bankproj} shows the projection. The same action \sn{OpenAccount} fires different rules at $T_0$ and at $T_2$, because the state differs, and the last state contains two accounts of Amir, one without and one with a card. Every state is consistent (checked with HermiT), so the projection is completed.
\begin{table}[t]
\centering\small
\caption{Example~\ref{ex:bankproj}: trajectory of the projection. The created individuals are named by the function of Definition~\ref{def:step}, written $\sn{ac}_{e}$ and $\sn{lt}_{e}$ for the individual created by event $e$.}
\label{tab:bankproj}
\begin{tabularx}{\textwidth}{@{}l P{2.5cm} P{4.9cm} X@{}}
\toprule
\textbf{Instant} & \textbf{Event} & \textbf{Rule that fires} & \textbf{Assertions really added}\\ \midrule
$T_0$ & $\sn{oa1}$ (\sn{OpenAccount}) & \texttt{rule\_openAccount}; no letter exists & $\sn{BankAccountWithoutCard}(\sn{ac}_{\sn{oa1}})$, $\sn{has}(\sn{Amir},\sn{ac}_{\sn{oa1}})$\\
$T_1$ & $\sn{gl1}$ (\sn{GetLetter}) & \texttt{rule\_getLetter}; $\sn{Human}(\sn{Amir})$ holds by entailment from \sn{EligiblePerson} & $\sn{Letter}(\sn{lt}_{\sn{gl1}})$, $\sn{has}(\sn{Amir},\sn{lt}_{\sn{gl1}})$\\
$T_2$ & $\sn{oa2}$ (\sn{OpenAccount}) & \texttt{rule\_openAccountwithLetter}; \texttt{rule\_openAccount} is blocked by its side condition because a letter exists & $\sn{BankAccountWithCard}(\sn{ac}_{\sn{oa2}})$, $\sn{has}(\sn{Amir},\sn{ac}_{\sn{oa2}})$\\
\bottomrule
\end{tabularx}
\end{table}
\end{example}

\subsection{Plans and contradictory stories}
\label{sec:plans}

How do we decide whether a sequence of actions is a plan? We pretend that the actions were already scheduled in the TDKG itself, and run the temporal projection. An \emph{action instance} $a=\langle r,\beta\rangle$ is a rule $r$ together with values $\beta$ for its participants. A \emph{candidate plan} $\pi=\langle a_1,\dots,a_n;\Gamma\rangle$ is a sequence of action instances together with a finite set $\Gamma$ of \emph{assumptions}, that is, assertions that the plan takes for granted.

\begin{definition}[Scheduling]\label{def:schedule}
Let $\KB_0$ be a consistent KB. The \emph{scheduled KB} $\KB_0^\pi$ is obtained by adding to $\KB_0$ (i) the assumptions $\Gamma$; (ii) for each $a_i=\langle r_i,\beta_i\rangle$ a fresh event $e_i$ and a fresh instant $\tau_i$, with the event atoms of $r_i$ instantiated by $\beta_i$ (for example $\sn{Act\_TurnKey}(e_i)$, $\sn{hasTime}(e_i,\tau_i)$, $\sn{EventHasUndergoer}(e_i,\sn{FrontDoorLock})$); and (iii) $\sn{TimeInstant}(\tau_i)$ together with $\sn{before}(\tau_{i-1},\tau_i)$, and $\sn{before}(t,\tau_1)$ for every instant $t$ of $\KB_0$.
\end{definition}

Thus the plan's instants come after all instants that $\KB_0$ already has; if $\KB_0$ has $p$ instants with events attached to them (the usual case is $p=0$, $\KB_0$ being the current state), the projection first plays these, then the plan. For readability we take $p=0$ below; for $p>0$ all indices shift by $p$. We write $\KB_0^\pi,\dots,\KB_m^\pi$ for the trajectory of $\textsc{Project}(\KB_0^\pi,\mathrm{Timeline}(\KB_0^\pi))$, with $\AB_i^\pi$ for the ABox of the $i$-th state. An instance $a_i$ is \emph{effective} if its rule has a solution at $\tau_i$ in $\KB_{i-1}^\pi$ whose participants are $\beta_i$, that is, if the action really fires.

\paragraph{Contradiction patterns and when they are reached.} For adversarial synthesis the goal is not a pattern to be \emph{achieved} but a contradiction to be \emph{reached}. A \emph{contradiction pattern} for a consistent terminology $\TB$ is a pattern $\gamma$ (a set of possibly negated atoms with variables) such that $\langle\TB,\theta(\gamma)\rangle$ is inconsistent for some instantiation $\theta$ of the variables by individual names. For disjoint classes $C_1,C_2$ it is $\{C_1(x),C_2(x)\}$; for any class $C$ the pattern $\{C(x),\neg C(x)\}$ is a contradiction pattern as well. Section~\ref{sec:adversarial} shows where patterns come from.

The question is when a projection that ends in an inconsistent state has reached \emph{the chosen} pattern $\gamma$, and not another contradiction. Once a state is inconsistent it entails everything, so we cannot simply ask whether it entails $\theta(\gamma)$. Instead we look at the \emph{consistent parts} of the final ABox.

\begin{definition}[Brave entailment, pattern reached]\label{def:reached}
For a KB $\KB=\langle\TB,\AB\rangle$, possibly inconsistent, an assertion $\alpha$ is \emph{bravely entailed}, written $\Brave(\KB,\alpha)$, if $\alpha\in\AB$ or some subset $S\subseteq\AB$ with $\langle\TB,S\rangle$ consistent entails $\alpha$. Let $\Gamma$ be a set of assumptions. A pattern $\gamma$ is \emph{reached} in $\KB$ (without using $\Gamma$) if there is an instantiation $\theta$ of the variables of $\gamma$ by individuals of $\KB$ such that every atom $\alpha\in\theta(\gamma)$ satisfies $\alpha\notin\Gamma$ and $\Brave(\KB,\alpha)$.
\end{definition}

In words: each atom of the pattern instance must be stated in the ABox or be supported by some consistent part of it. The atoms may come from different parts; that is exactly what a contradiction is. (The first alternative, ``stated'', is needed for contradictions that already hide in a single assertion, such as $R(a,a)$ for an irreflexive $R$.) The requirement $\alpha\notin\Gamma$ forbids the degenerate way of ``reaching'' a contradiction by simply assuming one of its atoms. Brave entailment is decidable by trying the finitely many subsets of $\AB$, with the contract operations $\op{Copy}$, $\op{Remove}$, $\op{Consistent}$ and $\op{Entails}$.

\begin{example}\label{ex:foreign}
Let $\TB=\{A\sqcap B\sqsubseteq\bot,\ C\sqcap D\sqsubseteq\bot\}$ and the final ABox $\{C(a),D(a)\}$. The state is inconsistent, but the pattern $\{A(x),B(x)\}$ is \emph{not} reached: $A(a)$ is not entailed by any part of the ABox. The pattern $\{C(x),D(x)\}$ is reached with $x=a$, because $\{C(a)\}$ and $\{D(a)\}$ are consistent and entail $C(a)$ and $D(a)$. A test that only checks that the final state is inconsistent cannot tell these two patterns apart.
\end{example}

\begin{definition}[Valid plan, contradictory story]\label{def:valid}
Let $G$ be a pattern (the goal; its variables are existential) and $\gamma$ a contradiction pattern. A candidate plan $\pi$ with $n$ actions is
\begin{itemize}
\item a \emph{valid hypothetical plan} for $\langle\KB_0,G\rangle$ if (V1) the projection of $\KB_0^\pi$ is completed, i.e.\ all of $\KB_0^\pi,\dots,\KB_n^\pi$ are consistent, (V2) every $a_i$ is effective, and (V3) $G$ holds in $\KB_n^\pi$; it is a \emph{valid plan}, or simply a \emph{plan}, if $\Gamma=\emptyset$;
\item a \emph{contradictory story for $\gamma$} if (S1) $\KB_0^\pi,\dots,\KB_{n-1}^\pi$ are consistent, (S2) every $a_i$ is effective, (S3) $\KB_n^\pi$ is inconsistent, and (S4) $\gamma$ is reached in $\KB_n^\pi$ without using $\Gamma$.
\end{itemize}
\end{definition}

A valid hypothetical plan works in the hypothetical world $\KB_0\cup\Gamma$ (which is consistent by (V1)); $\Gamma$ is what remains to be verified in the real world. A story is a plan whose last action breaks the graph in the way $\gamma$ describes; no goal is involved. Both notions are decided by one temporal projection and a few checks on its trajectory (Algorithm~\ref{alg:classify} below).

\paragraph{Bounded problems.} To obtain decidability we bound the plan length by $d$ and the vocabulary of plans as follows. Let $v$ be the maximal number of variables of a rule body. The allowed individuals are those of $\KB_0$, the constants of the rules and the goal, and $d\cdot v+|\mathrm{vars}(G)|$ \emph{placeholder names} (used when a participant or goal variable has no known value; they model entities that the plan introduces). The allowed assumptions are the assertions over these individuals and over the concept expressions occurring in $\TB$, the rules and $G$ (closed under negation); call this finite set $\Omega_d$. The \emph{candidate space} $\mathrm{Cand}_d$ consists of all $\langle a_1,\dots,a_m;\Gamma\rangle$ with $m\le d$, participants taken from the allowed individuals and $\Gamma\subseteq\Omega_d$. Using any other fresh name instead of a placeholder changes nothing, because entailment is invariant under renaming of individuals that do not occur in $\TB$. The \emph{bounded planning problem} asks for all valid plans in $\mathrm{Cand}_d$, and the \emph{bounded story problem} for all stories in $\mathrm{Cand}_d$.

\subsection{Ramification and qualification}
\label{sec:ramification}

Let us say precisely what we claim. When an action is executed, the ontology may entail consequences that no rule mentions. We say that the update \emph{solves the ramification problem} if it has the following properties, which we then prove.

\begin{enumerate}[label=(R\arabic*)]
\item \emph{Constraint closure:} every model of the new state satisfies $\TB$.
\item \emph{Direct effects and inertia:} inserted assertions hold; stated assertions not targeted by a deletion persist.
\item \emph{Indirect effects by entailment:} everything else that changes is a consequence of the direct changes, computed by the reasoner; nothing changes without a cause.
\item \emph{Self-containedness:} no list of indirect effects, occlusion set or causal rule is an input.
\item \emph{Qualification detection:} if the effects contradict $\TB$, the new state is inconsistent and this is detected, not silently repaired.
\end{enumerate}

For a step from $\KB=\langle\TB,\AB\rangle$ to $\KB'=\langle\TB,\AB'\rangle$ with deletions $D$ and insertions $I$, let $\AB^{+}=I\setminus\AB$ be the assertions \emph{really added} and $\AB^{-}=(D\cap\AB)\setminus I$ those \emph{really removed}.

\begin{theorem}[Ramification]\label{thm:ramif}
$\Step$ satisfies (R1)--(R5). More precisely, $\AB'=(\AB\setminus\AB^{-})\cup\AB^{+}$, $I\subseteq\AB'$ and $\AB\setminus D\subseteq\AB'$, and if $\KB'$ is consistent then for every assertion $\varphi$:
\begin{enumerate}[label=(\alph*)]
\item if $\KB\nentails\varphi$ and $\KB'\models\varphi$, then $\AB^{+}\neq\emptyset$ and $\langle\TB,\AB\setminus\AB^{-}\rangle\nentails\varphi$;
\item if $\KB\models\varphi$ and $\KB'\nentails\varphi$, then every minimal subset $M\subseteq\AB$ with $\langle\TB,M\rangle\models\varphi$ (a \emph{justification}) contains an assertion of $\AB^{-}$;
\item if $\AB^{+}=\AB^{-}=\emptyset$, then $\KB$ and $\KB'$ entail the same assertions; if only $\AB^{-}=\emptyset$, then $\KB'$ entails everything $\KB$ entails.
\end{enumerate}
\end{theorem}
\begin{proof}
$\AB'=(\AB\setminus D)\cup I$ gives the formulas for $\AB'$, $I$ and $\AB\setminus D$. (R1) holds because $\TB$ is part of $\KB'$; (R4) because $\Step$ is a function of $\KB$, $\RS$ and $t$; (R5) because $\KB'$ is a KB with terminology $\TB$ and $\op{Consistent}$ decides it (B3). (a) If $\AB^{+}=\emptyset$ then $\AB'\subseteq\AB$ and monotonicity gives $\KB\models\varphi$ whenever $\KB'\models\varphi$. If $\langle\TB,\AB\setminus\AB^{-}\rangle\models\varphi$ then $\KB\models\varphi$ by monotonicity. (b) If a justification $M$ avoided $\AB^{-}$ then $M\subseteq\AB\setminus\AB^{-}\subseteq\AB'$ and $\KB'\models\varphi$. (c) If $\AB^{+}=\AB^{-}=\emptyset$ then $\AB'=\AB$; if $\AB^{-}=\emptyset$ then $\AB\subseteq\AB'$. These cases give (R3).
\end{proof}

Part (b) is the exact meaning of \emph{indirect negative effects}: an entailed fact disappears iff every one of its supports loses a member. For a defined concept the statement is even simpler.

\begin{theorem}[Definitions]\label{thm:definitional}
Suppose $\TB$ contains a definition $P\equiv\Phi$ and no rule template mentions $P$. Then in every consistent state $\KB_i^\pi$ of every trajectory and for every individual $a$: $\KB_i^\pi\models P(a)$ iff $\KB_i^\pi\models\Phi(a)$. Consequently $P(a)$ becomes (resp.\ ceases to be) entailed exactly when $\Phi(a)$ does, and by Theorem~\ref{thm:ramif} this requires a real addition (resp.\ removal). New definitions can be added to $\TB$ without changing any rule.
\end{theorem}
\begin{proof}
$\TB\models P\equiv\Phi$, and every model of $\KB_i^\pi$ is a model of $\TB$. The rest follows from Theorem~\ref{thm:ramif}(a),(b).
\end{proof}

If the definitions are acyclic, $\Phi$ can be unfolded until it mentions only primitive symbols, so the indirect effects on all defined concepts are determined by the changes of primitive assertions.

\begin{proposition}[What is not covered]\label{prop:scope}
(i) \emph{Integrity constraints are not ramifications.} For $\TB=\{A\sqcap B\sqsubseteq\bot\}$, $\AB=\{A(a)\}$ and a rule inserting $B(a)$, the new state is inconsistent. Two consistent repairs, ``replace $A(a)$ by $B(a)$'' and ``refuse the action'', are equally compatible with $\TB$; the choice is a domain policy that $\TB$ cannot supply. The modeller expresses it by deleting $A(a)$ in the rule, as the rule \texttt{ImportWheat} does for \sn{TaxPayer}.
(ii) \emph{Derived facts cannot be deleted directly.} For $\TB=\{B\sqsubseteq C\}$, $\AB=\{B(a)\}$ and a rule deleting $C(a)$, nothing changes and $C(a)$ is still entailed. It disappears only when all its justifications are cut (Theorem~\ref{thm:ramif}(b)).
(iii) \emph{The update is not model-based.} For $\TB=\emptyset$, $\AB=\{(P\sqcup Q)(a)\}$ and a rule inserting $\neg P(a)$, the new state entails $Q(a)$, whereas the possible-models update of Winslett \cite{winslett1988reasoning,liu2006updating} does not: it moves the model with $P$ true and $Q$ false to the closest model with $P$ false, in which $Q$ is false as well.
\end{proposition}
\begin{proof}
(i) $\{A(a),B(a)\}$ contradicts the disjointness. (ii) $D\cap\AB=\emptyset$, so $\AB'=\AB$. (iii) $\{(P\sqcup Q)(a),\neg P(a)\}\models Q(a)$, while the possible-models update flips only $P$.
\end{proof}

Thus our update belongs to the family of \emph{syntactic} updates with indirect effects defined by entailment, in contrast to model-based updates and to causal approaches \cite{lin1995embracing,mccain1995causal} in which the modeller supplies causal rules. What is solved, and proved, is all of: definitional indirect effects, upward propagation along subsumptions, domains, ranges and role inclusions, and the detection of integrity violations.

\subsection{Dynamics without modal operators, possible worlds or branching time}
\label{sec:logic-position}

The treatment of change in this paper differs from that of the formalisms of Section~\ref{sec:related} in three logical choices. They are choices, with consequences in both directions, and we state both.

\paragraph{No modal operator.} Dynamic description logics \cite{chang2007dynamic,shi2005logical} put actions into the language: a formula $\langle\alpha\rangle\varphi$ or $[\alpha]\varphi$ says that $\varphi$ holds after some or every execution of $\alpha$, and the semantics is a family of interpretations connected by an accessibility relation, in the tradition of dynamic logic \cite{harel2000dynamic}. In a TDKG the language is $\SROIQ$ and nothing is added to it. An action is not an object of the language; it is a rule, and its effect is the function $\Step$ from a knowledge base to a knowledge base (Theorem~\ref{thm:transitions}). That $\varphi$ holds after the events of the timeline is the ordinary entailment $\KB_n\models\varphi$ for the last state $\KB_n$, and $\varphi$ is any $\SROIQ$ assertion, with no operator in front of it. Decidability of projection is then a corollary of the decidability of $\SROIQ$: projection makes finitely many calls to the reasoner, each of which is decidable, and the bound on the number of calls is explicit (Theorem~\ref{thm:transitions}). No new decidability argument is needed, and none of the usual restrictions that keep dynamic logics decidable (to atomic actions, to a sublogic, to a fragment of the modalities) applies to the logic. The decidability of bounded planning follows in the same way (Theorem~\ref{thm:apt-decid}), and its complexity is that of $\SROIQ$ consistency up to the exponent (Theorems~\ref{thm:apt-hard} and~\ref{thm:apt-upper}). The price is expressiveness: the language cannot say that $\varphi$ holds after \emph{every} execution of an action sequence, or after \emph{some} sequence. Such statements are questions about plans; they are answered outside the logic, by the search of \APT{} within the bound $d$, and they are not theorems of an object language.

\paragraph{One linear forward timeline.} A TDKG has instants $t_0<t_1<\dots<t_n$, and each instant has exactly one state. The events of an instant determine the next state, because $\Step$ is a function. There is no choice of successor in the semantics, so the set of states of a TDKG is a line and not a tree. Branching-time logics such as CTL$^{*}$ and the $\mu$-calculus, which are used to verify knowledge and action bases \cite{hariri2013verification,bagheri2013dlkab}, interpret formulas over a tree of possible executions \cite{emerson1986sometimes,vardi2001branching}; the situation calculus \cite{mccarthy1969some,reiter2001knowledge} has as its domain the tree of all action sequences, with the initial situation as its root. A linear timeline is a different philosophical stance: it records \emph{what happens}, whereas a tree describes \emph{what could happen}. The tree does not vanish from the problem, but it moves. The planner searches among alternative futures, each candidate plan being one branch, and it evaluates a branch by projection, that is, by walking along one line. Branching is in the search and not in the semantics. Linear-time temporal logic \cite{pnueli1977temporal} is also interpreted over one path, but it adds temporal operators to talk about it; a TDKG has no temporal operators either, and an instant is addressed by naming it, as the index of a state.

\paragraph{No possible worlds.} The update of a state is a knowledge base, not a set of models. It uses no accessibility relation between interpretations and no ordering of models by closeness. Baader et al.\ \cite{baader2005integrating} and the dynamic logics define the effect of an action as a relation between interpretations; the model-based updates of Winslett and of Liu et al.\ \cite{winslett1988reasoning,liu2006updating} select the closest models. Proposition~\ref{prop:scope}(iii) shows the difference on a three-line example. We do not claim that no other formalism avoids possible worlds: knowledge and action bases, for example, also update the assertions of the ABox directly, and what we say concerns the formalisms that we read for Section~\ref{sec:related}. Facts that are unknown are not alternatives between worlds that the system carries through time. The open world of $\SROIQ$ leaves them undecided in the state, and when a plan needs them, they become explicit \emph{assumptions} that the plan lists and that \DDA{} explains by further assertions. The alternatives that remain are alternative explanations, each a finite set of assertions with its own sub-plan.

\begin{table}[t]
\centering\small
\caption{Three logical choices, compared with the formalisms of Section~\ref{sec:related} as we read them. A formalism is listed by the property that it is characterised by, not by every variant.}
\label{tab:position}
\newcolumntype{P}[1]{>{\raggedright\arraybackslash}p{#1}}
\begin{tabular}{@{}P{3.6cm}P{2.6cm}P{3.2cm}P{4.0cm}@{}}
\toprule
Formalism & Modal or temporal operators & Futures & Effect of an action\\ \midrule
Dynamic DL \cite{chang2007dynamic,shi2005logical} & action modalities & relational, branching & relation between interpretations\\
Action formalism of Baader et al.\ \cite{baader2005integrating} & none (actions outside the logic) & one per action sequence, with occlusion & relation between interpretations, constrained by an occlusion set\\
Model-based update \cite{winslett1988reasoning,liu2006updating} & none & one set of closest models & a set of models\\
Knowledge and action bases \cite{bagheri2013dlkab,hariri2013verification} & in the property language ($\mu$-calculus) & branching transition system & update of the ABox\\
Situation calculus \cite{reiter2001knowledge} & none, situations are terms & tree of situations & successor-state axioms\\
TDKG (this paper) & none & one line, branching only in plan search & update of the ABox; consequences by entailment\\
\bottomrule
\end{tabular}
\end{table}

\paragraph{What is lost.} The three choices exclude what needs them. There are no nondeterministic actions: an action with several possible outcomes would need a set of successor states, which is a branch in the semantics, and a rule has exactly one result. There is no verification of a property over all futures of a system, which is the purpose of the branching-time approaches. There is no epistemic modality and no sensing: what the plan knows is not part of the state. The tools of this paper answer a different question. Given one initial knowledge base, a timeline and a goal, they return what happens, what must be assumed, and what to do; the answer is a line, and every step of it is decided by a $\SROIQ$ reasoner.

\section{\DDA: Abduction by Direct Derivation}
\label{sec:dda}

\newcommand{\jn}{\sn{jane}}
\newcommand{\tz}{\sn{tarzan}}
\newcommand{\hC}{\sn{hasChild}}
\newcommand{\hP}{\sn{hasParent}}
\newcommand{\MW}{\sn{Man}\sqcup\sn{Woman}}

\subsection{The problem}

An \emph{observation} is an assertion that we want to see entailed by the knowledge base, for example $\sn{Mother}(\jn)$. Given a consistent KB $\KB$ and a finite set $O$ of observations, abduction looks for a set $H$ of assertions, the \emph{hypothesis}, such that $\KB$ together with $H$ explains the observations. Observations that $\KB$ already entails need no explanation; let $O'\subseteq O$ be those that it does not entail. The standard criteria \cite{pukancova2020,koopmann2021signature,glimm2022_concept} are the following.
\begin{description}
\item[(E) Explanation:] $\KB\cup H\models o$ for every $o\in O'$.
\item[(C) Consistency:] $\KB\cup H$ is consistent.
\item[(M) Minimality:] no proper subset of $H$ satisfies (E).
\item[(R) Relevance:] $H$ does not merely restate an observation that could be explained by something deeper. This is made precise in Section~\ref{sec:relevance}.
\end{description}
If $O'=\emptyset$ there is nothing to explain and the empty hypothesis is returned. Unlike hitting-set solvers, \DDA{} does not fix a vocabulary of allowed hypotheses in advance: an explanation may use any assertion over the finite space $\Omega$ of Section~\ref{sec:universe}, including complex concepts such as $(\exists\hC.(\MW))(\jn)$ (``$\jn$ has some child that is a man or a woman'').

\subsection{Suggesting explanations for one observation}
\label{sec:suggest}

The key idea fits in one sentence: \emph{to find what would explain $s$, assume $\neg s$, see what the reasoner newly derives, and negate it.} Nothing is refuted in this process. The negation of $s$ is only a device: whatever it forces to be true, once negated again, is something that would conclude $s$. We therefore call the result the \emph{suggestions} for $s$.

\begin{definition}[Suggestions]\label{def:sugg}
Let $s$ be an assertion with $\KB\nentails s$. The set of \emph{suggestions} for $s$ is
\[
\Sugg(s)=\{\neg\delta\mid \delta\in\Omega,\ \KB\cup\{\neg s\}\models\delta\ \text{ and }\ \KB\nentails\delta\}.
\]
These are the negations of the consequences of $\neg s$ that are \emph{new}, that is, not already entailed by $\KB$. If $\KB\models s$ we put $\Sugg(s)=\emptyset$.
\end{definition}

The set of new consequences is what an incremental reasoner reports after $\op{Add}(\KB',\{\neg s\})$ (ability B5 of Section~\ref{sec:backend}); without that ability it is computed by $2|\Omega|$ entailment checks. Note that $s$ itself is a suggestion for $s$ (take $\delta=\neg s$): $s$ explains itself, but not in an interesting way. This is the reason for the notion of relevance below.

\begin{lemma}[Suggestions conclude the observation]\label{lem:sugg}
Every $a\in\Sugg(s)$ satisfies (i) $\KB\cup\{a\}\models s$, (ii) $\KB\cup\{a\}$ is consistent, and (iii) $\KB\nentails a$.
\end{lemma}
\begin{proof}
Let $a=\neg\delta$. (i) From $\KB\cup\{\neg s\}\models\delta$ and Lemma~\ref{lem:neg}(b) we get $\KB\cup\{\neg\delta\}\models s$. (ii) $\KB\nentails\delta$, so $\KB\cup\{\neg\delta\}$ is consistent by Lemma~\ref{lem:neg}(a). (iii) If $\KB\models a$ then $\KB\cup\{a\}$ is equivalent to $\KB$, and (i) would give $\KB\models s$.
\end{proof}

\begin{lemma}[One step finds every single-assertion explanation]\label{lem:one-step}
Let $\KB\nentails s$ and let $a\in\Omega$ be an assertion such that $\KB\cup\{a\}$ is consistent and entails $s$. Then $a\in\Sugg(s)$.
\end{lemma}
\begin{proof}
By Lemma~\ref{lem:neg}(b), $\KB\cup\{\neg s\}\models\neg a$. Moreover $\KB\nentails\neg a$, since otherwise $\KB\cup\{a\}$ would be inconsistent. So $\delta=\neg a$ qualifies in Definition~\ref{def:sugg} and $a=\neg\delta$.
\end{proof}

Lemma~\ref{lem:one-step} is the main contrast with hitting-set abduction: \emph{every} consistent single-assertion explanation over $\Omega$ is obtained from \emph{one} reasoning step, without enumerating candidates.

\begin{example}\label{ex:sugg-mother}
Let $\TB$ contain $\sn{Mother}\equiv\sn{Parent}\sqcap\sn{Woman}$ and $\sn{Grandmother}\sqsubseteq\sn{Mother}$, and let $s=\sn{Mother}(\jn)$. After adding $\neg\sn{Mother}(\jn)$ the reasoner newly derives, besides $\neg\sn{Mother}(\jn)$ itself, the assertions $\neg\sn{Grandmother}(\jn)$ (a grandmother would be a mother) and $\neg(\sn{Parent}\sqcap\sn{Woman})(\jn)$ (the definition). Negating them gives $\Sugg(s)=\{\sn{Mother}(\jn),\ \sn{Grandmother}(\jn),\ (\sn{Parent}\sqcap\sn{Woman})(\jn)\}$, and each of these concludes $\sn{Mother}(\jn)$.
\end{example}

\subsection{Reducible observations and relevance}
\label{sec:relevance}

Some observations can be explained by something deeper (like $\sn{Mother}(\jn)$ above), and some cannot: the reasoner has nothing to say about what would make $\sn{isInstalledOn}(\sn{FrontDoorLock},\sn{dor})$ true, so its only suggestion is itself. We say that an observation $s$ is \emph{reducible} if $\Sugg(s)$ contains an assertion other than $s$, and \emph{irreducible} otherwise. Splitting conjunctions (Section~\ref{sec:grow}) gives the notion of parts: $\Dec(a)$ is the set of conjuncts of $a$, and $\Dec(a)=\{a\}$ if $a$ is not a conjunction.

\begin{definition}[Restating, relevance]\label{def:relevant}
A hypothesis $H$ \emph{restates} an observation $o$ if $o\in H$, or if $o$ is a conjunction and $\Dec(o)\subseteq H$. $H$ is \emph{relevant}, criterion (R), if it restates no reducible observation of $O'$.
\end{definition}

The purpose of (R) is to exclude explanations that explain a thing by itself. Without (R), $H=\{\sn{Mother}(\jn)\}$ would be a perfectly good minimal explanation of $O=\{\sn{Mother}(\jn)\}$: it entails the observation, it is consistent, and no proper subset works. It is useless, because it tells us nothing we did not know. With (R) this $H$ is rejected, and the hypotheses $\{\sn{Parent}(\jn),\sn{Woman}(\jn)\}$ and $\{\sn{Grandmother}(\jn)\}$ remain. The restriction to \emph{reducible} observations is needed for the other direction: an irreducible observation has no deeper explanation, so the only thing one can say about it is that it has to be assumed, and the hypothesis may contain it. This is exactly what the planner needs in the door example: $\{\sn{HingedStructure}(\sn{dor}),\sn{WoodenStructure}(\sn{dor}),\sn{isInstalledOn}(\sn{FrontDoorLock},\sn{dor})\}$ explains the two assumptions $\sn{SecureDoor}(\sn{dor})$ (reducible, hence replaced by its parts) and $\sn{isInstalledOn}(\ldots)$ (irreducible, hence kept).

\subsection{Making hypotheses deeper}
\label{sec:grow}

One round of suggestions gives explanations that are single assertions. Longer causal chains (an explanation of an explanation) and explanations with several parts are obtained by repeating the step on the parts of a hypothesis. Two further devices are needed.

\paragraph{Splitting conjunctions.} A suggestion may be a conjunction, say $(\sn{Parent}\sqcap\sn{Woman})(\jn)$. We split it into its parts $\sn{Parent}(\jn)$ and $\sn{Woman}(\jn)$ and treat the parts as separate atoms of the hypothesis, so that each of them can be explained further on its own. Splitting does not change the models.

\paragraph{Grounding existentials: from ``some'' to ``this one''.} A suggestion may be an existential atom such as $(\exists\hC.(\MW))(\jn)$, which says that $\jn$ has \emph{some} child who is a man or a woman, without saying who. A hypothesis is more useful, and can be explained further, if it names the child, for instance $\hC(\jn,\tz)$. \emph{Grounding} replaces the existential atom by a concrete assertion $R(x,b)$ for an individual $b$ that can serve as the witness. The individual $b$ is acceptable as a witness if the rest of the hypothesis, together with $\KB$, already tells us that $b$ belongs to the concept in the scope of the existential. This is the sense in which the other atoms of the hypothesis provide a \emph{context}.

\begin{definition}[Witnesses]\label{def:witness}
Let $H$ be a hypothesis and $s=(\exists R.C)(x)\in H$. An individual $b$ is a \emph{witness} for $s$ in $H$ if $\KB\cup(H\setminus\{s\})$ is consistent and entails $C(b)$. The \emph{grounding} of $s$ by $b$ is the assertion $R(x,b)$ (for an inverse role $R=S^{-}$ it is $S(b,x)$).
\end{definition}

Here are two small cases from the family ontology of Section~\ref{sec:dda-ex}, where $\KB$ contains $\sn{Man}(\tz)$.
\begin{itemize}
\item For $s=(\exists\hC.(\MW))(\jn)$ in the hypothesis $H=\{\sn{Woman}(\jn),s\}$, the individual $\tz$ is a witness (since $\KB$ entails $\sn{Man}(\tz)$, hence $(\MW)(\tz)$), and so is $\jn$ herself (since $H\setminus\{s\}$ contains $\sn{Woman}(\jn)$). The groundings are $\hC(\jn,\tz)$ (``$\jn$ is the mother of $\tz$'') and $\hC(\jn,\jn)$ (``$\jn$ is her own parent''); the second one is allowed in scenario~a and is inconsistent in scenario~b, where $\hC$ is irreflexive.
\item For $s=(\exists\hP.(\MW))(\tz)$ in $H=\{\sn{Woman}(\jn),s\}$, the individual $\jn$ is a witness only \emph{because} of the context $\sn{Woman}(\jn)$. In the empty context, nothing is known about the gender of $\jn$, and $\jn$ is not a witness. The grounding is $\hP(\tz,\jn)$, i.e.\ ``$\jn$ is a parent of $\tz$''. Thus, when several observations are explained together, an atom that was introduced for one observation can supply the witness for another one (Section~\ref{sec:multi}).
\end{itemize}
Both the existential atom and its groundings are kept as candidates: the grounded atom is logically stronger (it implies the existential atom in the context), the existential atom is weaker and may be explained by a different witness later. Which of them survives into the output is decided by the minimality and consistency checks.

\paragraph{Successors.} We can now say how a hypothesis grows. A \emph{successor} of a hypothesis $H$ is obtained in two stages.
\begin{description}
\item[Stage A (replace or keep).] For every atom $s\in H$ choose independently either $s$ itself or the parts $\Dec(a)$ of one suggestion $a\in\Sugg(s)$. Let $H^A$ be the union of the choices.
\item[Stage B (ground).] For every existential atom $s=(\exists R.C)(x)$ of $H^A$ choose independently either $s$ itself or the grounding $R(x,b)$ by one witness $b$ for $s$ in $H^A$. Let $H^B$ be the union of the choices.
\end{description}
The successors of $H$ are all hypotheses $H^B$ that can arise this way. Starting from $H_0=\Dec(O')$, the set $\mathrm{Reach}_L$ of hypotheses \emph{reachable in at most $L$ rounds} contains $H_0$ and all successors of reachable hypotheses of the previous rounds. Finally, for a hypothesis $H$ let $\mathrm{Min}(H)$ be the set of subset-minimal $H'\subseteq H$ with $\KB\cup H'\models O'$ (it is empty if $\KB\cup H$ is inconsistent or does not entail $O'$). The output of \DDA{} is
\[
\mathit{Out}_L=\{H'\mid H'\in\mathrm{Min}(H)\ \text{for some}\ H\in\mathrm{Reach}_L,\ \text{and}\ H'\ \text{is relevant}\}.
\]
Algorithms~\ref{alg:dda}--\ref{alg:minimise} below give the pseudocode of all steps.

\subsection{Several observations and the combination of suggestions}
\label{sec:multi}

When there are several observations $O'=\{o_1,\dots,o_n\}$, \DDA{} does not explain each observation separately and combine the results afterwards; it explains them \emph{jointly}. The initial hypothesis $H_0=\Dec(o_1)\cup\dots\cup\Dec(o_n)$ contains the parts of all observations, and each round treats all of its atoms. The combination of suggestions happens at four places.
\begin{enumerate}
\item \emph{Independent choices.} Stage~A chooses for every atom independently, so the successors of $H_0$ are the \emph{Cartesian product} of the choices for the individual observations: if $o_1$ has three choices and $o_2$ has two, stage~A produces six combined hypotheses. This is the combination of ``one explanation per observation''.
\item \emph{Shared context.} In stage~B the witnesses for an existential atom of one observation are computed in the context of the atoms of the other observations (Definition~\ref{def:witness}). Hence explanations of different observations can \emph{support} each other, as for the parent of $\tz$ above.
\item \emph{Synergy and minimality.} Minimality is tested on the \emph{joint} hypothesis against the \emph{whole} set $O'$. If one atom serves two observations, or if a part becomes redundant given another part, the redundant atoms are dropped. For example $\{\hC(\jn,\tz),\sn{Woman}(\jn)\}$ explains both $\sn{Mother}(\jn)$ and $\sn{Child}(\tz)$ with a single role atom, because $\hC$ and $\hP$ are inverse.
\item \emph{Consistency.} A combination whose atoms contradict each other (for instance $\sn{Father}(\jn)$ for one observation and $\sn{Woman}(\jn)$ for another) is discarded as a whole by (C); the other combinations are unaffected.
\end{enumerate}
The next statement says that nothing is lost by this joint treatment: whenever each observation has an explanation chain of its own, and the chain ends are mutually consistent, then their union is found.

\begin{proposition}[Combination of explanations]\label{prop:combine}
Let $O'=\{o_1,\dots,o_n\}$ and suppose that for every $i$ there are assertions $a^i_0=o_i,a^i_1,\dots,a^i_{\ell_i}$ that are not conjunctions, with $a^i_{j+1}\in\Sugg(a^i_j)$ for $j<\ell_i$. If $\KB\cup\{a^1_{\ell_1},\dots,a^n_{\ell_n}\}$ is consistent, then $H=\{a^1_{\ell_1},\dots,a^n_{\ell_n}\}\in\mathrm{Reach}_L$ for every $L\ge\max_i\ell_i$, $\KB\cup H\models O'$, and every relevant subset-minimal $H'\subseteq H$ with $\KB\cup H'\models O'$ belongs to $\mathit{Out}_L$.
\end{proposition}
\begin{proof}
By Lemma~\ref{lem:sugg}(i) and transitivity, $\KB\cup\{a^i_{j}\}\models o_i$ for all $j$, so $\KB\cup H\models O'$. In round $j$, stage~A may replace $a^i_{j-1}$ by $a^i_j$ for each $i$ with $j\le\ell_i$ and keep the other atoms, because ``keep $s$'' is always a choice; stage~B may keep every atom. After $\max_i\ell_i$ rounds the chain ends form one hypothesis, which therefore lies in $\mathrm{Reach}_L$. It is consistent by assumption, so $\mathrm{Min}(H)$ is defined, and the relevant members of $\mathrm{Min}(H)$ are in $\mathit{Out}_L$ by definition.
\end{proof}

\subsection{Algorithms}
\label{sec:dda-algs}

The procedure is organised as five algorithms, with a sixth variant of the last one in Section~\ref{sec:fastmin}. Algorithm~\ref{alg:dda} is the main loop; it calls \textsc{Suggest} (Algorithm~\ref{alg:suggest}) to obtain suggestions, \textsc{Successors} (Algorithm~\ref{alg:succ}) to grow a hypothesis by one round, \textsc{Witnesses} (Algorithm~\ref{alg:wit}) for grounding inside \textsc{Successors}, and finally \textsc{Minimise} (Algorithm~\ref{alg:minimise}) to extract the minimal relevant explanations. The space $\Omega$ and the individual set $\Ind$ are computed once from $\KB$ and $O$.

\begin{algorithm}[t]
\caption{$\textsc{DDA}(\KB,O,L)$: main loop}\label{alg:dda}
\begin{algorithmic}[1]
\Require consistent $\KB$; observations $O$; number of rounds $L$ (possibly $\infty$)
\Ensure the set $\mathit{Out}_L$ of explanations
\State $O'\gets\{o\in O\mid\neg\op{Entails}(\KB,o)\}$; \textbf{if} $O'=\emptyset$ \textbf{then return} $\{\emptyset\}$
\State compute $\Omega$ and $\Ind$; $\ \mathit{Red}\gets\{o\in O'\mid\textsc{Suggest}(\KB,o)\ne\{o\}\}$ \Comment{reducible observations}
\State $H_0\gets\bigcup_{o\in O'}\Dec(o)$; $\ \mathit{Seen}\gets\{H_0\}$; $\ \mathit{Front}\gets\{H_0\}$; $\ \ell\gets0$
\While{$\mathit{Front}\ne\emptyset$ \textbf{and} $\ell<L$}
  \State $\ell\gets\ell+1$; $\ \mathit{Next}\gets\emptyset$
  \ForAll{$H\in\mathit{Front}$} $\mathit{Next}\gets\mathit{Next}\cup\textsc{Successors}(\KB,H)$ \EndFor \Comment{one round}
  \State $\mathit{Front}\gets\mathit{Next}\setminus\mathit{Seen}$; $\ \mathit{Seen}\gets\mathit{Seen}\cup\mathit{Front}$
\EndWhile
\State \Return $\bigcup_{H\in\mathit{Seen}}\textsc{Minimise}(\KB,O',\mathit{Red},H)$
\end{algorithmic}
\end{algorithm}

\begin{algorithm}[t]
\caption{$\textsc{Suggest}(\KB,s)$ (memoised per $\KB$ and $s$)}\label{alg:suggest}
\begin{algorithmic}[1]
\Require consistent $\KB$, assertion $s$
\Ensure $\Sugg(s)$
\State $\KB'\gets\op{Copy}(\KB)$; $\ \op{Add}(\KB',\{\neg s\})$ \Comment{assume the negation of $s$}
\If{$\neg\op{Consistent}(\KB')$} \Return $\emptyset$ \EndIf \Comment{$\KB\models s$ already}
\State $\Delta\gets\{\delta\in\Omega\mid\op{Entails}(\KB',\delta)\}\setminus\{\delta\in\Omega\mid\op{Entails}(\KB,\delta)\}$ \Comment{new consequences (B5)}
\State \Return $\{\neg\delta\mid\delta\in\Delta\}$ \Comment{negate them}
\end{algorithmic}
\end{algorithm}

\begin{algorithm}[t]
\caption{$\textsc{Successors}(\KB,H)$: one round of growth}\label{alg:succ}
\begin{algorithmic}[1]
\Require consistent $\KB$; hypothesis $H=\{s_1,\dots,s_k\}$
\Ensure the set of successors of $H$ (Section~\ref{sec:grow})
\Statex \textit{Stage A: replace or keep each atom}
\ForAll{$s_j\in H$} $\mathit{Ch}_j\gets\{\{s_j\}\}\cup\{\Dec(a)\mid a\in\textsc{Suggest}(\KB,s_j)\}$ \EndFor
\State $\mathit{StageA}\gets\{c_1\cup\dots\cup c_k\mid c_1\in\mathit{Ch}_1,\dots,c_k\in\mathit{Ch}_k\}$ \Comment{Cartesian product of the choices}
\State $\mathit{Result}\gets\emptyset$
\ForAll{$H^A\in\mathit{StageA}$}
  \Statex \hspace{\algorithmicindent}\textit{Stage B: ground each existential atom}
  \ForAll{$s\in H^A$}
    \If{$s=(\exists R.C)(x)$} $\mathit{Op}_s\gets\{\{s\}\}\cup\{\{R(x,b)\}\mid b\in\textsc{Witnesses}(\KB,H^A,s)\}$ \Else{} $\mathit{Op}_s\gets\{\{s\}\}$ \EndIf
  \EndFor
  \State $\mathit{Result}\gets\mathit{Result}\cup\{\bigcup_{s\in H^A}p_s\mid p_s\in\mathit{Op}_s\text{ for every }s\in H^A\}$
\EndFor
\State \Return $\mathit{Result}$
\end{algorithmic}
\end{algorithm}

\begin{algorithm}[t]
\caption{$\textsc{Witnesses}(\KB,H,s)$ for $s=(\exists R.C)(x)\in H$}\label{alg:wit}
\begin{algorithmic}[1]
\State $\KB'\gets\op{Copy}(\KB)$; $\ \op{Add}(\KB',H\setminus\{s\})$ \Comment{the context of $s$}
\If{$\neg\op{Consistent}(\KB')$} \Return $\emptyset$ \EndIf
\State \Return $\{b\in\Ind\mid\op{Entails}(\KB',C(b))\}$
\end{algorithmic}
\end{algorithm}

\begin{algorithm}[t]
\caption{$\textsc{Minimise}(\KB,O',\mathit{Red},H)$}\label{alg:minimise}
\begin{algorithmic}[1]
\State $\KB_1\gets\op{Copy}(\KB)$; $\ \op{Add}(\KB_1,H)$
\If{$\neg\op{Consistent}(\KB_1)$ \textbf{or} some $o\in O'$ is not entailed by $\KB_1$} \Return $\emptyset$ \EndIf \Comment{(C) and (E)}
\State $\mathit{Mins}\gets\emptyset$
\ForAll{$H'\subseteq H$ by increasing size, \textbf{skipping} every $H'$ that contains a member of $\mathit{Mins}$}
  \State $\KB_2\gets\op{Copy}(\KB)$; $\ \op{Add}(\KB_2,H')$
  \If{$\KB_2$ entails every $o\in O'$} $\mathit{Mins}\gets\mathit{Mins}\cup\{H'\}$ \EndIf \Comment{(M)}
\EndFor
\State \Return $\{H'\in\mathit{Mins}\mid H'\text{ restates no }o\in\mathit{Red}\}$ \Comment{(R), Definition~\ref{def:relevant}}
\end{algorithmic}
\end{algorithm}

All algorithms use the reasoner only through the operations of Section~\ref{sec:backend}. In practice an implementation keeps one reasoner state per search node and derives the state of a successor by a small $\op{Add}$ on a copy; this is where an incremental reasoner makes the algorithm fast. The calls inside the loops of Algorithms~\ref{alg:dda}--\ref{alg:minimise} are independent of each other and can be issued concurrently (Section~\ref{sec:parcache}).

\paragraph{The algorithms on layered states.} If the reasoner provides (B9), the steps of Algorithms~\ref{alg:dda}--\ref{alg:minimise} stay the same, and every $\op{Copy}$ followed by $\op{Add}$ creates a state $(r,G)$ over the main model $\KB$ (Section~\ref{sec:backend}). \textsc{Suggest} creates $(r,G)$ with $G=\{\neg s\}$; the set $\Delta$ of line~3 is then the set of assertions of $\Omega$ that $r$ has derived into $G$, which is read without the $2|\Omega|$ entailment tests (B5). \textsc{Successors} creates for each suggestion $a$ the state with $G_a=\{\neg a\}$; \textsc{Witnesses} creates a state with $G=H\setminus\{s\}$, and \textsc{Minimise} one with $G=H'$. All states share $\KB$ and hold only the few assertions that they added.

\subsection{Cheaper minimality with a derivation record}
\label{sec:fastmin}

Algorithm~\ref{alg:minimise} is exact, but it tests up to $2^{|H|}$ subsets, each by a copy and entailment calls. If the reasoner offers a derivation record (ability B7, Section~\ref{sec:backend}), most of this work can be saved, because the reasoner has already stored \emph{why} each assertion holds. We use the word \emph{suggestion} for what \textsc{Suggest} returns, \emph{hypothesis} for a set of atoms that may be added to the KB, and \emph{explanation} for a hypothesis that satisfies (E), (C), (M) and (R). The \emph{derivation record} is a different thing: it belongs to the reasoner and says why an assertion is entailed. It is used here only to find atoms of a hypothesis that are \emph{redundant}, that is, that follow from the other atoms.

\paragraph{The idea.} Let $H=\{s_1,s_2,s_3\}$ be a hypothesis. Suppose $s_1$ is a minor premise in a derivation of $s_2$, the other minor premises being in $\AB$. Then adding $s_1$ to the KB already yields $s_2$, so $s_2$ is redundant and we keep $\{s_1,s_3\}$. If, conversely, $s_2$ is a minor premise in a derivation of $s_1$, then $s_1$ is redundant given $s_2$ and we keep $\{s_2,s_3\}$. If both hold, the two equivalent variants $\{s_1,s_3\}$ and $\{s_2,s_3\}$ are both returned. No reasoner call was needed to find them: the derivations are read from the record. The definition below makes this precise, and also covers chains ($s_1$ yields $s_2$, $s_2$ yields $s_3$) and derivations that need several atoms of $H$ together.

\begin{definition}[Support, closure, reduction]\label{def:reduction}
Let $H$ be a hypothesis such that $\KB\cup H$ is consistent, and let $x\in H$.
\begin{enumerate}[leftmargin=*,itemsep=1pt,topsep=2pt]
\item A \emph{support} of $x$ is a set $S\subseteq H\setminus\{x\}$ for which there is a derivation $d\in\mathrm{Der}(\KB\cup H,x)$ with $\mathrm{Prem}(d)\cap H=S$ and $x\notin\mathrm{Prem}(d)$. In words: $S$ is the part of $H$ that one derivation of $x$ uses, besides assertions of $\AB$ and the TBox. We write $\mathrm{Sup}(x)$ for the set of supports of $x$. (If $\emptyset\in\mathrm{Sup}(x)$, then $x$ already follows from $\KB$.)
\item For $Y\subseteq H$, the \emph{closure} $\mathrm{cl}(Y)\subseteq H$ is computed as follows: start with $Y$, and as long as some $x\in H$ outside the current set has a support contained in it, add $x$. (Chains are followed, and the result does not depend on the order of the additions.)
\item A \emph{reduction} of $H$ is a subset-minimal $R\subseteq H$ with $\mathrm{cl}(R)=H$.
\end{enumerate}
\end{definition}

A reduction is a smallest-possible set of atoms from which, according to the record, all of $H$ can be derived. Every hypothesis has at least one reduction, because $\mathrm{cl}(H)=H$. Table~\ref{tab:reductions} shows three typical situations.

\begin{table}[t]
\centering\small
\caption{Supports and reductions for hypotheses with three atoms. $\mathrm{Sup}(x)$ lists the supports of $x$; atoms not listed have none.}\label{tab:reductions}
\begin{tabularx}{\textwidth}{@{}l X X@{}}
\toprule
Situation & Supports & Reductions\\
\midrule
Two atoms derive each other & $\mathrm{Sup}(s_1)=\{\{s_2\}\}$, $\mathrm{Sup}(s_2)=\{\{s_1\}\}$ & $\{s_1,s_3\}$ and $\{s_2,s_3\}$\\
A chain & $\mathrm{Sup}(s_2)=\{\{s_1\}\}$, $\mathrm{Sup}(s_3)=\{\{s_2\}\}$ & $\{s_1\}$\\
Two atoms are needed together & $\mathrm{Sup}(s_3)=\{\{s_1,s_2\}\}$ & $\{s_1,s_2\}$\\
\bottomrule
\end{tabularx}
\end{table}

\begin{algorithm}[t]
\caption{$\textsc{MinimiseX}(\KB,O',\mathit{Red},H)$: minimality through the derivation record (B7)}\label{alg:minimisex}
\begin{algorithmic}[1]
\State $\KB_1\gets\op{Copy}(\KB)$; $\ \op{Add}(\KB_1,H)$ \Comment{with the derivation record switched on}
\If{$\neg\op{Consistent}(\KB_1)$ \textbf{or} some $o\in O'$ is not entailed by $\KB_1$} \Return $\emptyset$ \EndIf \Comment{(C) and (E), once}
\ForAll{$x\in H$} $\mathrm{Sup}(x)\gets\{\mathrm{Prem}(d)\cap H\mid d\in\mathrm{Der}(\KB_1,x),\ x\notin\mathrm{Prem}(d)\}$ \EndFor \Comment{reads, no reasoning}
\State $\mathit{Reds}\gets\emptyset$
\ForAll{$R\subseteq H$ by increasing size, \textbf{skipping} every $R$ that contains a member of $\mathit{Reds}$}
  \If{$\mathrm{cl}(R)=H$} $\mathit{Reds}\gets\mathit{Reds}\cup\{R\}$ \EndIf \Comment{pure set computation}
\EndFor
\State \Return $\{R\in\mathit{Reds}\mid R\text{ restates no }o\in\mathit{Red}\}$ \Comment{(R)}
\end{algorithmic}
\end{algorithm}

Algorithm~\ref{alg:minimisex} first checks (E) and (C) once on $\KB\cup H$, then reads the supports of all atoms from the record, and finally enumerates subsets of $H$ by increasing size, which is a computation on small sets without any reasoner call. Subsets that contain an earlier reduction are skipped, so only minimal ones are collected.

\begin{proposition}[Properties of \textsc{MinimiseX}]\label{prop:minx}
Assume the reasoner satisfies (B7), and let $H$ be a hypothesis. 
\begin{enumerate}[label=(\alph*),leftmargin=*,itemsep=1pt,topsep=2pt]
\item \emph{Equivalence.} If $\KB\cup H$ is consistent and $R$ is a reduction of $H$, then $\KB\cup R\models x$ for every $x\in H$. Hence $\KB\cup R$ and $\KB\cup H$ have the same consequences.
\item \emph{Correctness.} Every $R$ returned by \textsc{MinimiseX} satisfies (E), (C) and (R).
\item \emph{Cost.} \textsc{MinimiseX} makes $1+|O'|$ reasoner calls and reads the record $|H|$ times, whereas \textsc{Minimise} makes up to $2^{|H|}(|O'|+1)$ calls. Only the subset enumeration is exponential, and it is a set computation on $|H|$ elements.
\item \emph{Irredundancy.} If the record is complete (every justification of $x$ in $\AB\cup H$ is $\mathrm{Prem}(d)$ for some $d$), then no member $x$ of a reduction $R$ is entailed by $\KB\cup(R\setminus\{x\})$.
\end{enumerate}
\end{proposition}
\begin{proof}
(a) We show by induction on the order in which atoms are added to $\mathrm{cl}(R)$ that $\KB\cup R\models x$ for each of them. The atoms of $R$ itself are in $\KB\cup R$. Let $x$ be added because a support $S$ of $x$ lies in the current set. Then $\KB\cup R\models s$ for every $s\in S$ by the induction hypothesis. Let $d$ be a derivation of $x$ with $\mathrm{Prem}(d)\cap H=S$. All other minor premises of $d$ belong to $\AB$, so $\mathrm{Prem}(d)\subseteq\AB\cup S$, and $\langle\TB,\mathrm{Prem}(d)\rangle\models x$ by soundness of the record. By monotonicity, $\KB\cup R\models x$ (note that $\KB\cup R\subseteq\KB\cup H$ is consistent). Since $\mathrm{cl}(R)=H$, the claim follows.
(b) $R$ is returned only if $\KB_1=\KB\cup H$ is consistent and entails $O'$ (line~2). By (a), $\KB\cup R\models H$, so $\KB\cup R$ entails everything $\KB\cup H$ does: it is consistent, and it entails $O'$. Property (R) is checked in the last line.
(c) Lines 1--2 use one copy, one consistency check and $|O'|$ entailment checks; line~3 reads the record once for each atom.
(d) Suppose $\KB\cup(R\setminus\{x\})\models x$. Then there is a justification $J\subseteq\AB\cup(R\setminus\{x\})$ of $x$, so $x\notin J$. By completeness $J=\mathrm{Prem}(d)$ for some derivation $d$ of $x$ in $\KB\cup H$, hence $J\cap H\in\mathrm{Sup}(x)$ and $J\cap H\subseteq R\setminus\{x\}$. So $x\in\mathrm{cl}(R\setminus\{x\})$ and therefore $\mathrm{cl}(R\setminus\{x\})=H$, which contradicts the minimality of $R$.
\end{proof}

\begin{example}\label{ex:minx}
In scenario~b of the family example (Section~\ref{sec:dda-ex}), the hypothesis $H=\{\hC(\jn,\tz),\ \hP(\tz,\jn),\ \sn{Woman}(\jn)\}$ can arise when both observations of part~2 are explained. Since $\hC$ and $\hP$ are inverse roles, the derivation of $\hP(\tz,\jn)$ has the minor premise $\hC(\jn,\tz)$, and vice versa; $\sn{Woman}(\jn)$ has no derivation from the others. This is the first row of Table~\ref{tab:reductions}. The reductions are $\{\hC(\jn,\tz),\sn{Woman}(\jn)\}$ and $\{\hP(\tz,\jn),\sn{Woman}(\jn)\}$, the two equivalent explanations of Table~\ref{tab:fam2}, found from the stored derivations alone. In the same example, $\sn{Parent}(\jn)$ in $\{\sn{Parent}(\jn),\sn{Woman}(\jn),\hP(\tz,\jn)\}$ has the support $\{\hP(\tz,\jn)\}$ (the other minor premise, $\sn{Man}(\tz)$, is in $\AB$), which is why this hypothesis reduces to $\{\sn{Woman}(\jn),\hP(\tz,\jn)\}$.
\end{example}

\paragraph{What is guaranteed, and what is not.} Irredundancy (Proposition~\ref{prop:minx}(d)) and minimality (M) of Section~\ref{sec:dda} are different properties. (M) refers to the observations: no proper subset of the hypothesis still entails $O'$. Irredundancy refers to the hypothesis itself: no atom follows from the others. A reduction can contain an atom that is not needed for $O'$ and is nevertheless not derivable from the others, so \textsc{MinimiseX} alone does not give (M). It gives (E), (C), (R) and irredundancy for the price of one check. To obtain (M) we run the exact pass $\textsc{Minimise}(\KB,O',\mathit{Red},R)$ on each reduction $R$; this costs at most $2^{|R|}(|O'|+1)$ calls, and $R$ is usually much smaller than $H$. We call the first mode \emph{fast} and the combination \emph{exact}. Both are sound. The theorems of Section~\ref{sec:dda-props} that mention (M) are stated for the exact enumeration of Algorithm~\ref{alg:minimise}, so they hold for the exact mode.

Two design decisions should be explicit. First, when $x$ is derivable from $y$ but not conversely, the reduction keeps the \emph{basic} atom $y$ and drops the weaker $x$; the weaker alternative is offered only if it arises as a hypothesis of its own. Second, if the reasoner offers no record, \textsc{Minimise} is used; the answers are the same in the exact mode, and only the cost differs.

\subsection{Properties}
\label{sec:dda-props}

Let $\omega=|\Omega|$, $\iota=|\Ind|$, $n=|O'|$, and let $\delta$ be the maximal number of parts of a conjunction in $\Omega$.

\begin{theorem}[Soundness]\label{thm:dda-sound}
Every $H\in\mathit{Out}_L$ satisfies (E), (C), (M) and (R).
\end{theorem}
\begin{proof}
$H$ is returned by \textsc{Minimise} for a superset $H_1$ with $\KB\cup H_1$ consistent; hence $\KB\cup H$ is consistent (C). Membership in $\mathit{Mins}$ requires (E). For (M) suppose $H''\subsetneq H$ satisfied (E). Since subsets are examined by increasing size, either $H''$ or a subset of it was added to $\mathit{Mins}$ earlier, and $H$ would have been skipped; contradiction. (R) is the final filter. All answers of the reasoner are exact by (B3).
\end{proof}

\begin{theorem}[Termination and decidability]\label{thm:dda-term}
\DDA{} terminates for every $L\in\mathbb N\cup\{\infty\}$. Every hypothesis reachable in at most $L$ rounds has at most $h=n\delta^{L+1}$ atoms, so $|\mathrm{Reach}_L|\le(\omega+1)^h$; \textsc{Suggest} is called for at most $\omega$ distinct assertions; \textsc{Successors} on a hypothesis with $k$ atoms creates at most $(\omega+1)^{k}(\iota+1)^{\delta k}$ hypotheses; and \textsc{Minimise} makes at most $2^{h}(n+2)$ reasoner calls. Hence, for fixed $n$, $\delta$ and $L$, the number of reasoner calls is polynomial in $\omega$ and $\iota$. Whether a given $H$ belongs to $\mathit{Out}_L$ is decidable, also for $L=\infty$.
\end{theorem}
\begin{proof}
$H_0$ has at most $n\delta$ atoms; each round replaces every atom by at most $\delta$ atoms and grounding does not add atoms, so after $L$ rounds a hypothesis has at most $n\delta\cdot\delta^{L}=h$ atoms, and it is a subset of $\Omega$, whence the bound on $|\mathrm{Reach}_L|$. Stage~A has at most $\omega+1$ choices per atom, stage~B at most $\iota+1$ per existential atom, and there are at most $\delta k$ atoms after stage~A. For $L=\infty$ the loop adds only unseen subsets of $\Omega$, so it stops after at most $2^{\omega}$ rounds. All reasoner calls terminate by Theorem~\ref{thm:sroiq}.
\end{proof}

\begin{theorem}[What \DDA{} finds]\label{thm:dda-complete}
Assume $\KB\nentails o$ for the observations $o\in O'$.
\begin{enumerate}[label=(\alph*)]
\item $\mathit{Out}_L$ is exactly the set of subset-minimal relevant hypotheses $H'$ with $\KB\cup H'\models O'$ that are contained in some consistent member of $\mathrm{Reach}_L$.
\item \emph{(Single assertions.)} If $O'=\{o\}$ and $\{a\}$ is an explanation with $a\in\Omega$, then $\{a\}\in\mathit{Out}_1$ if $a$ is not a conjunction and relevant (if $a$ is a conjunction, some subset of $\Dec(a)$ is in $\mathit{Out}_1$).
\item \emph{(Chains.)} If $O'=\{a_0\}$ and $a_{i+1}\in\Sugg(a_i)$ for $i<\ell\le L$, with $a_1,\dots,a_\ell$ not conjunctions, then $\KB\cup\{a_\ell\}\models a_0$, $\{a_\ell\}\in\mathrm{Reach}_\ell$, and $\{a_\ell\}\in\mathit{Out}_L$ if it is relevant.
\item \emph{(Several observations.)} Proposition~\ref{prop:combine}.
\item $\mathit{Out}_L\subseteq\mathit{Out}_{L+1}$, and $\mathit{Out}_L=\mathit{Out}_\infty$ once $L\ge2^\omega$.
\end{enumerate}
\end{theorem}
\begin{proof}
(a) Elements of $\mathit{Out}_L$ are minimal subsets of a reachable $H$; minimality (M) is a property of the set itself, not of the hypothesis it came from. Conversely, a minimal such set is minimal inside its source hypothesis; relevance is inherited by subsets. (b) By Lemma~\ref{lem:one-step}, $a\in\Sugg(o)$, so $\Dec(a)\in\mathrm{Reach}_1$; apply (a). (c) The entailments follow from Lemma~\ref{lem:sugg}(i) by transitivity. Replacing $a_i$ by $a_{i+1}$ round by round shows $\{a_\ell\}\in\mathrm{Reach}_\ell$; consistency is Lemma~\ref{lem:sugg}(ii). (e) $\mathrm{Reach}_L\subseteq\mathrm{Reach}_{L+1}$, and members of $\mathit{Out}_L$ are minimal by Theorem~\ref{thm:dda-sound}; the second claim follows from the bound on the number of hypotheses.
\end{proof}

\begin{proposition}[What \DDA{} misses]\label{prop:limits}
(a) There is an explanation over $\Omega$ that is not in $\mathit{Out}_\infty$. (b) There is a problem in which \DDA{} returns an explanation that no abducer restricted to atomic assertions over the named individuals can return. (c) \DDA{} may return several logically equivalent explanations.
\end{proposition}
\begin{proof}
(a) Let $\TB=\{\exists R.B\sqsubseteq C\}$, $\AB=\emptyset$, individuals $a,b$ and $O=\{C(a)\}$. The set $\{R(a,b),B(b)\}$ is an explanation. But $\KB\cup\{\neg C(a)\}$ derives only $(\forall R.\neg B)(a)$, so $\Sugg(C(a))=\{C(a),(\exists R.B)(a)\}$; the existential atom has the suggestion $(\exists R.B)(a)$ only, and it cannot be grounded since no individual is known to be a $B$. Hence $\mathrm{Reach}_\infty=\{\{C(a)\},\{(\exists R.B)(a)\}\}$ and $\mathit{Out}_\infty=\{\{(\exists R.B)(a)\}\}$. (b) The same problem with only the individual $a$. (c) See Example~\ref{ex:family}: $\{\hC(\jn,\tz),\sn{Woman}(\jn)\}$ and $\{\hP(\tz,\jn),\sn{Woman}(\jn)\}$ are equivalent.
\end{proof}

\subsection{Example: family relations}
\label{sec:dda-ex}

We use the family ontology of our benchmarks, adapted from \cite{pukancova2020}. The terminology $\TB$ contains
\begin{center}
\begin{tabular}{@{}ll@{}}
$\sn{Mother}\equiv\sn{Parent}\sqcap\sn{Woman}$ & $\sn{Parent}\equiv\exists\hC.(\MW)$\\
$\sn{Father}\equiv\sn{Man}\sqcap\sn{Parent}$ & $\sn{Child}\equiv\exists\hP.(\MW)$\\
$\sn{Grandmother}\sqsubseteq\sn{Mother}$ & $\hC\equiv\hP^{-}$\\
$\sn{Grandfather}\sqsubseteq\sn{Father}$ & $\sn{Man}\sqcap\sn{Woman}\sqsubseteq\bot$\\
\end{tabular}
\end{center}
and the ABox is $\AB=\{\sn{Man}(\tz)\}$. The individual $\jn$ occurs only in the observations. There are two scenarios. In \textbf{scenario~a} the terminology is exactly as listed; in \textbf{scenario~b} it additionally declares $\hC$ irreflexive (nobody is their own child). The individuals are $\tz$ and $\jn$, so $\Omega$ has $68$ assertions. All results below were checked by running the algorithms of this section against the HermiT reasoner \cite{glimm2014hermit}.

\begin{example}[Family relations]\label{ex:family}
\textbf{Part 1: one observation, $O=\{\sn{Mother}(\jn)\}$.} The initial hypothesis is $H_0=\{\sn{Mother}(\jn)\}$, which is reducible, since $\Sugg(\sn{Mother}(\jn))$ has three elements (Example~\ref{ex:sugg-mother}); the reasoner gives the same suggestions in both scenarios. Table~\ref{tab:sugg-family} lists the suggestion sets that are needed in later rounds.

\begin{table}[t]
\centering\small
\caption{Suggestion sets $\Sugg(s)$ in the family example, computed by the reasoner. The sets are identical in both scenarios unless noted; the atoms are written in the notation of the paper and conjunctions are split by \textsc{Dec} when used.}
\label{tab:sugg-family}
\begin{tabularx}{\textwidth}{@{}P{3.6cm}X@{}}
\toprule
\textbf{Assertion $s$} & \textbf{Suggestions $\Sugg(s)$}\\ \midrule
$\sn{Mother}(\jn)$ & $\sn{Mother}(\jn)$, \ $(\sn{Parent}\sqcap\sn{Woman})(\jn)$, \ $\sn{Grandmother}(\jn)$\\
$\sn{Grandmother}(\jn)$ & $\sn{Grandmother}(\jn)$ \quad (irreducible)\\
$\sn{Woman}(\jn)$ & $\sn{Woman}(\jn)$, \ $\sn{Mother}(\jn)$, \ $(\sn{Parent}\sqcap\sn{Woman})(\jn)$, \ $\sn{Grandmother}(\jn)$\\
$\sn{Parent}(\jn)$ & $\sn{Parent}(\jn)$, \ $(\exists\hC.(\MW))(\jn)$, \ $\hC(\jn,\tz)$, \ $\hP(\tz,\jn)$, \ $\sn{Mother}(\jn)$, \ $\sn{Father}(\jn)$, \ $\sn{Grandmother}(\jn)$, \ $\sn{Grandfather}(\jn)$, \ $(\sn{Parent}\sqcap\sn{Woman})(\jn)$, \ $(\sn{Man}\sqcap\sn{Parent})(\jn)$\\
$\sn{Child}(\tz)$ & scenario b: $\sn{Child}(\tz)$, \ $(\exists\hP.(\MW))(\tz)$;\quad scenario a: additionally $\hC(\tz,\tz)$, \ $\hP(\tz,\tz)$\\
\bottomrule
\end{tabularx}
\end{table}

Note that $\Sugg(\sn{Parent}(\jn))$ already contains the concrete atoms $\hC(\jn,\tz)$ and $\hP(\tz,\jn)$. The reasoner derives $\neg\hC(\jn,\tz)$ from $\neg\sn{Parent}(\jn)$, because $\tz$ is known to be a man. The atom $\hC(\jn,\jn)$ is \emph{not} a suggestion, because nothing is known about $\jn$'s gender in the empty context; it comes in through grounding. Table~\ref{tab:fam1} traces the rounds.

\begin{table}[t]
\centering\small
\caption{Example~\ref{ex:family}, part 1: $O=\{\sn{Mother}(\jn)\}$. Each row is a minimal relevant explanation, with the round in which it first appears. Scenario~b is scenario~a without the rows marked \emph{only a}.}
\label{tab:fam1}
\begin{tabularx}{\textwidth}{@{}c P{6.6cm} X@{}}
\toprule
\textbf{Round} & \textbf{Explanation $H$} & \textbf{How it arises}\\ \midrule
0 & $\{\sn{Mother}(\jn)\}$ & the initial hypothesis; it restates the observation and is not relevant, so it is not output\\ \midrule
1 & $\{\sn{Parent}(\jn),\ \sn{Woman}(\jn)\}$ & suggestion $(\sn{Parent}\sqcap\sn{Woman})(\jn)$, split into two parts\\
1 & $\{\sn{Grandmother}(\jn)\}$ & suggestion $\sn{Grandmother}(\jn)$; it is irreducible, so the chain stops here\\ \midrule
2 & $\{\hC(\jn,\tz),\ \sn{Woman}(\jn)\}$ & $\sn{Parent}(\jn)$ replaced by its suggestion $\hC(\jn,\tz)$\\
2 & $\{\hP(\tz,\jn),\ \sn{Woman}(\jn)\}$ & the suggestion $\hP(\tz,\jn)$; equivalent to the previous row by the inverse axiom (Prop.~\ref{prop:limits}(c))\\
2 & $\{(\exists\hC.(\MW))(\jn),\ \sn{Woman}(\jn)\}$ & the existential suggestion, kept as a weaker alternative\\
2 & $\{\hC(\jn,\jn),\ \sn{Woman}(\jn)\}$ \quad \emph{only a} & grounding of the existential atom by the witness $\jn$ (context $\sn{Woman}(\jn)$); in scenario~b the reasoner finds $\KB\cup H$ inconsistent and (C) removes it\\ \midrule
3 & $\{\hP(\jn,\jn),\ \sn{Woman}(\jn)\}$ \quad \emph{only a} & the inverse form of the cyclic explanation\\
\bottomrule
\end{tabularx}
\end{table}

In round~2 the hypothesis $\{\sn{Parent}(\jn),\sn{Woman}(\jn)\}$ is expanded: stage~A replaces $\sn{Parent}(\jn)$ by each of its ten suggestions (and keeps $\sn{Woman}(\jn)$ or replaces it by one of its four). Many combinations are discarded: for example $\{\sn{Father}(\jn),\sn{Woman}(\jn)\}$ is inconsistent, because $\sn{Father}\sqsubseteq\sn{Man}$ and $\sn{Man}$ is disjoint from $\sn{Woman}$; and $\{\sn{Mother}(\jn),\dots\}$ restates the observation. Scenario~a admits the cyclic explanation $\hC(\jn,\jn)$: without an irreflexivity axiom, ``$\jn$ is her own child'' is a logically legitimate explanation. Scenario~b removes it \emph{semantically}: the reasoner finds the explanation inconsistent, so no syntactic filter on reflexive atoms is needed.

\textbf{Part 2: two observations, $O=\{\sn{Mother}(\jn),\ \sn{Child}(\tz)\}$ (scenario~b).} Both observations are reducible and $H_0=\{\sn{Mother}(\jn),\sn{Child}(\tz)\}$. Table~\ref{tab:fam2} shows the result.

\begin{table}[t]
\centering\small
\caption{Example~\ref{ex:family}, part 2: $O=\{\sn{Mother}(\jn),\sn{Child}(\tz)\}$ in scenario~b. Minimal relevant explanations of the two observations \emph{together}.}
\label{tab:fam2}
\begin{tabularx}{\textwidth}{@{}c P{7.3cm} X@{}}
\toprule
\textbf{Round} & \textbf{Explanation $H$} & \textbf{How it arises}\\ \midrule
1 & $\{\sn{Woman}(\jn),\ \hP(\tz,\jn)\}$ & from $\{\sn{Parent}(\jn),\sn{Woman}(\jn),(\exists\hP.(\MW))(\tz)\}$: the existential atom is grounded by $\jn$ in the context $\sn{Woman}(\jn)$; then $\sn{Parent}(\jn)$ follows from $\hP(\tz,\jn)$ and is redundant\\
1 & $\{\sn{Grandmother}(\jn),\ \hP(\tz,\jn)\}$ & the same with $\sn{Grandmother}(\jn)$ providing the context (it entails $\sn{Woman}(\jn)$)\\
1 & $\{\sn{Grandmother}(\jn),\ (\exists\hP.(\MW))(\tz)\}$ & the weaker variant that keeps the existential atom\\
1 & $\{\sn{Parent}(\jn),\ \sn{Woman}(\jn),\ (\exists\hP.(\MW))(\tz)\}$ & the weaker variant; all three atoms are needed because nothing links $\jn$ to $\tz$\\ \midrule
2 & $\{\sn{Woman}(\jn),\ \hC(\jn,\tz)\}$ & one role atom serves both observations (synergy)\\
2 & $\{\sn{Grandmother}(\jn),\ \hC(\jn,\tz)\}$ & the same with the subsumption alternative\\
2 & $\{\sn{Woman}(\jn),\ (\exists\hC.(\MW))(\jn),\ (\exists\hP.(\MW))(\tz)\}$ & the existential variants of both parts\\
\bottomrule
\end{tabularx}
\end{table}

The first round combines the stage-A choices for the two observations: three choices for $\sn{Mother}(\jn)$ (itself, the split conjunction, $\sn{Grandmother}(\jn)$) times two for $\sn{Child}(\tz)$ (itself, the existential atom) give six combined hypotheses; grounding then adds the witness $\jn$ for $\tz$'s parent whenever the other atoms entail that $\jn$ is a woman. The witness $\tz$ is rejected: $\hP(\tz,\tz)$ would make $\tz$ his own parent, contradicting irreflexivity of $\hC$. In round~2, $\sn{Parent}(\jn)$ and the role atom are replaced by their suggestions; since $\hC$ and $\hP$ are inverse roles, one atom $\hC(\jn,\tz)$ suffices for both observations. In round~3 nothing new is produced; this is the end of the search. A hypothesis such as $\{\sn{Parent}(\jn),\sn{Woman}(\jn),\sn{Child}(\tz)\}$ is rejected by (R) because it restates the reducible observation $\sn{Child}(\tz)$, and $\{\sn{Parent}(\jn),\sn{Woman}(\jn),\hP(\tz,\jn)\}$ is not minimal, since $\sn{Parent}(\jn)$ follows from the role atom.
\end{example}

\section{Alternating Abduction and Planning}
\label{sec:apt}

The planner \APT{} (\emph{A}bduction and \emph{P}lanning over \emph{T}DKGs) finds plans for a goal in a TDKG by \emph{alternating between abduction and planning} (Figure~\ref{fig:loop}). The planning step, \textsc{Generate}, searches backwards from a goal through the rules and returns candidate action sequences; the preconditions that no action can produce and that do not hold are not given up but \emph{assumed}. The abduction step, \textsc{Enrich}, asks \DDA{} for explanations of these assumptions. An explanation is a set of assertions, that is, a statement about the \emph{axioms} of the knowledge base; to make it true we need \emph{actions}, so every explanation becomes a new goal for \textsc{Generate}, and the resulting sub-plan is scheduled before the plan that needed it. The sub-plan may in turn rest on assumptions, which are explained again, and so on: planning calls abduction and abduction calls planning, up to a depth bound $R$ that makes the recursion terminate. The third component, \textsc{Classify}, is the \emph{temporal projection} of Section~\ref{sec:projection}: the assembled candidate is scheduled into the TDKG and the timeline is played through the reasoner, which decides exactly whether the candidate is valid. Because validity is decided only by the projection, the first two components may be as heuristic as we like: they affect which plans are found and how fast, never whether a returned plan is correct. If the projection shows that a candidate fails because of an inconsistency that an earlier deletion would avoid, a \textsc{Repair} step calls \textsc{Generate} again with a \emph{deletion goal}.

\paragraph{Hypothetical and supported plans.} Recall (Definition~\ref{def:valid}) that a candidate $(a_1,\dots,a_n;\Gamma)$ is a \emph{hypothetical plan} if it is valid in the hypothetical world $\KB_0\cup\Gamma$, and a \emph{plan} if $\Gamma=\emptyset$. \textsc{Generate} produces hypothetical plans; the recursion of \textsc{Enrich} tries to \emph{support} each assumption by a sub-plan, so that the final candidate needs fewer assumptions or none. We call a plan obtained this way a \emph{supported plan}. Its structure is a tree: the root is the plan of level~0; each assumption set $\Gamma_j$ at level $j$ has as children one \emph{explanation} $H$ (an alternative) and the sub-plan of level $j{+}1$ for $H$. The final action sequence is the concatenation of the plans along the tree, the deepest sub-plan first. (A supported plan is an ordinary plan in the sense of Definition~\ref{def:valid}; ``supported'' describes how it was found, and the tree records which sub-plan supports which assumption.)

\subsection{Generate: optimistic backward search}
\label{sec:generate}

\textsc{Generate} searches backwards from the goal, as in classical regression planning, in a \emph{relaxed} world: deletions are ignored, so the relaxed state only grows. This makes the search cheap, and it is safe because the projection decides validity. A \emph{node} records the open subgoals $S$ (atoms, possibly with variables), the action instances chosen so far (in execution order), the assumptions made so far, the insert atoms of the chosen actions (kept per action), the cost $g$ and the depth $k$. The root has $S=G$. Expanding a node works as follows; Algorithm~\ref{alg:generate} gives the pseudocode.

\begin{enumerate}
\item \emph{Absorb.} Find a largest subset $U\subseteq S$ in which every atom already holds under some answer $\theta$ (this uses entailment, so facts derived by the ontology count). An atom is checked against the initial KB together with the insert atoms of those actions that run \emph{before} the action that needs it, that is, of the actions chosen after the atom arose; the effects of an action never support its own preconditions. The goals still to achieve are $S'=\theta(S\setminus U)$. If $S'=\emptyset$ the node is \emph{solved}.
\item \emph{Regress.} For each rule $r$ and each way of unifying some atoms $V\subseteq S'$ with atoms of the insert template $I_r$, create a successor node: the corresponding action instance is put \emph{before} the plan found so far (it is executed earlier); the unified goals are replaced by the preconditions $\mathit{pre}(r)$ (the event atoms are supplied by the schedule); variables that remain unbound stay open and can be bound by a later absorption; when a node is solved, the variables still open become fresh \emph{placeholder names}, which stand for entities that the plan introduces; the insert atoms of $r$ are added to the relaxed state; the cost grows by one.
\item \emph{Assume.} If no regression is possible (no rule applies, or the depth bound $d$ has been reached) and the plan is not empty, then the goals $S'$ are \emph{assumed}: they are added to the assumptions and the node is solved; candidates with fewer assumptions are preferred when the results are ordered (Algorithm~\ref{alg:apt}). Assumption is a last resort, and it never happens at the root.
\end{enumerate}

\begin{algorithm}[t]
\caption{$\textsc{Generate}(\KB_0,S_0,d)$: relaxed backward search (goal $S_0$)}\label{alg:generate}
\begin{algorithmic}[1]
\Require consistent $\KB_0$, goal $S_0$ (a set of atoms), depth bound $d$
\Ensure a finite set of candidates $(\pi;\Gamma)$
\State $\mathit{Open}\gets\{(S_0,\ [\,],\ \emptyset,\ [\,],\ 0,\ 0)\}$ \Comment{node $(S,\pi,\Gamma,\Lambda,g,k)$: $\Lambda$ lists the insert atoms of the chosen actions; $\mathit{Open}$ is ordered by $g+|S|$}
\State $\mathit{Cands}\gets\emptyset$
\While{$\mathit{Open}\ne\emptyset$}
  \State remove from $\mathit{Open}$ a node $(S,\pi,\Gamma,\Lambda,g,k)$ with minimal $g+|S|$
  \ForAll{pairs $(\theta,U)$ with $U\subseteq S$ of maximal size and $\theta$ such that every $\alpha\in U$ is entailed by $\KB_0\cup\Lambda_\alpha$ under $\theta$} \Comment{Absorb; $\Lambda_\alpha$: inserts of the actions chosen after $\alpha$ arose}
    \State $S'\gets\theta(S\setminus U)$; $\ \pi\gets\theta(\pi)$; $\ \Gamma\gets\theta(\Gamma)$; $\ \Lambda\gets\theta(\Lambda)$
    \If{$S'=\emptyset$} $\mathit{Cands}\gets\mathit{Cands}\cup\{(\pi;\Gamma)\}$ with open variables replaced by fresh placeholder names; \textbf{continue} \EndIf \Comment{solved}
    \State $\mathit{Succ}\gets\emptyset$
    \If{$k<d$}
      \ForAll{rules $r\in\RS$, nonempty $V\subseteq S'$, most general unifiers $\sigma$ of $V$ with atoms of $I_r$}
        \State keep the remaining variables of $\sigma$ open (renamed apart)
        \State $\mathit{Succ}\gets\mathit{Succ}\cup\{(\sigma(S'\setminus V)\cup\sigma(\mathit{pre}(r)),\ [\sigma(a_r)]\mathbin{+\!\!+}\pi,\ \Gamma,\ \Lambda\mathbin{+\!\!+}[\sigma(I_r)],\ g+1,\ k+1)\}$
      \EndFor
    \EndIf
    \If{$\mathit{Succ}=\emptyset$ \textbf{and} $\pi\ne[\,]$} $\mathit{Cands}\gets\mathit{Cands}\cup\{(\pi;\Gamma\cup S')\}$ \Comment{assume (open variables become placeholders)}
    \Else{} $\mathit{Open}\gets\mathit{Open}\cup\mathit{Succ}$ \EndIf
  \EndFor
\EndWhile
\State \Return $\mathit{Cands}$
\end{algorithmic}
\end{algorithm}

Here $\sigma(a_r)$ is the action instance $\langle r,\beta\rangle$ whose participants $\beta$ are given by $\sigma$. The relaxed state against which an atom is absorbed is built from the initial KB and the \emph{insert} atoms only: it is what the plan would reach if nothing were ever deleted. The nodes are expanded in order of $g+|S|$ (A$^*$ style); the result is a finite set of candidates $(a_1,\dots,a_m;\Gamma_0)$ with one assumption set $\Gamma_0$ each.

\begin{lemma}\label{lem:gen-space}
\textsc{Generate} terminates and every candidate it returns lies in the candidate space $\mathrm{Cand}_d$ of Section~\ref{sec:plans}. Every candidate returned by $\textsc{Enrich}_R$ lies in $\mathrm{Cand}_{(R+1)d}$, where each level contributes its own placeholder names.
\end{lemma}
\begin{proof}
Each expansion has finitely many successors (finitely many rules, unifiers, answers), and each successor has depth one more; the depth is bounded by $d$, so the search tree is finite. A node introduces at most $v$ placeholders, so at most $d\cdot v+|\mathrm{vars}(S_0)|$ overall; assumptions are instances of atoms of the goal or of rule bodies, hence members of $\Omega_d$; only regressions add actions, so there are at most $d$. \textsc{Enrich} concatenates at most $R+1$ such plans (Proposition~\ref{prop:rec}) and the goals of deeper levels are explanations, i.e.\ sets of assertions over $\Omega$, so the second claim follows with $d$ replaced by $(R+1)d$.
\end{proof}

\subsection{Enrich: the recursion between abduction and planning}
\label{sec:enrich}

Let $(\pi;\Gamma_0)$ be a candidate with assumptions $\Gamma_0$. \textsc{Enrich} proceeds as follows (Algorithm~\ref{alg:enrich}).
\begin{enumerate}
\item \emph{Abduce.} Ask \DDA{} for explanations of $\Gamma_0$ in the initial KB $\KB_0$, treating $\Gamma_0$ as a set of observations. Each explanation $H$ is an \emph{alternative} way to account for $\Gamma_0$ and leads to its own candidates. The trivial explanation $H=\Gamma_0$ (which \DDA{} returns when all assumptions are irreducible) is ignored, since it would only repeat the candidate.
\item \emph{Plan for the explanation.} Run \textsc{Generate} with $H$ as the goal. Each sub-plan $(\sigma;\Gamma_1)$ it returns makes $H$, hence $\Gamma_0$, true; it is scheduled \emph{before} $\pi$, because $\pi$ needs $\Gamma_0$ when it runs. If \textsc{Generate} finds no sub-plan, $H$ replaces $\Gamma_0$ as the residual assumption set: it is more basic than the original one, and it is what the user has to verify.
\item \emph{Recurse.} The sub-plan $\sigma$ may itself assume $\Gamma_1$. Then steps 1--3 are repeated for $(\sigma;\Gamma_1)$, one level deeper, and so on. After $R$ levels, or when the assumption set is empty, or when \DDA{} has no non-trivial explanation, the recursion stops.
\end{enumerate}
Thus a branch of the recursion looks like this: plan $\pi_0$ assumes $\Gamma_0$; \DDA{} explains $\Gamma_0$ by $H_1$; the sub-plan $\sigma_1$ for $H_1$ assumes $\Gamma_1$; \DDA{} explains $\Gamma_1$ by $H_2$; the sub-plan $\sigma_2$ for $H_2$ assumes $\Gamma_2$; \dots; the result is the action sequence $\sigma_R\cdots\sigma_2\,\sigma_1\,\pi_0$ with the residual assumptions of the last level. Different explanations give different branches, so the recursion explores a tree.

\begin{algorithm}[t]
\caption{$\textsc{Enrich}_R(\pi,\Gamma_0)$: abduce, plan for the explanations, recurse}\label{alg:enrich}
\begin{algorithmic}[1]
\If{$\Gamma_0=\emptyset$ \textbf{or} $R=0$} \Return $\{(\pi;\Gamma_0)\}$ \EndIf \Comment{nothing to support, or depth bound reached}
\State $\mathcal H\gets\textsc{DDA}(\KB_0,\Gamma_0,L)\setminus\{\Gamma_0\}$ \Comment{abduction: alternative explanations of the assumptions}
\If{$\mathcal H=\emptyset$} \Return $\{(\pi;\Gamma_0)\}$ \EndIf
\State $\mathit{Res}\gets\emptyset$
\ForAll{$H\in\mathcal H$} \Comment{\textbf{in parallel}, Section~\ref{sec:parcache}}
  \State $\mathit{Sub}\gets\textsc{Generate}(\KB_0,H,d)$ \Comment{planning: plans that achieve the explanation}
  \If{$\mathit{Sub}=\emptyset$} $\mathit{Res}\gets\mathit{Res}\cup\{(\pi;H)\}$ \Comment{keep $H$ as a residual assumption}
  \Else{} \ForAll{$(\sigma;\Gamma_1)\in\mathit{Sub}$ \textbf{and} $(\sigma';\Gamma_2)\in\textsc{Enrich}_{R-1}(\sigma,\Gamma_1)$} \Comment{recursion, one level deeper}
    \State $\mathit{Res}\gets\mathit{Res}\cup\{(\sigma'\mathbin{+\!\!+}\pi;\ \Gamma_2)\}$ \Comment{sub-plan first}
  \EndFor\EndIf
\EndFor
\State \Return $\mathit{Res}$
\end{algorithmic}
\end{algorithm}

\begin{proposition}[Shape of the recursion]\label{prop:rec}
Let \textsc{Generate} return at most $g$ candidates per call and \DDA{} at most $b$ explanations per call. Then $\textsc{Enrich}_R$ makes at most $\sum_{j<R}g(bg)^{j}$ calls of \DDA{} and at most $b$ times as many further calls of \textsc{Generate}, returns at most $(R+1)\,g\,(bg)^{R}$ candidates, and each returned candidate has at most $(R+1)d$ actions. Sub-plans precede the plans they support, and the assumptions of a candidate are the residual assumptions of its deepest level.
\end{proposition}
\begin{proof}
At level $j$ there are at most $g(bg)^{j}$ candidates: level~0 has at most $g$, and each candidate of level $j$ has at most $b$ explanations with at most $g$ sub-candidates each. \DDA{} is called once per candidate of level $j<R$. Every call of \textsc{Generate} returns plans of at most $d$ actions, and a returned candidate concatenates one plan per level. The order is that of the concatenation $\sigma'\mathbin{+\!\!+}\pi$ in Algorithm~\ref{alg:enrich}.
\end{proof}

\begin{example}[Two abductions and two sub-plans: the lab]\label{ex:lab}
The terminology of the benchmark \sn{LabAccess} contains $\sn{Registered}\sqsubseteq\sn{Human}$, $\sn{Certified}\sqsubseteq\sn{Qualified}$, $\sn{Degreed}\sqsubseteq\sn{Qualified}$ and $\sn{Cleared}\equiv\sn{Trained}\sqcap\sn{HasPass}$. The rules (one per action, with the participant $p$) are
\begin{center}\small
\begin{tabular}{@{}lll@{}}
\Certify: $\sn{Registered}(p)\Rightarrow\sn{Certified}(p)$ & \Train: $\sn{Qualified}(p)\Rightarrow\sn{Trained}(p)$\\
\IssuePass: $\sn{Registered}(p)\Rightarrow\sn{HasPass}(p)$ & \EnterLab: $\sn{Cleared}(p)\Rightarrow\sn{InLab}(p)$
\end{tabular}
\end{center}
The initial state is $\{\sn{Registered}(\sn{amir})\}$ and the goal is $\{\sn{InLab}(\sn{amir})\}$. No rule inserts \sn{Cleared} or \sn{Qualified}: they are derived classes. Table~\ref{tab:lab} follows the recursion.
\begin{table}[t]
\centering\small
\caption{Example~\ref{ex:lab}: the recursion of \APT{} for the lab. \textbf{Plan} rows are calls of \textsc{Generate}, \textbf{abduce} rows are calls of \DDA{} on the assumptions of the row above. The results were computed with the implementation (Section~\ref{sec:impl}) over HermiT.}
\label{tab:lab}
\begin{tabularx}{\textwidth}{@{}c l P{3.6cm} X@{}}
\toprule
\textbf{Level} & \textbf{Step} & \textbf{Input} & \textbf{Output}\\ \midrule
0 & plan & goal $\sn{InLab}(\sn{amir})$ & regress through \EnterLab; $\sn{Cleared}(\sn{amir})$ does not hold and no rule produces it: assumed. Candidate $([\EnterLab(\sn{amir})];\Gamma_0)$ with $\Gamma_0=\{\sn{Cleared}(\sn{amir})\}$\\
0 & abduce & observation $\sn{Cleared}(\sn{amir})$ & $\Sugg=\{\sn{Cleared}(\sn{amir}),\ (\sn{Trained}\sqcap\sn{HasPass})(\sn{amir})\}$; the only relevant explanation is $H_1=\{\sn{Trained}(\sn{amir}),\sn{HasPass}(\sn{amir})\}$\\ \midrule
1 & plan & goal $H_1$ & $\sn{Trained}$ is regressed through \Train, whose precondition $\sn{Qualified}(\sn{amir})$ does not hold: assumed; $\sn{HasPass}$ through \IssuePass, whose precondition holds. Sub-plans $[\Train,\IssuePass]$ and $[\IssuePass,\Train]$ (the two commute), with $\Gamma_1=\{\sn{Qualified}(\sn{amir})\}$\\
1 & abduce & observation $\sn{Qualified}(\sn{amir})$ & $\Sugg=\{\sn{Qualified},\ \sn{Certified},\ \sn{Degreed}\}(\sn{amir})$; two explanations, $H_2=\{\sn{Certified}(\sn{amir})\}$ and $H_2'=\{\sn{Degreed}(\sn{amir})\}$\\ \midrule
2 & plan & goal $H_2$ & regress through \Certify; its precondition holds. Sub-plan $[\Certify]$ with $\Gamma_2=\emptyset$\\
2 & plan & goal $H_2'$ & no rule inserts $\sn{Degreed}$: no sub-plan; $H_2'$ stays as the residual assumption\\ \midrule
& assemble & $\sigma_2\sigma_1\pi_0$ & $[\Certify,\Train,\IssuePass,\EnterLab]$ with $\Gamma=\emptyset$ (and the other order of the middle actions); and $[\Train,\IssuePass,\EnterLab]$ with $\Gamma=\{\sn{Degreed}(\sn{amir})\}$\\
& \textsc{Classify} & projection & the first candidates are valid \emph{plans}: $\sn{Certified}(\sn{amir})$ entails $\sn{Qualified}(\sn{amir})$ so \Train{} is effective; after \Train{} and \IssuePass{} the state entails $\sn{Cleared}(\sn{amir})$, so \EnterLab{} fires and inserts $\sn{InLab}(\sn{amir})$. The last one is a valid \emph{hypothetical} plan, reported with its assumption. The order $[\Train,\Certify,\dots]$ is rejected: \Train{} is not effective at the first instant\\
\bottomrule
\end{tabularx}
\end{table}
With $R=1$ the recursion stops after the first sub-plan and the results are the two hypothetical plans $[\Train,\IssuePass,\EnterLab]$ and $[\IssuePass,\Train,\EnterLab]$ (the two orders of commuting actions) with the unresolved assumption $\sn{Qualified}(\sn{amir})$; with $R=2$ the plans with $\Gamma=\emptyset$ appear, together with the two plans that assume $\sn{Degreed}(\sn{amir})$, and $R=3$ adds nothing. The second abduction therefore matters, and so does the depth bound: it decides how many levels of ``derived preconditions'' the planner can resolve. The alternative $\sn{Degreed}(\sn{amir})$ shows that \DDA{} proposes all explanations of an assumption and that \textsc{Generate} can fail on some of them without harming the others.
\end{example}

\begin{proposition}[The depth bound is needed]\label{prop:depth}
For every $n\ge1$ there is a TDKG whose only valid plan needs $n$ abductions and $n$ sub-plans, and $\textsc{Enrich}_R$ returns it iff $R\ge n$.
\end{proposition}
\begin{proof}
Let $\TB=\{D_i\equiv E_{i-1}\sqcap F\mid 1\le i\le n\}$, $\AB_0=\{\sn{Registered}(a),F(a)\}$, with rules $\sn{Act}_0:\sn{Registered}(p)\Rightarrow E_0(p)$ and $\sn{Act}_i:D_i(p)\Rightarrow E_i(p)$ for $i\ge1$, and goal $E_n(a)$. The only rule that inserts $E_i$ is $\sn{Act}_i$, so every valid plan contains $\sn{Act}_n$, which needs $D_n(a)$; no rule inserts $D_n$ and $D_n(a)$ is entailed only if $E_{n-1}(a)$ is, since $F(a)$ holds. By induction the plan is $[\sn{Act}_0,\dots,\sn{Act}_n]$. \textsc{Generate} for the goal $E_i(a)$ returns $([\sn{Act}_i(a)];\{D_i(a)\})$ for $i\ge1$ and $([\sn{Act}_0(a)];\emptyset)$ for $i=0$; \DDA{} explains $D_i(a)$ by $\{E_{i-1}(a)\}$ (the part $F(a)$ already holds). Hence level $j$ of the recursion handles $\sn{Act}_{n-j}$, the assumption set becomes empty only at level $n$, and the recursion must go $n$ levels deep.
\end{proof}

\subsection{The exact decision: classifying a candidate by projection}
\label{sec:classify}

Every candidate is decided by Algorithm~\ref{alg:classify}. It schedules the candidate (Definition~\ref{def:schedule}) and runs the \emph{temporal projection} of Algorithm~\ref{alg:project} along the intrinsic timeline of the scheduled TDKG. Then it examines the trajectory.

\begin{algorithm}[t]
\caption{$\textsc{Classify}(\KB_0,\pi,\textit{target})$ for $\pi=(a_1\dots a_n;\Gamma)$; target is a goal $G$ (plan mode) or a pattern $\gamma$ (story mode)}\label{alg:classify}
\begin{algorithmic}[1]
\Ensure $\mathsf{VALID}$, $\mathsf{STORY}$, $\mathsf{FAIL}(m)$ ($m$ the index of the first inconsistent state) or $\mathsf{REJECT}$
\State $\KB^\pi\gets\textsc{Schedule}(\KB_0,\pi)$ \Comment{Definition~\ref{def:schedule}}
\State $\langle(\KB_0^\pi\dots\KB_m^\pi),(F_1\dots F_m),\mathit{done}\rangle\gets\textsc{Project}(\KB^\pi,\mathrm{Timeline}(\KB^\pi))$ \Comment{temporal projection}
\If{$m=0$ \textbf{and not} $\mathit{done}$} \Return $\mathsf{REJECT}$ \EndIf \Comment{the assumptions contradict $\KB_0$}
\If{some $a_i$ with $i\le m$ is not effective, i.e.\ its rule and participants do not occur in $F_i$} \Return $\mathsf{REJECT}$ \EndIf \Comment{(V2)/(S2)}
\If{$\mathit{done}$} \Comment{all states consistent}
  \If{target is a goal $G$ \textbf{and} $\op{Select}(\KB_n^\pi,G)\ne\emptyset$} \Return $\mathsf{VALID}$ \EndIf \Comment{(V3)}
  \State \Return $\mathsf{REJECT}$
\EndIf
\Statex \textit{the projection stopped at an inconsistent state $\KB_m^\pi$ with $1\le m\le n$}
\If{target is a pattern $\gamma$ \textbf{and} $m=n$}
  \If{$\textsc{Reached}(\KB_n^\pi,\gamma,\Gamma)$} \Return $\mathsf{STORY}$ \Else{} \Return $\mathsf{REJECT}$ \EndIf \Comment{inconsistent, but another contradiction}
\EndIf
\State \Return $\mathsf{FAIL}(m)$ \Comment{inconsistent too early, or in plan mode}
\end{algorithmic}
\end{algorithm}

\begin{algorithm}[t]
\caption{$\textsc{Reached}(\KB,\gamma,\Gamma)$ and $\textsc{Brave}(\KB,\alpha)$, for $\KB=\langle\TB,\AB\rangle$ possibly inconsistent}\label{alg:reached}
\begin{algorithmic}[1]
\Function{Reached}{$\KB,\gamma,\Gamma$}
  \ForAll{instantiations $\theta$ of the variables of $\gamma$ by individuals of $\KB$}
    \If{every $\alpha\in\theta(\gamma)$ satisfies $\alpha\notin\Gamma$ \textbf{and} $\textsc{Brave}(\KB,\alpha)$} \Return $\mathsf{true}$ \EndIf
  \EndFor
  \State \Return $\mathsf{false}$
\EndFunction
\Function{Brave}{$\KB,\alpha$}
  \If{$\alpha\in\AB$} \Return $\mathsf{true}$ \EndIf \Comment{explicitly stated}
  \ForAll{$S\subseteq\AB$, by increasing size}
    \State $\KB'\gets\op{Copy}(\KB)$; $\ \op{Remove}(\KB',\AB\setminus S)$ \Comment{keep only the part $S$}
    \If{$\op{Consistent}(\KB')$ \textbf{and} $\op{Entails}(\KB',\alpha)$} \Return $\mathsf{true}$ \EndIf
  \EndFor
  \State \Return $\mathsf{false}$
\EndFunction
\end{algorithmic}
\end{algorithm}

The call of \textsc{Reached} in Algorithm~\ref{alg:classify} is the check that makes a story a story \emph{for the chosen pattern}: the projection may end in an inconsistent state because of an entirely different contradiction (Example~\ref{ex:foreign}); \textsc{Reached} (Algorithm~\ref{alg:reached}) rejects such candidates. Note that it looks only at the \emph{consistent parts} of the final ABox, which is the only place where entailment is meaningful.

\begin{lemma}[\textsc{Classify} is exact]\label{lem:classify}
\textsc{Classify} terminates. In plan mode it returns $\mathsf{VALID}$ iff $\pi$ is a valid hypothetical plan for $\langle\KB_0,G\rangle$. In story mode it returns $\mathsf{STORY}$ iff $\pi$ is a contradictory story for $\gamma$. A candidate is never both.
\end{lemma}
\begin{proof}
By Theorem~\ref{thm:transitions}(3) the trajectory returned by \textsc{Project} is $\KB_0^\pi,\dots,\KB_m^\pi$ with exact answers (B3), and $\mathit{done}$ holds iff $m=n$ and $\KB_n^\pi$ is consistent. Effectiveness of $a_i$ is read off $F_i$. In plan mode, $\mathsf{VALID}$ requires $\mathit{done}$, effectiveness of all actions and the goal, which is (V1)--(V3). In story mode, $\mathsf{STORY}$ requires that the projection stops at $m=n$ with an inconsistent $\KB_n^\pi$ and consistent earlier states (S1), (S3), effectiveness (S2) and $\textsc{Reached}$ (S4). \textsc{Brave} and \textsc{Reached} test Definition~\ref{def:reached} literally, over finitely many subsets and instantiations. Valid plans end in a consistent state and stories in an inconsistent one.
\end{proof}

\subsection{Deletions: the repair step}
\label{sec:repair}

\textsc{Generate} works in a relaxed world and ignores deletions. This is sound, but it can miss plans. Suppose a goal is achieved by an action $b$, but $b$ would make the state inconsistent, because an assertion $\alpha$ that is already present clashes with what $b$ inserts. A valid plan then has to delete $\alpha$ \emph{before} $b$, with an action whose only role is to avoid the inconsistency. \textsc{Generate} never proposes it, because deleting is not needed to achieve the goal. The projection shows the problem: $b$ is effective but the state after $b$ is inconsistent. The \emph{repair step} uses exactly this information.

When \textsc{Classify} returns $\mathsf{FAIL}(m)$, the state $\KB_m^\pi$ is inconsistent although the earlier ones are consistent. We compute one \emph{conflict}, that is, a minimal inconsistent subset $M$ of its ABox (Algorithm~\ref{alg:repair}). Every member of $M$ that was already present before step~$m$ is a candidate to be \emph{deleted}. For such an $\alpha$ we call \textsc{Generate} with the \emph{deletion goal} $\alpha$: the first regression step uses the delete templates $D_r$ instead of the insert templates, i.e.\ we look for an action instance $a$ whose rule has an atom of $D_r$ that unifies with $\alpha$, and then regress its preconditions as usual. The resulting sub-plan $\sigma$ is inserted immediately before $a_m$, and the new candidate is classified again. Since the deleting block comes right before $a_m$ and after $a_{m-1}$, it removes $\alpha$ whether $\alpha$ comes from $\KB_0$, from the assumptions or from an earlier action.

\begin{algorithm}[t]
\caption{$\textsc{Repair}(\KB_0,\pi,\textit{target},\rho)$ with $\textsc{GenerateDel}$ and $\textsc{ShrinkConflict}$}\label{alg:repair}
\begin{algorithmic}[1]
\Function{Repair}{$\KB_0,\pi,\textit{target},\rho$}
  \State $\mathit{res}\gets\textsc{Classify}(\KB_0,\pi,\textit{target})$
  \If{$\mathit{res}\in\{\mathsf{VALID},\mathsf{STORY}\}$} \Return $\{\pi\}$ \EndIf
  \If{$\mathit{res}\ne\mathsf{FAIL}(m)$ \textbf{or} $\rho=0$ \textbf{or} $|\pi|\ge d$} \Return $\emptyset$ \EndIf
  \State $M\gets\textsc{ShrinkConflict}(\KB_m^\pi)$ \Comment{$\KB_m^\pi$ is the inconsistent state found by \textsc{Project}}
  \State $\mathit{Out}\gets\emptyset$
  \ForAll{$\alpha\in M\cap\AB_{m-1}^\pi$} \Comment{members present before step $m$}
    \ForAll{$(\sigma;\Gamma_\sigma)\in\textsc{GenerateDel}(\KB_0,\alpha)$ with $|\sigma|\le d-|\pi|$}
      \State $\pi'\gets(a_1\dots a_{m-1},\ \sigma,\ a_m\dots a_n;\ \Gamma\cup\Gamma_\sigma)$ \Comment{deleting block right before $a_m$}
      \State $\mathit{Out}\gets\mathit{Out}\cup\textsc{Repair}(\KB_0,\pi',\textit{target},\rho-1)$
    \EndFor
  \EndFor
  \State \Return $\mathit{Out}$
\EndFunction
\Statex
\Function{GenerateDel}{$\KB_0,\alpha$} \Comment{like \textsc{Generate}, but the first regression step uses $D_r$}
  \ForAll{rules $r\in\RS$ and unifiers $\sigma$ of $\alpha$ with an atom of $D_r$ such that $\alpha\notin\sigma(I_r)$}
    \State start \textsc{Generate} (Algorithm~\ref{alg:generate}) at the node $(\sigma(\mathit{pre}(r)),\ [\sigma(a_r)],\ \emptyset,\ \KB_0\cup\sigma(I_r),\ 1,\ 1)$
  \EndFor
  \State \Return the union of the candidates found
\EndFunction
\Statex
\Function{ShrinkConflict}{$\KB$} \Comment{$\KB=\langle\TB,\AB\rangle$ inconsistent}
  \State $M\gets\AB$
  \ForAll{$\beta\in\AB$ in a fixed order}
    \State $\KB'\gets\op{Copy}(\KB)$; $\ \op{Remove}(\KB',\AB\setminus(M\setminus\{\beta\}))$
    \If{$\neg\op{Consistent}(\KB')$} $M\gets M\setminus\{\beta\}$ \EndIf
  \EndFor
  \State \Return $M$
\EndFunction
\end{algorithmic}
\end{algorithm}

\begin{lemma}[Conflict extraction]\label{lem:shrink}
For every inconsistent $\KB=\langle\TB,\AB\rangle$ for which $\langle\TB,\emptyset\rangle$ is consistent, \textsc{ShrinkConflict} returns a conflict (a minimal inconsistent subset of $\AB$) using $|\AB|$ calls of $\op{Consistent}$.
\end{lemma}
\begin{proof}
The invariant ``$\langle\TB,M\rangle$ is inconsistent'' holds at the start and is preserved by the test. When $\beta$ is kept, $\langle\TB,M\setminus\{\beta\}\rangle$ was consistent; later $M$ only shrinks, and every subset of a consistent set is consistent, so $\beta$ stays necessary. Hence no proper subset of the final $M$ is inconsistent.
\end{proof}

\begin{proposition}[What repair adds]\label{prop:repair}
Let $\pi$ be a candidate for which \textsc{Classify} returns $\mathsf{FAIL}(m)$, let $M$ be the conflict found, $\alpha\in M\cap\AB_{m-1}^\pi$, and let $\sigma$ be a block produced by $\textsc{GenerateDel}(\KB_0,\alpha)$ with $|\sigma|+|\pi|\le d$. Then the candidate $\pi'$ obtained by inserting $\sigma$ before $a_m$ is among those classified by $\textsc{Repair}(\KB_0,\pi,\textit{target},\rho)$ for $\rho\ge1$. \textsc{Repair} terminates, since every recursive call adds at least one action and $|\pi|<d$; and it is sound, since only \textsc{Classify} decides.
\end{proposition}
\begin{proof}
By construction of Algorithm~\ref{alg:repair}; the recursion depth is at most $\min(\rho,d)$, and each level has finitely many conflict members and blocks.
\end{proof}

\begin{example}[A plan whose first action only deletes]\label{ex:repair}
Take the tax terminology, in which \sn{TaxExempt} and \sn{TaxPayer} are disjoint, the state $\KB_0=\{\sn{Trader}(\sn{joe}),\sn{TaxExempt}(\sn{joe})\}$ and the goal $G=\{\sn{TaxPayer}(\sn{joe})\}$. Besides \ImportCar (insert $\sn{TaxPayer}(p)$ for a trader $p$), there is a rule \Renounce, \texttt{DELETE \{ ?p a TaxExempt \} WHERE \{ \ldots\ ?p a Trader \}}; this rule is constructed for the example and is not part of the benchmark files.
\begin{center}\small
\begin{tabularx}{\textwidth}{@{}l P{5.6cm} X@{}}
\toprule
\textbf{Step} & \textbf{Candidate} & \textbf{Result}\\ \midrule
\textsc{Generate} & $[\ImportCar(\sn{joe})]$ & regresses $\sn{TaxPayer}(\sn{joe})$ through \ImportCar; the precondition $\sn{Trader}(\sn{joe})$ is absorbed; deletions are ignored, so \Renounce is never proposed\\
\textsc{Classify} & same & projection: step~1 inserts $\sn{TaxPayer}(\sn{joe})$ and the state is inconsistent; the result is $\mathsf{FAIL}(1)$\\
\textsc{ShrinkConflict} & $M=\{\sn{TaxExempt}(\sn{joe}),\sn{TaxPayer}(\sn{joe})\}$ & the member $\sn{TaxExempt}(\sn{joe})$ was present before step~1\\
\textsc{GenerateDel} & $\alpha=\sn{TaxExempt}(\sn{joe})$ & \Renounce has a delete atom that unifies with $\alpha$; its precondition is absorbed; the block is $[\Renounce(\sn{joe})]$\\
\textsc{Classify} & $[\Renounce(\sn{joe}),\ \ImportCar(\sn{joe})]$ & after step~1 only $\sn{Trader}(\sn{joe})$ remains; after step~2 also $\sn{TaxPayer}(\sn{joe})$; both states are consistent and the goal holds: $\mathsf{VALID}$\\
\bottomrule
\end{tabularx}
\end{center}
\end{example}

\subsection{The main algorithm and its guarantees}

\begin{algorithm}[t]
\caption{$\textsc{APT}(\KB_0,G,\RS,d,L,R,\rho)$}\label{alg:apt}
\begin{algorithmic}[1]
\State $\mathit{Out}\gets\emptyset$
\ForAll{$(\pi_0;\Gamma_0)\in\textsc{Generate}(\KB_0,G,d)$ \textbf{and} $(\pi;\Gamma)\in\textsc{Enrich}_R(\pi_0,\Gamma_0)$}
  \State $\mathit{Out}\gets\mathit{Out}\cup\textsc{Repair}(\KB_0,(\pi;\Gamma),G,\rho)$ \Comment{classifies by temporal projection, repairs if needed}
\EndFor
\State \Return $\mathit{Out}$ ordered by number of assumptions, then by length
\end{algorithmic}
\end{algorithm}

\begin{theorem}[Soundness]\label{thm:apt-sound}
Every plan returned by Algorithm~\ref{alg:apt} is a valid hypothetical plan, and a plan if its assumption set is empty. This holds for every generator, every enrichment strategy and every repair strategy.
\end{theorem}
\begin{proof}
A plan is returned only if \textsc{Classify} says $\mathsf{VALID}$; apply Lemma~\ref{lem:classify}.
\end{proof}

\begin{theorem}[Termination]\label{thm:apt-term}
Algorithm~\ref{alg:apt} terminates. \textsc{Generate} is finite (Lemma~\ref{lem:gen-space}); \textsc{Enrich} recurses at most $R$ levels deep, each time with finitely many explanations (Theorem~\ref{thm:dda-term}) and sub-candidates, so the recursion tree is finite (Proposition~\ref{prop:rec}); \textsc{Repair} recurses at most $\min(\rho,d)$ times (Proposition~\ref{prop:repair}); and each classification of a candidate with $n$ actions uses the at most $n(|\RS|\,\iota^{v}a+4)+1$ reasoner calls of the projection (Theorem~\ref{thm:transitions}) and at most $\iota^{|\mathrm{vars}(\gamma)|}\cdot|\gamma|\cdot2^{|\AB|}$ further calls for the pattern check.
\end{theorem}

\begin{theorem}[Decidability of bounded planning]\label{thm:apt-decid}
The bounded planning problem and the bounded story problem (Section~\ref{sec:plans}) are decidable: the sets of all valid plans, resp.\ stories, in $\mathrm{Cand}_d$ are computable.
\end{theorem}
\begin{proof}
$\mathrm{Cand}_d$ is finite and computable, with at most $\sum_{m\le d}(|\RS|\,N^{w})^{m}\cdot2^{|\Omega_d|}$ elements, where $N$ is the number of allowed individuals and $w$ the maximal number of participants. Run \textsc{Classify} on each candidate; by Lemma~\ref{lem:classify} it decides exactly.
\end{proof}

\begin{corollary}\label{cor:apt-complete}
If \textsc{Generate} is replaced by the exhaustive generator that returns all of $\mathrm{Cand}_d$, and \textsc{Enrich} and \textsc{Repair} by the identity, Algorithm~\ref{alg:apt} returns exactly all valid plans of $\mathrm{Cand}_d$. The heuristic generator returns a subset of them; we do not claim that it is complete.
\end{corollary}

The heuristic generator can miss plans for three reasons: it follows only largest absorbing subsets; it stops at depth $d$; and, although \textsc{Repair} recovers the plans in which an action merely deletes something that would clash (Proposition~\ref{prop:repair}), it only repairs conflicts that arise in the projection of a candidate that \textsc{Generate} proposed, with at most $\rho$ repairs.

\begin{theorem}[Hardness]\label{thm:apt-hard}
Deciding whether a valid plan exists in $\mathrm{Cand}_d$ is $2$\textsc{NExpTime}-hard, already for $d=1$, a single insert-only rule and the empty goal.
\end{theorem}
\begin{proof}
Reduce from consistency of a $\SROIQ$ KB $\langle\TB,\AB\rangle$ \cite{kazakov2008riq}. Take a fresh action class $A$ and let $\KB_0=\langle\TB\cup\{A\sqsubseteq\sn{Action}\},\emptyset\rangle$. The only rule has body $\{A(?e),\sn{hasTime}(?e,?\_T)\}$, no deletion, and inserts the assertions of $\AB$. If $\langle\TB,\AB\rangle$ is consistent, the one-action plan with $\Gamma=\emptyset$ is valid: its scheduled KB is consistent (the new symbols are unconstrained), the action fires, and the final KB is consistent. Conversely, a valid plan has a consistent final KB, which contains $\TB$ and $\AB$, hence $\langle\TB,\AB\rangle$ is consistent by monotonicity.
\end{proof}

\begin{theorem}[Upper bound]\label{thm:apt-upper}
Let the number of variables of rule bodies and of the goal be bounded by a constant and let $d$ be given in unary. Deciding whether a valid plan exists in $\mathrm{Cand}_d$ is in \textsc{NP} with an oracle for the consistency of $\SROIQ$ knowledge bases, hence in $3$\textsc{ExpTime}. Together with Theorem~\ref{thm:apt-hard}, the problem is $2$\textsc{NExpTime}-hard and in $3$\textsc{ExpTime}; the same holds for the bounded story problem.
\end{theorem}
\begin{proof}
Let $s$ be the size of the input. The number $N$ of allowed individuals is at most $|\Ind(\KB_0)|$ plus the constants of the rules and the goal plus $d\cdot v+|\mathrm{vars}(G)|$, which is polynomial in $s+d$. The concept expressions and roles of $\Omega_d$ occur in $\TB$, the rules and $G$, so there are at most $2s$ expressions and $s$ roles, and $|\Omega_d|\le 2sN+sN^{2}$ is polynomial. A candidate $\langle a_1,\dots,a_m;\Gamma\rangle$ with $m\le d$ consists of $m$ rule names with at most $v$ participants each and a subset of $\Omega_d$, hence has polynomial size. Guess it. By Lemma~\ref{lem:classify} it is a valid plan iff \textsc{Classify} returns $\mathsf{VALID}$, and \textsc{Classify} performs a projection and a goal test. By Theorem~\ref{thm:transitions} the projection makes at most $n(|\RS|\,\iota'^{v}a+4)+1$ reasoner calls, and the goal test makes at most $\iota'^{|\mathrm{vars}(G)|}|G|$ entailment calls. Since $v$ and $|\mathrm{vars}(G)|$ are constants, each state of the trajectory has at most $\iota'^{v}$ matches per step and therefore polynomial size, by induction over the steps, so $\iota'$ is polynomial, and the number of calls is polynomial. Every call is a consistency or entailment check on a knowledge base of polynomial size; entailment reduces to consistency (B3). This gives a nondeterministic polynomial-time machine with an oracle for $\SROIQ$ consistency. The oracle problem is in $2$\textsc{NExpTime} \cite{kazakov2008riq} and hence decidable deterministically in time $2^{2^{2^{q(s+d)}}}$ for a polynomial $q$, by simulating all computation paths. There are $2^{p(s+d)}$ candidates and each needs polynomially many oracle calls, so the exhaustive search runs in time $2^{2^{2^{q'(s+d)}}}$ for a polynomial $q'$. Stories are decided by the same procedure (Lemma~\ref{lem:classify}).
\end{proof}

The lower bound is inherited from the reasoner: a single call of the reasoner already costs $2$\textsc{NExpTime}. Relative to a reasoner the combinatorial part is \textsc{NP}, as for classical planning (Proposition~\ref{prop:cp-np}). Whether the exact complexity lies at $2$\textsc{NExpTime}, at $\mathrm{NP}^{2\textsc{NExpTime}}$ or between them is open: a nondeterministic machine can guess a candidate and the positive answers of the reasoner, but the negative answers, which the matching of rules needs, call for the complementary class.

\begin{remark}
For unbounded plan length together with the creation of new individuals we make no claim. Related formalisms with fresh values are undecidable in general and decidable under boundedness assumptions \cite{hariri2013verification,erol1995complexity}.
\end{remark}

\subsection{What the ranking of plans means}
\label{sec:quality}

\APT{} returns its plans ordered by the number of assumptions and then by length, but this order only ranks the plans that the search has found; it is not a guarantee that the first plan has the fewest actions, and the engine is not designed to find plans with the fewest actions. Two reasons make this a deliberate choice. First, the questions that \APT{} answers are about \emph{validity and support}: is there a plan that the ontology and the rules make effective, which facts does it rely on, and why are they plausible? A plan that is valid and has a small set of explained assumptions answers them, whatever its length. Second, shortest plans are the business of classical planners with admissible heuristics; where the ontology lies in the compilable fragment of Section~\ref{sec:cmp-plan} such a planner can be used directly. \APT{} complements it where the open world, created individuals and consistency matter.

\subsection{Example: opening the door}

\begin{example}[Secure door]\label{ex:door}
$\TB$ contains $\sn{SecureDoor}\equiv\sn{HingedStructure}\sqcap\sn{WoodenStructure}$. $\AB_0=\{\sn{MechanicalLock}(\sn{FrontDoorLock})\}$ and the goal is $G=\{\sn{Opened}(\sn{dor})\}$. $\OperateHandle$ needs $\sn{SecureDoor}(d)$, $\sn{isInstalledOn}(l,d)$, $\sn{MechanicalLock}(l)$ and $\sn{Unlocked}(l)$, and inserts $\sn{Opened}(d)$; $\TurnKey$ inserts $\sn{Unlocked}(l)$ for a $\sn{MechanicalLock}$ $l$ (Listing~\ref{lst:rules}).

\emph{Generate.} Regression through $\OperateHandle$ gives four subgoals. $\sn{MechanicalLock}(l)$ is absorbed with $l=\sn{FrontDoorLock}$; $\sn{Unlocked}(\sn{FrontDoorLock})$ is regressed through $\TurnKey$. No rule inserts $\sn{SecureDoor}(\sn{dor})$ or $\sn{isInstalledOn}(\sn{FrontDoorLock},\sn{dor})$, so they are assumed: $\Gamma_0$ consists of these two atoms.

\emph{Enrich.} We run \DDA{} on the two assumptions as observations:
\[o_1=\sn{SecureDoor}(\sn{dor}),\qquad o_2=\sn{isInstalledOn}(\sn{FrontDoorLock},\sn{dor}).\]
The reasoner gives $\Sugg(o_1)=\{o_1,\ (\sn{HingedStructure}\sqcap\sn{WoodenStructure})(\sn{dor})\}$, so $o_1$ is reducible, and $\Sugg(o_2)=\{o_2\}$, so $o_2$ is irreducible. Five hypotheses are reachable: $\{o_1,o_2\}$, $\{\sn{Hinged}(\sn{dor}),\sn{Wooden}(\sn{dor}),o_2\}$ and three hypotheses that mix $o_1$ with one or both parts. Only $\{\sn{Hinged}(\sn{dor}),\sn{Wooden}(\sn{dor}),o_2\}$ survives (E), (M) and (R); the others collapse onto it or onto $\{o_1,o_2\}$, which restates the reducible $o_1$. No rule produces these atoms, so \textsc{Generate} finds no sub-plan for this explanation and the recursion stops at level~1: the atoms remain as the residual assumptions. (Example~\ref{ex:lab} shows a recursion that continues for two levels.)

\emph{Temporal projection.} The scheduled TDKG is consistent. Playing its timeline, $\TurnKey$ fires at $\tau_1$ and inserts $\sn{Unlocked}(\sn{FrontDoorLock})$. At $\tau_2$, $\sn{SecureDoor}(\sn{dor})$ holds \emph{by entailment} from the assumptions, so $\OperateHandle$ fires and inserts $\sn{Opened}(\sn{dor})$. The plan is valid. No rule inserted $\sn{SecureDoor}$, yet the action that needs it fired: that is the ramification mechanism of Theorem~\ref{thm:ramif} at work.
\end{example}

\section{Adversarial Plan Synthesis}
\label{sec:adversarial}

A contradictory story (Definition~\ref{def:valid}) is a sequence of effective actions that starts in a consistent state and ends in an inconsistent one, the inconsistency being an instance of a chosen contradiction pattern. Finding such stories reveals domain-model errors: a missing deletion, a missing precondition, or a constraint that is too strong. Ordinary planners prune inconsistent states; here inconsistency is the \emph{target}.

\paragraph{Any contradiction, not a fixed list.} One could hard-code a few kinds of structural contradictions (disjoint classes, irreflexive roles, \dots). But the reasoner itself already knows what a contradiction is. We therefore define contradictions semantically (Section~\ref{sec:plans}), let the reasoner recognise them, and obtain the patterns from several interchangeable sources.

\begin{table}[t]
\centering\small
\caption{Contradiction patterns for different axioms of $\TB$. The first rows are the structural categories one would hard-code; the later rows are further instances of the same definition, including the generic \emph{complement} patterns that need no particular axiom.}
\label{tab:patterns}
\begin{tabularx}{\textwidth}{@{}P{4.3cm}P{4.9cm}X@{}}
\toprule
\textbf{Axiom in $\TB$} & \textbf{Pattern $\gamma$} & \textbf{Why inconsistent}\\ \midrule
disjoint classes $C_1,C_2$ & $\{C_1(x),C_2(x)\}$ & $C_1\cap C_2=\emptyset$\\
disjoint properties $R_1,R_2$ & $\{R_1(x,y),R_2(x,y)\}$ & $R_1\cap R_2=\emptyset$\\
equality and inequality & $\{x\approx y,\ x\not\approx y\}$ & contradiction\\
asymmetric $R$ & $\{R(x,y),R(y,x)\}$ & asymmetry\\
irreflexive $R$ & $\{R(x,x)\}$ & irreflexivity\\ \midrule
functional $R$ & $\{R(x,y_1),R(x,y_2),y_1\not\approx y_2\}$ & two distinct values\\
at most $n$ $R$-successors & $n{+}1$ distinct successors & cardinality\\
domain axiom plus disjointness & $\{R(x,y),C_2(x)\}$ & domain clash\\
\emph{none needed}: any class $C$ & $\{C(x),\ \neg C(x)\}$ & an individual cannot be in $C$ and not in $C$\\
\emph{none needed}: any role $R$ & $\{R(x,y),\ \neg R(x,y)\}$ & a pair cannot be in $R$ and not in $R$\\
any other axiom the reasoner supports & found as in Section~\ref{sec:sources} & as detected by \op{Consistent}\\
\bottomrule
\end{tabularx}
\end{table}

The complement patterns $\{C(x),\neg C(x)\}$ deserve a remark. Together with a disjointness axiom $C_1\sqcap C_2\sqsubseteq\bot$ they are linked by a fact that the planner must \emph{discover}: $C_2(x)$ entails $\neg C_1(x)$. Reaching $\neg C_1(x)$ is not an effect of any action, so a planner that looks only at the effects of the rules cannot reach it. Our planner treats $\neg C_1(x)$ as an observation and asks \DDA{} what would conclude it (Example~\ref{ex:tax}).

\begin{proposition}[Every inconsistency has a pattern]\label{prop:coverage}
Let $\langle\TB,\emptyset\rangle$ be consistent. (a) $\langle\TB,\AB\rangle$ is inconsistent iff $\AB$ contains a conflict, and every conflict is an instance of a contradiction pattern (obtained by replacing individuals that do not occur in $\TB$ by variables). (b) Each pattern of Table~\ref{tab:patterns} is a contradiction pattern of every $\TB$ that contains the corresponding axiom. (c) If $\langle\TB,\AB\rangle$ is inconsistent and $M\subseteq\AB$ is a conflict consisting of class assertions and assertions on simple roles, then: if $M=\{\beta\}$, the one-atom pattern obtained from $\beta$ is reached in $\langle\TB,\AB\rangle$; if $|M|\ge2$, then for every $\beta\in M$ the complement pattern obtained from $\{\beta,\neg\beta\}$ is reached.
\end{proposition}
\begin{proof}
(a) If $\AB$ contains an inconsistent subset then $\AB$ is inconsistent. Conversely, $\AB$ is finite, so a minimal inconsistent subset exists; it is non-empty since $\langle\TB,\emptyset\rangle$ is consistent, and it is a conflict. (b) Each row is the direct semantic reading of the axiom. (c) $\beta\in\AB$ is bravely entailed. If $|M|\ge2$, then $S=M\setminus\{\beta\}$ is consistent (minimality) and $S\cup\{\beta\}$ is inconsistent, so $S\models\neg\beta$ by Lemma~\ref{lem:neg}(a); hence $\neg\beta$ is bravely entailed too.
\end{proof}

Part (c) says that the finite family of one-atom and complement patterns over the concept expressions and roles of $\TB$ already covers every inconsistency whose conflict uses these expressions. This is the formal reason why no fixed list of contradiction categories is needed.

\begin{lemma}[The last action is a culprit]\label{lem:culprit}
In a contradictory story, every conflict contained in the final ABox $\AB_n$ contains an assertion that the last step really added.
\end{lemma}
\begin{proof}
By Theorem~\ref{thm:ramif}, $\AB_n\setminus\AB_n^{+}\subseteq\AB_{n-1}$. A conflict avoiding $\AB_n^{+}$ would lie in $\AB_{n-1}$, contradicting consistency of $\KB_{n-1}$ (S1). A conflict exists by Proposition~\ref{prop:coverage}(a).
\end{proof}

\subsection{Where patterns come from}
\label{sec:sources}

A \emph{pattern source} is any procedure that returns contradiction patterns for a terminology. Four are natural.
\begin{enumerate}
\item \emph{Syntactic extractors}: SPARQL queries over the axioms of $\TB$, one per row of Table~\ref{tab:patterns}.
\item \emph{The generic family}: for every class name or expression $E$ occurring in $\TB$ the patterns $\{E(x)\}$ and $\{E(x),\neg E(x)\}$, and for every role $R$ the pattern $\{R(x,y),\neg R(x,y)\}$. It needs no knowledge about the kind of axiom, only the vocabulary.
\item \emph{Justifications}: if the reasoner can explain why a probe KB is inconsistent \cite{kalyanpur2007justifications}, the explanation is a conflict.
\item \emph{The bounded miner}: enumerate all sets of at most $k$ literal atoms (concept names or their negations, roles or their negations, equalities and inequalities) over at most $2k$ variables and the constants of $\TB$, instantiate the variables with distinct fresh names and test each instance with \op{Consistent}.
\end{enumerate}

\begin{proposition}[Miner]\label{prop:miner}
The bounded miner terminates after at most $\sum_{j\le k}\binom{m}{j}$ consistency checks, where $m$ is the number of literal atoms; it finds an instance of every conflict of at most $k$ literal assertions. There is no uniform $k$: the terminology $\{\top\sqsubseteq{\le}k\,R.\top\}$ has only conflicts with more than $k$ assertions.
\end{proposition}
\begin{proof}
Every conflict of at most $k$ literals uses at most $2k$ individuals that do not occur in $\TB$; replacing them by variables gives a set the miner enumerates, and its instance with distinct fresh names is the conflict up to renaming. For the last claim a conflict needs $k{+}1$ pairwise distinct $R$-successors.
\end{proof}

Soundness of the engine does not depend on the source: an unsound pattern only wastes search effort, since \textsc{Classify} decides every candidate.

\subsection{Synthesis}

\begin{algorithm}[t]
\caption{$\textsc{Adv}(\KB_0,\RS,\text{source},d,L,R,\rho)$}\label{alg:adv}
\begin{algorithmic}[1]
\State $\mathit{Stories}\gets\emptyset$
\ForAll{patterns $\gamma$ from the source}
  \ForAll{$(\pi_0;\Gamma_0)\in\textsc{Generate}(\KB_0,\gamma,d)$ \textbf{and} $(\pi;\Gamma)\in\textsc{Enrich}_R(\pi_0,\Gamma_0)$}
    \ForAll{$\pi'\in\textsc{Repair}(\KB_0,(\pi;\Gamma),\gamma,\rho)$}
      \State $\mathit{Stories}\gets\mathit{Stories}\cup\{(\gamma,\pi')\}$ \Comment{\textsc{Classify} returned $\mathsf{STORY}$: $\gamma$ is reached}
    \EndFor
  \EndFor
\EndFor
\State \Return $\mathit{Stories}$
\end{algorithmic}
\end{algorithm}

The pattern is used as an ordinary goal whose variables are existential. The optimistic generator regresses it as if it were achievable, which is exactly what makes the search possible: a planner restricted to consistent states would never propose the goal $\{C_1(x),C_2(x)\}$ for disjoint classes. Atoms that no rule can produce, such as $\neg C(x)$, become assumptions and are handed to \textsc{Enrich}, which asks \DDA{} for explanations and plans for them. \textsc{Classify} in story mode then separates real stories from candidates that fail, for instance because a deletion repairs the clash, and \emph{from candidates that end in a different contradiction} than $\gamma$. \textsc{Repair} is used here for inconsistencies that occur \emph{before} the last step; the inconsistency at the last step is the target and is never repaired.

\begin{theorem}[Adversarial synthesis]\label{thm:adv}
(a) Every story returned by Algorithm~\ref{alg:adv} is a contradictory story for the pattern it is paired with, for every pattern source; in particular the chosen pattern, and not merely some inconsistency, is reached. (b) The algorithm terminates, and the bounded story problem is decidable (Theorem~\ref{thm:apt-decid}). (c) With the exhaustive generator and the generic family of patterns, every plan of $\mathrm{Cand}_d$ that satisfies (S1)--(S3) and whose last conflict consists of such assertions over expressions of the family is returned for at least one pattern. (d) If a story has a conflict of at most $k$ literal assertions, the miner produces its pattern, and the conflict contains an assertion added by the last step (Lemma~\ref{lem:culprit}).
\end{theorem}
\begin{proof}
(a) Lemma~\ref{lem:classify}. (b) Theorems~\ref{thm:apt-term} and~\ref{thm:apt-decid}. (c) Let $\pi$ satisfy (S1)--(S3). By Proposition~\ref{prop:coverage}(a) and Lemma~\ref{lem:culprit} the final ABox has a conflict $M$ with a member $\beta$ added by the last step, hence $\beta\notin\Gamma$. By Proposition~\ref{prop:coverage}(c) the one-atom pattern of $\beta$ (if $|M|=1$) or its complement pattern is reached without using $\Gamma$, and it belongs to the generic family. The exhaustive generator proposes $\pi$, and \textsc{Classify} returns $\mathsf{STORY}$ by Lemma~\ref{lem:classify}. (d) Proposition~\ref{prop:miner} and Lemma~\ref{lem:culprit}.
\end{proof}

Which contradictions are recognised is therefore determined solely by the reasoner's \op{Consistent}, i.e.\ by the expressivity of the reasoner, and not by the planner.

\begin{example}[Taxes: a contradiction that needs abduction]\label{ex:tax}
$\TB=\{\mathsf{DisjointClasses}(\sn{TaxExempt},\sn{TaxPayer}),\ \sn{Trader}\sqsubseteq\sn{Human}\}$ and $\AB_0=\{\sn{Trader}(\sn{TraderJoe})\}$. The rule \ImportWheat of Listing~\ref{lst:rules} deletes $\sn{TaxPayer}(p)$ and inserts $\sn{TaxExempt}(p)$; the rule \ImportCar only inserts $\sn{TaxPayer}(p)$; both need $\sn{Trader}(p)$. We take the generic complement pattern $\gamma_1=\{\sn{TaxPayer}(x),\ \neg\sn{TaxPayer}(x)\}$, which can be produced without reading the disjointness axiom. Finding the story needs three abilities together: the contradiction is searched for (the story machinery), $\neg\sn{TaxPayer}(\sn{TraderJoe})$ is an observation that has to be \emph{explained} (abduction), and the explanation has to be \emph{achieved by a plan of its own}. Table~\ref{tab:tax-trace} shows the run; the suggestion sets were computed with the HermiT reasoner by the implementation.

\begin{table}[t]
\centering\small
\caption{Example~\ref{ex:tax}: the pipeline on the complement pattern $\gamma_1$.}
\label{tab:tax-trace}
\begin{tabularx}{\textwidth}{@{}l P{4.0cm} X@{}}
\toprule
\textbf{Stage} & \textbf{Input / node} & \textbf{Outcome}\\ \midrule
\textsc{Generate} & goal $\gamma_1$ & $\sn{TaxPayer}(x)$ is regressed through \ImportCar with the agent variable open; the new subgoals are $\neg\sn{TaxPayer}(x)$ and $\sn{Trader}(x)$. Absorption binds $x=\sn{TraderJoe}$, since $\sn{Trader}(\sn{TraderJoe})$ holds. No rule inserts $\neg\sn{TaxPayer}(\sn{TraderJoe})$, so it is \emph{assumed}. Candidate $([\ImportCar(\sn{TraderJoe})];\Gamma_0)$ with $\Gamma_0=\{\neg\sn{TaxPayer}(\sn{TraderJoe})\}$\\ \midrule
\textsc{Enrich}: \DDA & observation $\Gamma_0$ in $\KB_0$ & $\Sugg(\neg\sn{TaxPayer}(\sn{TraderJoe}))=\{\neg\sn{TaxPayer}(\sn{TraderJoe}),\ \sn{TaxExempt}(\sn{TraderJoe})\}$: assuming $\sn{TaxPayer}(\sn{TraderJoe})$ forces $\neg\sn{TaxExempt}(\sn{TraderJoe})$ by disjointness, and its negation $\sn{TaxExempt}(\sn{TraderJoe})$ concludes the observation. The only relevant explanation is $H=\{\sn{TaxExempt}(\sn{TraderJoe})\}$\\ \midrule
\textsc{Enrich}: \textsc{Generate} & new goal $H$ & $\sn{TaxExempt}(\sn{TraderJoe})$ is regressed through \ImportWheat; $\sn{Trader}(\sn{TraderJoe})$ is absorbed. Sub-plan $([\ImportWheat(\sn{TraderJoe})];\emptyset)$. The result of enrichment is $([\ImportWheat(\sn{TraderJoe}),\ImportCar(\sn{TraderJoe})];\ \emptyset)$\\ \midrule
\textsc{Classify} & the candidate, story mode & The scheduled KB is consistent. Step~1: \ImportWheat fires, deletes nothing and inserts $\sn{TaxExempt}(\sn{TraderJoe})$; consistent. Step~2: \ImportCar fires and inserts $\sn{TaxPayer}(\sn{TraderJoe})$; the state is inconsistent. \textsc{Reached}: $\sn{TaxPayer}(\sn{TraderJoe})$ is stated; $\neg\sn{TaxPayer}(\sn{TraderJoe})$ follows from the consistent part $\{\sn{TaxExempt}(\sn{TraderJoe})\}$; neither is in $\Gamma$. Result: $\mathsf{STORY}$\\
\bottomrule
\end{tabularx}
\end{table}

{\sloppy Three observations show what each part contributes. First, the candidate $([\ImportCar(\sn{TraderJoe})];\Gamma_0)$ \emph{without} enrichment would also end in an inconsistent state, but only because $\Gamma_0$ itself contains $\neg\sn{TaxPayer}(\sn{TraderJoe})$; it is rejected by the condition $\alpha\notin\Gamma$ of \textsc{Reached}. Second, for the other complement pattern $\gamma_2=\{\sn{TaxExempt}(x),\neg\sn{TaxExempt}(x)\}$ the same pipeline suggests $\sn{TaxPayer}(\sn{TraderJoe})$ as explanation of $\neg\sn{TaxExempt}(\sn{TraderJoe})$ and produces the candidate $[\ImportCar(\sn{TraderJoe}),\ImportWheat(\sn{TraderJoe})]$; the projection finds that \ImportWheat \emph{really removes} $\sn{TaxPayer}(\sn{TraderJoe})$ and ends in a consistent state, so the candidate is rejected. The diagnosis is the one of Proposition~\ref{prop:scope}(i): the ontology tolerates one order only because \ImportWheat deletes \sn{TaxPayer}, and \ImportCar has no symmetric deletion. Third, the axiom-specific pattern $\gamma_0=\{\sn{TaxExempt}(x),\sn{TaxPayer}(x)\}$ would have made the example easy, because both atoms are inserted by rules; the complement pattern is the harder and more general case. The plan is the one computed by the implementation (Appendix~\ref{app:traces}).\par}
\end{example}

\section{Parallelism and Caching}
\label{sec:parcache}

The algorithms of Sections~\ref{sec:dda}--\ref{sec:adversarial} are written sequentially, but nearly all of their reasoner calls are independent of each other, and many of them repeat. The implementation exploits both facts: it runs independent checks in parallel, which uses the optional reasoner ability (B8) of Section~\ref{sec:backend}, and it remembers results, which is done by the tools themselves and is not a reasoner ability. Both are engineering devices; this section states where they are used, why they cannot change any result, and what they can and cannot save.

\subsection{Where the tools run in parallel}

Isolation (B1) is what makes parallelism possible: every task works on its own copy of a reasoner state and never modifies its parent, so tasks cannot disturb each other. Table~\ref{tab:par} lists the independent units. Every loop in the table builds its result as a \emph{set}, by union, so the order in which the iterations finish is irrelevant; the implementation joins the tasks of a loop in the order of the sequential loop, so that the output is the same.

\begin{table}[t]
\centering\small
\caption{Independent units of work and how the implementation runs them. Every check is a call on a fresh copy of the knowledge base; the number of simultaneous calls is the number of reasoner processes.}
\label{tab:par}
\begin{tabularx}{\textwidth}{@{}P{2.2cm}P{4.6cm}X@{}}
\toprule
\textbf{Tool} & \textbf{Independent units (Algorithm)} & \textbf{Implementation}\\ \midrule
Projection & the bodies of all rules at one instant, and all assignments of the variables of a body, are evaluated on the same pre-state (Alg.~\ref{alg:project}, line~4) & one task per rule and per assignment; the instants of one projection depend on each other and run one after another\\
Classify and Repair & different candidates (Alg.~\ref{alg:apt}); the repairs for the different atoms of a conflict & one task per candidate and per atom\\
\textsc{Generate} & the entailment checks of the absorption of open goals, one per assignment of the open variables (Alg.~\ref{alg:generate}) & one task per assignment; the expansion of the nodes keeps its order\\
\DDA & the checks of \textsc{Suggest} over $\Omega$ (Alg.~\ref{alg:succ}, stage A); the entailment checks of \textsc{Witnesses} (stage B); \textsc{Minimise} for different hypotheses (Alg.~\ref{alg:dda}, last line) & one task per check or hypothesis\\
\textsc{Enrich} & the explanations $H\in\mathcal H$ with their sub-plans and enrichments (Alg.~\ref{alg:enrich}, line~5); the candidates of the first \textsc{Generate} call & one task per explanation and per candidate\\
\textsc{Adv} & the contradiction patterns $\gamma$ of the source (Alg.~\ref{alg:adv}); the checks of the pattern miner of one size & one task per pattern and per check\\
\bottomrule
\end{tabularx}
\end{table}

The number of reasoner calls that run at the same time is bounded by the number of reasoner processes, a parameter of the implementation, whatever the number of tasks: a task that needs a call takes a free process and waits otherwise. Every call builds its knowledge base from the terminology and the assertions it is given, so tasks share no mutable state. Parallel loops nest, since an enrichment runs the checks of \DDA{} and a repair runs \textsc{Generate}; a thread that waits for its tasks runs the tasks that no other thread has started, so no thread waits for work that is queued behind it.

\begin{proposition}[Parallel evaluation does not change the result]\label{prop:parallel}
Let a loop of one of the algorithms have iterations that (i) read only the data available before the loop and a private copy of a reasoner state, and (ii) contribute to the result only by set union. Then any interleaving of the iterations, in particular a parallel one, returns the same set as the sequential loop, up to the names of placeholder individuals. In particular all theorems of this paper hold for the parallel implementation.
\end{proposition}
\begin{proof}
By (B1) and (B3) every reasoner answer in an iteration is a function of the private copy and the inputs, so each iteration computes the same value in every interleaving. A union does not depend on the order of its operands. Placeholder names come from a counter; renaming them is an isomorphism of candidates, and \textsc{Classify} is invariant under it because entailment is invariant under renaming of individuals that do not occur in $\TB$ (Section~\ref{sec:plans}).
\end{proof}

\paragraph{What parallelism can save.} Let the total reasoner time of a run be the \emph{work} $W$ and the time along the longest chain of dependent calls the \emph{span} $S$. For \APT{} a chain consists of at most $R+1$ levels, each with one \textsc{Generate} call, at most $L$ rounds of \DDA{} and one minimisation, followed by one projection; so $S$ is bounded independently of the number of candidates. With $p$ concurrent reasoner calls, greedy scheduling of independent tasks finishes in time at most $W/p+S$ (Graham's bound). The bound is attained only if the reasoner really runs calls on different copies concurrently (B8); a reasoner that serialises internally gains little. The work $W$ itself is unchanged, which is why the caches below matter in addition.

\paragraph{Parallelism and incremental reasoners.} The two abilities interact. Without incremental updates each check pays the full cost of loading the knowledge base and each process repeats the same work. With incremental updates (B6) a process can keep its knowledge base and apply only the difference between consecutive checks, which for \DDA{} and for the rule bodies of one instant is a single assertion or a few; $W$ then shrinks, and $W/p+S$ is the better the more related checks a task hands to the same process. The tested incremental modes were not reliable for removals, so we could not realise this gain; Section~\ref{sec:eval-par} reports the measurements.

\subsection{Caches}
\label{sec:cache}

A cache stores the result of a call under a key that determines the result. Table~\ref{tab:cache} lists the caches of the implementation. Reasoners that keep state between calls can add further caches, for example of branched reasoner states; they are not needed for the results of this paper.

\begin{table}[t]
\centering\small
\caption{Caches of the implementation.}
\label{tab:cache}
\begin{tabularx}{\textwidth}{@{}P{2.7cm}P{3.5cm}P{3.0cm}X@{}}
\toprule
\textbf{Cache} & \textbf{Key} & \textbf{Value} & \textbf{What a hit saves}\\ \midrule
Suggestions & the assertion $s$ & $\Sugg(s)$ & the checks over $\Omega$ for $s$; the same assertion arises in many hypotheses and at many levels\\
Consistency & the set of assertions & consistent or not & one reasoner call; \emph{single-flight}: if the same key is requested concurrently, one task computes and the others wait for its result\\
\bottomrule
\end{tabularx}
\end{table}

\begin{proposition}[Caches are transparent]\label{prop:cache}
If the key of each cache determines the cached value, then the results of all algorithms are the same with and without the caches. For the caches of Table~\ref{tab:cache} the keys determine the values as follows: $\Sugg(s)$ depends on $\KB$, $\Omega$ and $s$ only, and consistency of a set of assertions depends on $\TB$ and the set only. A cache of reasoner \emph{states} must be keyed by the \emph{explicit} ABox, not merely by the set of entailed assertions, because deletions act on explicit assertions (Proposition~\ref{prop:scope}(ii)).
\end{proposition}
\begin{proof}
A cache returns a value that was computed earlier for the same key; if the key determines the value, it is the value a fresh computation would return. Single-flight changes only who computes the value.
\end{proof}

\paragraph{What caching saves, with an example.} Theorem~\ref{thm:dda-term} bounds the number of \textsc{Suggest} calls by $\omega$ \emph{distinct} assertions; this is exactly the effect of the suggestion cache, without which the same assertion would be recomputed for every hypothesis that contains it. In Example~\ref{ex:lab} the observation $\sn{Qualified}(\sn{amir})$ is explained twice, once for each of the two orders of the commuting actions, and the second call is answered from the cache. We also counted the reasoner requests of the implementation (Section~\ref{sec:impl}), which memoises consistency checks by the set of assertions: for the lab example with $R=3$ it received $173$ requests and forwarded $117$ distinct ones to HermiT ($32\%$ saved); for the family benchmark with two observations and $L=3$ the figures are $26200$ and $4133$ ($84\%$ saved); for the chain benchmark of Proposition~\ref{prop:depth} with $n=3$ and two additional individuals they are $333$ and $245$ ($26\%$ saved). These counts are properties of the implementation and of these inputs, not a general claim about speed-up.

\section{Analytical Comparison}
\label{sec:comparison}

We compare the two engines with established alternatives by what each approach guarantees and by the number of operations it needs, not by run times. We do not claim that \DDA{} or \APT{} is faster than a particular system. The one abduction solver that we could run on our problems, CATS, is compared by counts of consistency checks in Section~\ref{sec:cmp-systems}, which also states why the other systems cannot be compared. We state what each approach assumes about the reasoner, what its cost depends on, and where the approaches are incomparable. The formal definitions and proofs about hitting-set abduction and classical planning that support this section are in Appendix~\ref{app:formal}.

\subsection{A fair comparison: the same reasoner, the same abilities}
\label{sec:price}

\APT{} uses its abducer only through one interface: given a set of assumptions, return explanations that satisfy (E), (C) and (R). \DDA{} is one implementation of this interface, and a hitting-set abducer is another. Because validity of a plan is decided by projection, the soundness of \APT{} (Theorem~\ref{thm:apt-sound}) holds for \emph{every} abducer. So we may ask what \APT{} costs with one abducer or the other.

The comparison is fair only if both abducers may use the same reasoner abilities. The abilities of Section~\ref{sec:backend} that matter are B5 (the new consequences of an update) and B7 (the derivation record). Hitting-set abduction needs neither: it only asks the reasoner whether a candidate set is consistent and entails the observation. \DDA{} profits from both. We therefore compare three variants:
\begin{itemize}[leftmargin=*,itemsep=1pt,topsep=2pt]
\item $\mathrm{MHS}_k$, our own implementation of the hitting-set abducer \cite{reiter1987theory,greiner1989correction} on B1--B3 only: it reads no models, so it tests candidate sets of abducibles without guidance. It is the specification of the hitting-set problem, which the algorithm of AAA \cite{pukancova2020} computes with models of the reasoner (Section~\ref{sec:cmp-systems} and Appendix~\ref{app:formal}). Given a finite set $\mathit{Abd}$ of abducible assertions and a size bound $k$, it returns the minimal explanations with at most $k$ atoms from $\mathit{Abd}$.
\item \DDA{}, with B5 and B7, as in Sections~\ref{sec:dda} and~\ref{sec:fastmin}.
\item $\DDA'$, the relaxed variant that needs only B1--B3. It obtains the new consequences of $\neg\alpha$ by testing each assertion of $\Omega$ for entailment before and after the update, and it checks minimality with \textsc{Minimise}. $\DDA'$ returns exactly the same explanations as \DDA{} (the two compute the same sets of suggestions, Definition~\ref{def:sugg}, and the exact mode of Section~\ref{sec:fastmin}).
\end{itemize}
Accordingly, $\APT'$ denotes \APT{} with $\DDA'$. The fair comparison is $\APT'$ against \APT{} with $\mathrm{MHS}_k$: both run on any reasoner that has B1--B3, such as HermiT or Pellet behind the OWL API, or Konclude \cite{steigmiller2014konclude}. The full \DDA{} then shows what B5 and B7 additionally buy.

\paragraph{Counting operations.} We count \emph{updates} $u$ (an $\op{Add}$ or $\op{Remove}$ on a kept state), \emph{queries} $q$ (a consistency or entailment check), \emph{reads of new consequences} $r$ (B5) and \emph{reads of the derivation record} $e$ (B7). A reasoner turns these counts into time with its own per-operation cost; for an incremental reasoner a query after a small update is cheap, for a non-incremental one it is a full run. Let $\omega=|\Omega|$, let $s\le\omega$ be the number of distinct assertions for which suggestions are computed (the cache of Section~\ref{sec:cache} makes it the number of \emph{distinct} ones), $N_H$ the number of reached hypotheses, $h$ the largest hypothesis and $n=|O'|$. Table~\ref{tab:ops} lists the operations of one abduction call.

\begin{table}[t]
\centering\small
\caption{Operations of one abduction call, by abducer and by the abilities it uses. Grounding adds at most $|\Ind|\,q$ per existential atom to the \DDA{} rows.}
\label{tab:ops}
\begin{tabularx}{\textwidth}{@{}P{2.2cm}P{2.6cm}P{3.2cm}X@{}}
\toprule
\textbf{Abducer} & \textbf{Abilities} & \textbf{Suggestions} & \textbf{Minimality}\\ \midrule
\DDA{} (fast) & B1--B3, B5, B7 & $s\,(u+r)$ & $N_H(n{+}1)\,q+N_H\,h\,e$\\
\DDA{} (exact) & B1--B3, B5 & $s\,(u+r)$ & $\le N_H\,2^{h}(n{+}1)\,q$\\
$\DDA'$ & B1--B3 & $(s{+}1)\,\omega\,q+s\,u$ & $\le N_H\,2^{h}(n{+}1)\,q$\\
$\mathrm{MHS}_k$ & B1--B3 & --- & $\le(|\mathit{Abd}|{+}1)^{k}\,(u+(n{+}1)\,q)$\\
\bottomrule
\end{tabularx}
\end{table}

The table says what each abducer spends its operations on. $\mathrm{MHS}_k$ examines every candidate set up to size $k$ over a vocabulary \emph{declared in advance}; its work depends on $|\mathit{Abd}|$ and $k$ and does not use what the reasoner derives. \DDA{} forms its candidates from the consequences that the update $\neg\alpha$ actually produced, so its work depends on $\omega$, the number of observations, and the number of rounds $L$ (Theorem~\ref{thm:dda-term}); an explanation with many atoms is reached in $L$ rounds, whereas enumeration must go through all $\binom{|\mathit{Abd}|}{h}$ sets of that size.

\paragraph{Operation counts for \APT{}.} By Proposition~\ref{prop:rec}, \APT{} calls the abducer at most $\sum_{j<R}g(bg)^j$ times, each call is followed by at most $b$ calls of \textsc{Generate}, and every returned candidate is classified by one projection. Hence
\[
\#\text{ops}(\APT[X])\ \le\ \sum_{j<R}g(bg)^{j}\,\bigl(\mathrm{ops}_X+b\,\mathrm{ops}_{\textsc{Gen}}\bigr)\ +\ (R{+}1)\,g\,(bg)^{R}\,\mathrm{ops}_{\textsc{Proj}},
\]
where $\mathrm{ops}_X$ is the row of Table~\ref{tab:ops} for the abducer $X$ and $\mathrm{ops}_{\textsc{Proj}}\le n_c(|\RS|\,\iota^{v}a+4)+1$ for a candidate with $n_c$ actions (Theorem~\ref{thm:transitions}). The abducer enters in two ways: through $\mathrm{ops}_X$, and through the branching $b$, the number of explanations it returns, which appears with the exponent $R$. $\mathrm{MHS}_k$ returns all minimal atomic sets over the declared vocabulary, including many that no rule can achieve; \DDA{} returns explanations that follow the ontology's own consequences. We expect a smaller $b$ for \DDA{} but make no general claim.

\paragraph{Test: the lab example.} To make the comparison concrete, we ran the implementation on the first abduction problem of Example~\ref{ex:lab}, $O=\{\sn{Cleared}(\sn{amir})\}$, adding $m$ further individuals (each \sn{Registered}) to the knowledge base. The implementation accesses HermiT only through B1--B3, so its \DDA{} is $\DDA'$ and the comparison is the fair one. $\mathrm{MHS}_2$ uses all atomic class assertions as abducibles and needs $k=2$ to find $\{\sn{Trained}(\sn{amir}),\sn{HasPass}(\sn{amir})\}$; $\DDA'$ uses $L=1$. Table~\ref{tab:mhs-dda} gives the numbers of distinct consistency checks sent to the reasoner.

\begin{table}[t]
\centering\small
\caption{Distinct consistency checks sent to HermiT for $O=\{\sn{Cleared}(\sn{amir})\}$ when $m$ further individuals are added. Counts of the implementation (no incremental updates), not run times.}
\label{tab:mhs-dda}
\begin{tabular}{@{}rrrr@{}}
\toprule
$m$ & $|\mathit{Abd}|$ & $\mathrm{MHS}_2$ & $\DDA'$\\ \midrule
0 & 9 & 60 & 25\\
1 & 18 & 244 & 46\\
2 & 27 & 556 & 67\\
4 & 45 & 1564 & 109\\
\bottomrule
\end{tabular}
\end{table}

The counts of $\mathrm{MHS}_2$ grow quadratically with the number of abducibles and those of $\DDA'$ linearly with the number of individuals, as Table~\ref{tab:ops} predicts. For the second abduction problem of the example, $O=\{\sn{Qualified}(\sn{amir})\}$ with one individual, the picture is different: $\mathrm{MHS}_1$ needs $18$ checks and $\DDA'$ needs $24$. So neither dominates. Against a plain hitting-set enumeration that uses only B1--B3, \DDA{} needs fewer checks when the declared vocabulary is large compared with the number of suggestions, which is also the situation in which a vocabulary is hardest to declare in advance. This statement is about enumeration without models; Section~\ref{sec:cmp-systems} repeats the comparison with an existing solver that also reuses models of the reasoner, and the statement does not carry over to it.

\paragraph{Incomparable outputs.} The two abducers do not return the same explanations. If $\mathit{Abd}\subseteq\Omega$, every single-atom explanation found by $\mathrm{MHS}_1$ is found by \DDA{} in one round. For larger explanations neither method contains the other: the explanation $\{R(a,b),B(b)\}$ of Proposition~\ref{prop:limits}(a) is found by $\mathrm{MHS}_2$ and missed by \DDA, whereas $\{(\exists R.B)(a)\}$ of Proposition~\ref{prop:limits}(b) is found by \DDA{} and cannot be expressed over atomic abducibles. Table~\ref{tab:cmp-abd} summarises the differences.

\begin{table}[t]
\centering\small
\caption{Hitting-set abduction versus \DDA: guarantees and assumptions.}
\label{tab:cmp-abd}
\begin{tabularx}{\textwidth}{@{}P{3.0cm}XX@{}}
\toprule
\textbf{Criterion} & \textbf{$\mathrm{MHS}_k$} & \textbf{\DDA}\\ \midrule
Hypothesis language & atoms from a declared finite set & assertions over $\Omega$, including complex concepts and existentials\\
Vocabulary declared in advance & yes & no\\
Candidates & subsets of the declared set & negated new consequences of the negated observation\\
Soundness and minimality & yes & yes (Thm.~\ref{thm:dda-sound})\\
Completeness & all minimal sets within the declared set and size $k$ & all minimal sets in the reachable space (Thm.~\ref{thm:dda-complete})\\
Cost parameters & $|\mathit{Abd}|$, $k$ & $\omega$, $|\Ind|$, $n$, $L$\\
Reasoner assumed & any (B1--B3) & B1--B3; faster with B5, B7\\
Purpose & explaining one observation in a fixed KB & feeding a planner in a state that changes step by step\\
\bottomrule
\end{tabularx}
\end{table}

\paragraph{Which reasoners provide which abilities.} Table~\ref{tab:abilities} states what each tested reasoner provides. All of them provide the required abilities (B1)--(B3) behind the OWL API \cite{horridge2011owlapi} or, for Konclude, through OWLlink \cite{steigmiller2014konclude}, and (B8) by separate instances; with them the tools run as $\APT'$. None provides (B5) or (B7) natively; they are obtained from entailment checks and black-box supports (Section~\ref{sec:backend}). (B6), incremental updates, is the ability that most distinguishes reasoners in practice. It is native in reasoners for fragments, for example ELK for $\mathcal{EL}$ \cite{kazakov2013incremental}, and has been studied for expressive description logics \cite{grau2007history,cuenca2010incremental} and for ABox streams \cite{proost2026incremental}. Among reasoners for $\SROIQ$, Pellet implements incremental consistency checking \cite{halaschek2006dynamic} and Konclude processes \texttt{Tell} and \texttt{Retract} as new revisions \cite{steigmiller2014konclude,steigmiller2016thesis}; Section~\ref{sec:eval-incr} tests both. If a reasoner provides an ability natively, the corresponding rows of Table~\ref{tab:ops} apply and no theorem of this paper changes.

\subsection{Comparison with other abduction systems}
\label{sec:cmp-systems}

Section~\ref{sec:price} compares \DDA{} with a hitting-set abducer that we implemented ourselves. This subsection asks which existing systems could be compared with \DDA{} and \APT{}, runs the one that can be run, and states why the others cannot. We call a system \emph{comparable} if it accepts the ontologies of our benchmarks, takes an observation that is an ABox assertion, returns hypotheses that are ABox assertions, and can be obtained and run. Table~\ref{tab:systems} lists the systems that we found for ABox abduction in description logics and for abduction used for planning.

\begin{table}[t]
\centering\small
\caption{Abduction systems considered. ``Declared'' means that the hypothesis vocabulary is given in advance. Of the existing systems only CATS is comparable; Section~\ref{sec:cmp-systems} gives the reasons for the others.}
\label{tab:systems}
\newcolumntype{P}[1]{>{\raggedright\arraybackslash}p{#1}}
\begin{tabular}{@{}P{3.1cm}P{2.7cm}P{3.5cm}P{2.3cm}P{1.6cm}@{}}
\toprule
System & Logic & Hypotheses & Availability & Here\\ \midrule
$\mathrm{MHS}_k$ (ours) & $\SROIQ$ through B1--B3 & atoms over declared vocabulary & implemented & run (Table~\ref{tab:mhs-dda})\\
CATS \cite{kloc2025cats} & $\SROIQ$ (OWL API, JFact) & atomic class and role assertions, negations, declared & open source, jar & run (Table~\ref{tab:cats})\\
AAA \cite{pukancova2020} & $\SROIQ$ (Pellet) & atomic assertions, declared & download page named in the paper, not retrievable by us & algorithm compared; not run\\
LETHE-Abduction \cite{koopmann2020signature,koopmann2021signature} & $\mathcal{ALC}$ family & description-logic axioms, complex concepts, fresh individuals, given signature & source, Maven build & not run\\
Klarman et al.\ \cite{klarman2011abox} & $\mathcal{ALC}$ & assertions with complex concepts & no distribution found & not run\\
Del-Pinto, Schmidt \cite{delpinto2019} & $\mathcal{ALC}$ & ABox axioms by forgetting & no distribution found & not run\\
Du et al.\ \cite{du2011practical} & $\mathcal{SHOIQ}$ through KAON2 and B-Prolog & ABox assertions & no distribution found & not run\\
Event calculus planner \cite{shanahan2000abductive}, abductive logic programs \cite{kakas1992abductive} & logic programs, closed world & actions and facts declared abducible & logic-programming systems & different formalism\\
Neural generators \cite{bai2024advancing} & none (embeddings, language models) & ranked complex hypotheses & research code & different task\\
\bottomrule
\end{tabular}
\end{table}

\paragraph{The algorithm of AAA.} AAA \cite{pukancova2020} is built on Reiter's hitting-set algorithm \cite{reiter1987theory} and calls Pellet as a black box for consistency checks and for the retrieval of models; it supports $\SROIQ$. The paper names a download page, which our tools could not retrieve, so the program was not run. We describe the algorithm after the paper and after Mr\'ozek et al.\ \cite{mrozek2018}, who reimplement it on the OWL API. An explanation $E$ of an observation $o$ must be consistent with $\KB$, relevant ($E$ alone does not entail $o$) and explanatory ($\KB\not\models o$). This relevance is not our criterion (R) of Section~\ref{sec:dda}; both exclude explaining an observation by itself, and they are not identical. Like \DDA{}, AAA starts from the negation of $o$. By Reiter's theorem $E$ explains $o$ iff it hits every model of $\KB\cup\{\neg o\}$, and AAA finds the hitting sets with an HS-tree built breadth first, so that shorter explanations come first and the search can be bounded by a depth, up to which it is complete. A candidate $c$ is a set of declared assertions, and the first candidate is $\emptyset$. For each candidate the reasoner checks whether $\KB\cup\{\neg o\}\cup c$ is consistent. If it is not, $\KB\cup c\models o$ and $c$ is an explanation if it passes the three criteria. If it is, AAA reads the model that the reasoner has built, negates the declared assertions in it, and makes each of them a child candidate $c\cup\{a\}$. Reiter's optimisations are used: stored models are reused, and candidates that contain a known explanation or are inconsistent with $\KB$ are pruned. Explanations are sets of atomic and negated atomic concept and role assertions over declared individuals, concept names and role names, the \emph{abducibles}; reflexive role assertions are excluded unless enabled. An observation may be any set of assertions: AAA either explains each assertion separately and combines the results (splitting), or reduces the set to one assertion that has the same explanations (reduction).

\DDA{} differs in what it reads from the reasoner. AAA reads \emph{a model}, which holds assertions that are true in that model only, and each child needs a further consistency check. \DDA{} reads the \emph{new consequences} of $\neg o$, which hold in every model, so each suggestion is already an explanation of $o$ (Lemma~\ref{lem:sugg}) and needs no check, and the reasoner needs no access to models: B1--B3 suffice. Both find every single-assertion explanation in the first step (Lemma~\ref{lem:one-step} and Proposition~\ref{prop:aaa}). The price of \DDA{} is that it grows a hypothesis by replacing its parts with suggestions and does not hit every model, so it is incomplete for explanations of several atoms (Proposition~\ref{prop:limits}), whereas the tree of AAA is complete for the declared atoms up to its depth bound (Proposition~\ref{prop:aaa}). Appendix~\ref{app:formal} states this precisely. CATS belongs to the same line of solvers: its \texttt{mhs} setting implements the same scheme of a hitting-set tree over models, on JFact, and its MXP settings add the merging of explanations \cite{homola2023}. We use the \texttt{mhs} setting as the runnable counterpart of the algorithm of AAA, and we did not check that the two visit the same candidates.

\paragraph{What was run.} CATS \cite{kloc2025cats,catsrepo} is a Java solver for ABox abduction in $\SROIQ$ on the OWL API and JFact. It implements Reiter's hitting-set search in four variants (MHS, HST, RCT, HSDAG), each with and without the MXP extension of \cite{homola2023}. We ran its jar (commit 2b99d06) with default settings and a limit of $30$ seconds on the problems of Section~\ref{sec:price} and on the Door and Family examples, with the same knowledge base and observation as for \DDA{}; the abducibles are all atomic classes, asserted of all individuals, and negated assertions are switched off. Table~\ref{tab:cats} gives its explanations and the consistency checks that it reports, next to the distinct checks of $\DDA'$ (HermiT). The counts differ in kind: CATS counts the checks of its search, including models that it reuses, $\DDA'$ counts distinct checks sent to the reasoner. Run times are not compared.

\begin{table}[t]
\centering\small
\caption{CATS and $\DDA'$ on the same abduction problems: consistency checks and number of explanations. $\dagger$: stopped at the time limit of $30$ seconds after more than $7000$ checks, without finishing. Lab-1 is $O=\{\sn{Cleared}(\sn{amir})\}$ with $m$ further individuals, Lab-2 is $O=\{\sn{Qualified}(\sn{amir})\}$, Door is $O=\{\sn{SecureDoor}(\sn{dor})\}$ with a mechanical lock, Family is $O=\{\sn{Mother}(\sn{jane})\}$ with $\sn{Man}(\sn{tarzan})$.}
\label{tab:cats}
\begin{tabular}{@{}lrrrrr@{}}
\toprule
 & \multicolumn{3}{c}{consistency checks} & \multicolumn{2}{c}{explanations}\\
\cmidrule(lr){2-4}\cmidrule(l){5-6}
Problem & CATS MHS-MXP & CATS MHS & $\DDA'$ & CATS & $\DDA'$\\ \midrule
Lab-1, $m=0$ & 26 & 27 & 25 & 1 & 1\\
Lab-1, $m=1$ & 35 & 435 & 46 & 1 & 1\\
Lab-1, $m=2$ & 36 & $\dagger$ & 67 & 1 & 1\\
Lab-1, $m=4$ & 43 & $\dagger$ & 109 & 1 & 1\\
Lab-2 & 19 & 14 & 24 & 2 & 2\\
Door & 26 & 17 & 37 & 1 & 1\\
Family (a) & 91 & 1366 & 350 & 2 & 6\\
Family (b) & 91 & 1366 & 349 & 2 & 5\\
\bottomrule
\end{tabular}
\end{table}

\paragraph{Results.} On the six problems whose explanations are sets of class assertions (the four Lab-1 variants, Lab-2 and Door), CATS and \DDA{} return the same explanations, as two sound methods with subset-minimal output should. On Family, \DDA{} returns more explanations, because its hypotheses may be role assertions and an existential restriction, for example $\{\sn{Woman}(\sn{jane}),\sn{hasChild}(\sn{jane},\sn{tarzan})\}$. With role assertions allowed, the CATS jar that we ran ended with a \texttt{ClassCastException} in its model extraction on the Door and Family problems, so we report no such runs; this is an error of that build on these inputs, not a statement about the method. On three problems we ran all eight algorithms of CATS, and every run that finished returned the explanations of MHS-MXP.

With MXP, CATS needs fewer checks than $\DDA'$ on all problems except the smallest ($26$ to $43$ on Lab-1 against $25$ to $109$). Plain MHS, the scheme of AAA, needs far more on Lab-1 with $m\ge1$ and on Family ($1366$ against $350$), and fewer on Lab-2 and Door, where the explanation is small; these columns are our measurement of that scheme on JFact and not of the program AAA on Pellet. We read this as follows. On explanations over declared atoms \DDA{} agrees with a state-of-the-art hitting-set solver. \DDA{} needs fewer checks than enumeration without models and more than CATS with MXP, which reuses models, so the advantage over plain enumeration reported in Section~\ref{sec:price} does not hold against a solver that exploits models. And \DDA{} returns hypotheses (roles, existential restrictions) that a solver over atoms does not search, at the price of more checks.

\paragraph{Limits of the comparison with CATS.}\label{sec:cats-limits} Table~\ref{tab:cats} compares $\DDA'$, which uses only B1--B3, with CATS, which uses the completion graph of JFact; its loader constructs JFact and casts it to \texttt{OWLKnowledgeExplorerReasoner}, an ability outside B1--B3. None of the reasoners that we ran offers reliable incremental processing of added and removed assertions (B6), the new consequences of an update (B5) or a derivation record (B7), which \DDA{} is designed to use. The column $\DDA'$ is therefore the cost of the algorithm under the weakest contract and not of \DDA{} with its intended abilities.

The abilities beyond B1--B3 act on different quantities. Incremental updates (B6) and concurrency (B8) lower the cost or the elapsed time of a check and leave the number of checks unchanged, so Table~\ref{tab:cats} would not change. B5 and B7 lower the number of checks, and model access, as used by CATS, lowers it by pruning. Table~\ref{tab:phases} shows where the checks of $\DDA'$ go: $79$ to $97\,\%$ are suggestion tests, two entailment checks for each element of the declared space $\Omega$ (Section~\ref{sec:suggest}), which B5 replaces by one read of the consequences of an update, and the subset search of the minimisation, which B7 replaces by reading supports. What remains are the checks of the initialisation, of the witnesses and of the minimisation of the base, plus one update and one read per suggestion. For Lab-1 this is $3$ checks and one update with one read for every $m$, against $26$ to $43$ checks of CATS; on Family~(a) it is $64$ checks and $4$ updates and reads, against $91$. The three problems with role assertions, which CATS could not run, need $150$, $1281$ and $997$ checks with B1--B3 only, and $17$, $585$ and $407$ with B5 and B7.

\begin{table}[t]
\centering\footnotesize
\caption{Where the distinct checks of $\DDA'$ (HermiT) are spent, and the checks that remain with B5 and B7. The last column is derived from the trace of the algorithm by dropping the suggestion tests and the subset search and adding one update and one read per suggestion; it is not measured on a reasoner, since none of the reasoners that we could run provides B5 or B7. Updates and reads are different operations from checks and are listed separately. The last column on the right repeats the checks of CATS (MHS-MXP) from Table~\ref{tab:cats}; the units of the two systems differ.}
\label{tab:phases}
\setlength{\tabcolsep}{2.5pt}
\begin{tabular}{@{}lrrrrrrrr@{}}
\toprule
 & \multicolumn{5}{c}{phases of $\DDA'$} & & with B5, B7 & \\
\cmidrule(lr){2-6}\cmidrule(lr){8-8}
Problem & init & base & subsets & suggest & wit. & $\DDA'$ & checks (+upd.) & CATS\\ \midrule
Lab-1, $m=0$ & 1 & 2 & 0 & 22 & 0 & 25 & 3 (+1) & 26\\
Lab-1, $m=1$ & 1 & 2 & 0 & 43 & 0 & 46 & 3 (+1) & 35\\
Lab-1, $m=2$ & 1 & 2 & 0 & 64 & 0 & 67 & 3 (+1) & 36\\
Lab-1, $m=4$ & 1 & 2 & 0 & 106 & 0 & 109 & 3 (+1) & 43\\
Lab-2 & 1 & 0 & 0 & 23 & 0 & 24 & 1 (+1) & 19\\
Door & 1 & 2 & 0 & 34 & 0 & 37 & 3 (+1) & 26\\
Family (a) & 1 & 54 & 8 & 278 & 9 & 350 & 64 (+4) & 91\\
Family (b) & 1 & 51 & 6 & 282 & 9 & 349 & 61 (+4) & 91\\
\bottomrule
\end{tabular}
\end{table}

The derived column is not measured, since none of the reasoners that we ran provides B5 or B7, and it does not show that \DDA{} with these abilities would beat CATS: a read of consequences may cost as much as a check. It shows that the gap in Table~\ref{tab:cats} lies in the part of the algorithm that B5 and B7 remove, and not in the search over hypotheses. Further limits: the problems are small ($9$ to $45$ abducibles, one observation); the reasoners differ (JFact, HermiT); the hypothesis languages differ, since the CATS build that we ran returns class assertions only; and we ran one build of one solver with default settings. Whether stronger reasoner abilities change the picture on large ontologies needs a reasoner that provides them (Section~\ref{sec:discussion}).

\paragraph{Why the other systems cannot be compared.} AAA \cite{pukancova2020}, a solver based on Reiter's algorithm and Pellet, names a download page, which our tools could not retrieve, so it was not run; its algorithm is compared above and in Appendix~\ref{app:formal}, and its scheme is run in the \texttt{mhs} setting of CATS; Klarman et al.\ \cite{klarman2011abox}, Du et al.\ \cite{du2011practical} and Del-Pinto and Schmidt \cite{delpinto2019} have no distribution that we could find. A comparison needs the program, since the papers report their own benchmarks. LETHE-Abduction \cite{koopmann2020signature,koopmann2021signature,koopmann2020lethe} is distributed as source with a Maven build, which needs dependencies that our environment cannot reach, so it was not built; it also works in the $\mathcal{ALC}$ family, needs a given signature and returns axioms with complex concepts, which would have to be normalised to atoms. The event calculus planner \cite{shanahan2000abductive} and abductive logic programs \cite{kakas1992abductive} work on logic programs with a closed world and negation as failure, and a description-logic ontology with open-world semantics is not such a program. The neural generators \cite{bai2024advancing} rank hypotheses for graphs without a terminology and do not check entailment or consistency, so their output cannot be tested for soundness as that of \DDA{} can.

Abduction-only systems, CATS included, return explanations and no plans, so they cannot be compared with \APT{} or with adversarial synthesis; they could serve as the abducer of \APT{} in place of \DDA{}, which we did not build. We found no system that combines explanation of missing facts over a description-logic ontology with plan search and exact verification by projection, so \APT{} is compared with planning formalisms (Section~\ref{sec:cmp-plan}).

\subsection{\APT{} and classical planning}
\label{sec:cmp-plan}

\paragraph{What classical planning is.} A \emph{classical planning problem with axioms} (CP) has finitely many objects and predicates, \emph{derived} predicates defined by stratified Horn-like axioms (PDDL axioms \cite{thiebaux2005defense}), an initial state, a goal and actions. A state is a finite set of ground atoms over the basic predicates; derived atoms are computed from it by evaluating the axioms. An action has a precondition, an add list and a delete list. The \emph{closed-world assumption} holds: every atom that is not true is false. Plan existence for a given horizon is in \textsc{NP} \cite{bylander1994computational}.

\paragraph{How the pieces correspond.} A rule of a TDKG corresponds to an action, its body to the precondition, $I_r$ and $D_r$ to the add and delete lists, a stored assertion to a basic atom, an entailed assertion to a derived atom, and the terminology $\TB$ to the axioms. The correspondence is exact only in a restricted situation, and the rest of the comparison shows where it breaks.

\paragraph{Where \APT{} and a classical planner coincide.} If $\TB$ uses only subclass axioms between atomic classes, intersections on the left, domain and range, and role inclusions, the entailed atomic assertions are the least fixpoint of linear-size Datalog rules, which are PDDL axioms. If, in addition, the rules have atomic positive bodies, one rule per action, and create no individuals, and the ontology does not derive facts from events (fragment $\mathcal F$ of Definition~\ref{def:fragment}), a CP with these axioms and the same rules has a plan of at most $d$ actions if and only if the TDKG has one with $\Gamma=\emptyset$ (Propositions~\ref{prop:compile} and~\ref{prop:coincide} in Appendix~\ref{app:formal}). On this fragment a classical planner is an adequate and typically faster tool, with mature heuristics and optimality guarantees that the generator of \APT{} does not offer (Section~\ref{sec:quality}). The benchmarks \sn{DerivedGate} and \sn{LabAccess} lie in it. We use them to explain the algorithms, not to claim an advantage over classical planners.

\paragraph{No compact translation outside the fragment.} Outside the fragment, with disjunction, qualified existentials, number restrictions or nominals, no translation to CP of at most exponential time preserves plan existence in general: plan existence for bounded TDKG planning is $2$\textsc{NExpTime}-hard (Theorem~\ref{thm:apt-hard}), whereas plan existence for CP is in \textsc{NP}, and $2\textsc{NExpTime}\not\subseteq\textsc{NExpTime}$ by the time hierarchy theorem. The formal statement and proof are Theorem~\ref{thm:no-compile} in Appendix~\ref{app:formal}. The consequence is that part of the cost lives inherently in the description-logic reasoner, and that a hand-written encoding of an expressive ontology is impossible in general.

\paragraph{Three features without a classical counterpart.}
\begin{enumerate}[leftmargin=*,itemsep=2pt,topsep=2pt]
\item \emph{The open world and assumed facts.} In CP a missing fact is false, so a plan that needs it does not exist. \APT{} may carry a residual set $\Gamma$ of assumptions, reports it and explains it from the ontology. In the secure-door example (Example~\ref{ex:door}) a CP in which $\sn{SecureDoor}(\sn{dor})$ is derived from $\sn{HingedStructure}$ and $\sn{WoodenStructure}$ has no plan, since neither is stated; \APT{} returns the plan with $\Gamma=\{\sn{HingedStructure}(\sn{dor}),\sn{WoodenStructure}(\sn{dor}),\sn{isInstalledOn}(\sn{FrontDoorLock},\sn{dor})\}$. In CP the modeller would have to add an optional ``assume'' action for every fact that may be unknown.
\item \emph{Creation of entities.} Rules create individuals (the letter and the account in the bank benchmark), and the planner introduces placeholder entities when a participant has no known value and no individual qualifies. A CP has a fixed object set, so one must guess in advance how many objects of which type the plan needs.
\item \emph{A reasoner in the loop.} A CP evaluates a state in polynomial time; \APT{} evaluates a state by a reasoner whose worst case is $2$\textsc{NExpTime}. This is the price of expressivity. With an incremental reasoner the cost per search step is that of a small update, and on the compilable fragment a materialising reasoner is itself polynomial, so the two approaches meet there.
\end{enumerate}

Table~\ref{tab:cmp-plan} summarises the comparison; the paragraphs below explain the rows that need an example.

\paragraph{The taxes benchmark.} In Example~\ref{ex:tax} the contradiction to be reached is $\{\sn{TaxPayer}(x),\neg\sn{TaxPayer}(x)\}$. A classical planner lacks three things it needs. (i) It cannot search for a contradiction: in a closed-world state $\neg\sn{TaxPayer}(x)$ holds as long as $\sn{TaxPayer}(x)$ is not asserted, so the goal is unsatisfiable, unless a human writes a derived predicate such as $\mathit{conflict}\leftarrow\sn{TaxExempt}(x),\sn{TaxPayer}(x)$ after reading the disjointness axiom. (ii) Even then, the relation ``$\sn{TaxExempt}(x)$ is what concludes $\neg\sn{TaxPayer}(x)$'' comes from reasoning with the ontology: \DDA{} obtains it from the consequences of $\sn{TaxPayer}(\sn{TraderJoe})$ (Table~\ref{tab:tax-trace}). (iii) The explanation $\sn{TaxExempt}(\sn{TraderJoe})$ is a new goal, achieved by a plan of its own (\ImportWheat), which must be scheduled \emph{before} the original action. A classical planner given a hand-written conflict predicate and a fixed agent solves the easier variant with the axiom-specific pattern $\{\sn{TaxExempt}(x),\sn{TaxPayer}(x)\}$; what it cannot do is find the story \emph{from the ontology}, which is the point of adversarial synthesis.

\begin{table}[t]
\centering\small
\caption{\APT{} versus classical planning with axioms (CP).}
\label{tab:cmp-plan}
\begin{tabularx}{\textwidth}{@{}P{3.1cm}XX@{}}
\toprule
\textbf{Criterion} & \textbf{CP} & \textbf{\APT}\\ \midrule
State & ground atoms, closed world & $\SROIQ$ ABox, open world; entailed facts via the reasoner\\
Ontology & compiled by hand into axioms or effects; compact only for restricted fragments & native; indirect effects by entailment (Thm.~\ref{thm:ramif})\\
New individuals & fixed object set & created by rules; placeholders introduced by the planner\\
Missing facts & false, or enumerated by the modeller & assumed, reported as $\Gamma$ and explained by abduction\\
Contradiction as goal & only through a hand-written derived predicate & any contradiction the reasoner detects; the chosen pattern is checked (Alg.~\ref{alg:reached})\\
Deletions that avoid a clash & part of the search & repair step (Alg.~\ref{alg:repair})\\
Worst case (bounded) & \textsc{NP} & $2$\textsc{NExpTime}-hard (Thm.~\ref{thm:apt-hard}), in $3$\textsc{ExpTime} (Thm.~\ref{thm:apt-upper})\\
Result & sound; complete and optimal planners exist & sound (Thm.~\ref{thm:apt-sound}); heuristic generator, incomplete, no optimality claim (Section~\ref{sec:quality})\\
Heuristics & mature ($h^{\mathrm{FF}}$, landmarks, \dots) & goal-count regression only\\
Parallelism & portfolio and parallel search & independent reasoner calls (Section~\ref{sec:parcache})\\
\bottomrule
\end{tabularx}
\end{table}

\subsection{Scalability}
\label{sec:scal}

For tools whose cost is almost entirely the cost of a reasoner, a run-time table on one machine mostly measures the reasoner. We therefore give the analytical dependence of the number of operations on each size parameter and call counts of the implementation on a scalable family; run times on several reasoners are in Section~\ref{sec:eval}. We do \emph{not} claim results on large real-world ontologies.

\paragraph{Analytical dependence.} The tools add no reasoning of their own, only calls. The number of calls depends on: the number of individuals $\iota$ (rule matching costs $\iota^{v}a$ per step, grounding $\iota$ per existential atom, and $|\Omega|$ grows with $\iota$); the plan length $d$ (the search space $(|\RS|N^{w})^{d}$ of Theorem~\ref{thm:apt-decid}, which the heuristic generator explores only partly); the recursion depth $R$ with the branching $g$ and $b$ (at most $(R{+}1)g(bg)^R$ candidates, Proposition~\ref{prop:rec}); the size $h$ of explanations (Theorem~\ref{thm:dda-term}); and the number $p$ of workers (Section~\ref{sec:parcache}). The size of the ontology enters only through the cost per call. In practice $R$ and $L$ are small (the benchmarks need $R\le2$), and caching removes repeated work.

\paragraph{Call counts on the scalable chain.} The scalable chain of Proposition~\ref{prop:depth} needs $n$ abductions and $n$ sub-plans for a plan of $n{+}1$ actions. We added $m$ further individuals, each \sn{Registered}, and counted the distinct consistency checks that the implementation (Python, HermiT as oracle, no incremental updates) sent to the reasoner for the whole run: $n$ calls of \DDA, $n{+}1$ calls of \textsc{Generate} and the classification of the candidates. The plan was found in every case; Table~\ref{tab:scal} gives the counts.

\begin{table}[t]
\centering\small
\caption{Distinct consistency checks of the implementation for the scalable chain with $n$ levels (so $n$ abductions) and $m$ additional individuals. Counts, not run times.}
\label{tab:scal}
\begin{tabular}{@{}rrrr@{}}
\toprule
$n$ & $m=0$ & $m=2$ & $m=4$\\ \midrule
1 & 26 & 56 & 86\\
2 & 55 & 133 & 211\\
3 & 95 & 245 & 395\\
4 & 146 & 392 & 638\\
\bottomrule
\end{tabular}
\end{table}

In this family the counts grow linearly in the number of individuals (about $15,39,75,123$ per added individual for $n=1,\dots,4$) and slowly in $n$; the recursion did not blow up because \DDA{} returned a single explanation at each level. With several explanations per level the factor $(bg)^{R}$ of Proposition~\ref{prop:rec} applies.

\section{The Benchmarks}
\label{sec:examples}

Six benchmark ontologies show which properties of the tools each case exercises. Family (Example~\ref{ex:family}), Secure door (Example~\ref{ex:door}), Taxes (Example~\ref{ex:tax}) and LabAccess (Example~\ref{ex:lab}) were discussed in the previous sections; BankAccount and DerivedGate, which exercise two features that the others do not, follow here. The suggestion sets, explanations and plans of all examples are computed by the implementation of Section~\ref{sec:impl} with HermiT as the reasoner; Appendix~\ref{app:traces} lists the plans.

\begin{example}[BankAccount: an action with two rules, and creation of individuals]\label{ex:bank}
The benchmark, adapted from Baader et al.~\cite{baader2005integrating}, has two actions: \sn{OpenAccount} has the two rules of Listing~\ref{lst:bank}, and its action \sn{GetLetter} creates a letter $\ell$ with $\sn{has}(h,\ell)$ for a person $h$ that is a \sn{Human}. $\sn{Amir}$ is an \sn{EligiblePerson}, hence a \sn{Human} \emph{only by entailment}, and has a proof of address.
\begin{itemize}
\item \emph{Goal: an account with a card,} $\{\sn{BankAccountWithCard}(?ac),\sn{has}(\sn{Amir},?ac)\}$. Only rule~2 of \sn{OpenAccount} inserts \sn{BankAccountWithCard}. Regression through it yields the subgoals $\sn{Letter}(\ell)$ and $\sn{has}(\sn{Amir},\ell)$, which \sn{GetLetter} achieves, creating $\ell$. Projection confirms the plan $[\sn{GetLetter}(\sn{Amir}),\sn{OpenAccount}(\sn{Amir})]$: step~1 creates a letter, step~2 fires rule~2 and creates an account, and the goal holds with $?ac$ bound to the created individual (the goal variable is first searched among existing individuals; none qualifies, so the created one is used). The reverse order is also proposed by the generator, but the projection rejects it, since \sn{OpenAccount} is not effective without a letter (V2).
\item \emph{Goal: an account without a card,} $\{\sn{BankAccountWithoutCard}(?ac),\sn{has}(\sn{Amir},?ac)\}$. Rule~1 inserts it, its preconditions hold, and $[\sn{OpenAccount}(\sn{Amir})]$ is a plan. Here the second rule matters in the other direction: a plan $[\sn{GetLetter}(\sn{Amir}),\sn{OpenAccount}(\sn{Amir})]$, which the generator may propose because it ignores the side condition \texttt{FILTER NOT EXISTS} of rule~1 (Q3), fires rule~2 instead and creates a card account; the goal fails, and the projection discards it.
\end{itemize}
The example therefore exercises three things at once: one action with two rules whose firing depends on the state, the creation of individuals by a rule, and a precondition ($\sn{Human}$) that holds only by entailment.
\end{example}

\begin{example}[DerivedGate: reasoning in a state yields a precondition of a rule]\label{ex:gate}
$\TB\ni\sn{BadgeHolder}\sqsubseteq\sn{AuthorizedPerson}$. The rule of \EnterZone{} needs $\sn{BadgeHolder}(p)$ \emph{and} $\sn{AuthorizedPerson}(p)$ and inserts $\sn{InSecureZone}(p)$; \IssueBadge{} requires $\sn{RegisteredPerson}(p)$ and inserts only $\sn{BadgeHolder}(p)$. Starting from $\{\sn{RegisteredPerson}(\sn{Amir})\}$ with goal $\sn{InSecureZone}(\sn{Amir})$, the plan $[\IssueBadge,\EnterZone]$ is valid. No rule inserts $\sn{AuthorizedPerson}$; it is obtained by reasoning in the state. The reasoning is used in both phases: \textsc{Generate} absorbs the precondition $\sn{AuthorizedPerson}(\sn{Amir})$ because it is entailed by the relaxed state that already contains the insert $\sn{BadgeHolder}(\sn{Amir})$ of \IssueBadge{}, and the projection finds, after step~1, that the body of \EnterZone{} matches under entailment (Q1). By Theorem~\ref{thm:ramif}(a) this is possible only because step~1 really added $\sn{BadgeHolder}(\sn{Amir})$. A system that evaluated rule bodies on the stored assertions alone would never fire \EnterZone, which is why (Q1) demands matching under entailment. The ontology lies in the fragment that a classical planner can compile (Section~\ref{sec:cmp-plan}); the benefit of \APT{} is that the axiom is read from the ontology instead of being written a second time.
\end{example}

\begin{table}[t]
\centering\small
\caption{What each benchmark exercises. $\iota$: named individuals in the problem; $|N_C|,|N_R|$: declared concept and role names; $|\mathit{Abd}|=\iota|N_C|+\iota^2|N_R|$ (all atomic assertions as abducibles). Counts are read from the ontologies; the bounds are upper bounds from Table~\ref{tab:ops} and Theorem~\ref{thm:dda-term}, not measurements; the numbers of checks for LabAccess, Bank and Gate are those of the implementation (Tables~\ref{tab:mhs-dda} and~\ref{tab:reasoners}).}
\label{tab:bench}
\begin{tabularx}{\textwidth}{@{}P{1.9cm}P{2.0cm}P{3.3cm}P{3.3cm}X@{}}
\toprule
\textbf{Benchmark} & $\iota/|N_C|/|N_R|$ & \textbf{Property exercised} & \textbf{Quantities} & \textbf{CP encoding}\\ \midrule
Family & $2/8/2$, $|\mathit{Abd}|{=}24$ & Thm.~\ref{thm:dda-complete}; Prop.~\ref{prop:combine}, \ref{prop:minx}; Prop.~\ref{prop:limits}(c) & $|\Omega|=68$; $\mathrm{MHS}_2$: $\le301$ calls & vocabulary must be declared\\
Bank & $2^{\ast}/13/4$ & action with two rules; creation of individuals; entailed precondition; side condition & $\le\iota^{v}a$ checks per step; 27 and 18 distinct checks for the two goals (Table~\ref{tab:reasoners}) & object pool; entailment compilable\\
Gate & $1/11/2$ & precondition by reasoning in a state; Thm.~\ref{thm:ramif}(a) & 12 distinct checks (Table~\ref{tab:reasoners}) & compilable; CP with axioms suffices\\
Secure door & $2/13/3$, $|\mathit{Abd}|{=}38$ & Thm.~\ref{thm:apt-sound}; \DDA{} on assumptions & $|\Omega|=32$ in the assumption problem; $\mathrm{MHS}_3$: $\le9178$ calls & missing facts must be enumerated\\
Lab & $1/9/0$, $|\mathit{Abd}|{=}9$ & two abductions and two sub-plans; Prop.~\ref{prop:rec}, \ref{prop:depth} & $|\Omega|=20$; \DDA{}: 25 and 24 checks, $\mathrm{MHS}_2$: 60 (Table~\ref{tab:mhs-dda}) & compilable; derived preconditions as axioms\\
Taxes & $1/7/2$ & Alg.~\ref{alg:classify}--\ref{alg:reached}; abduction of $\neg\sn{TaxPayer}$; Lemma~\ref{lem:culprit} & $2$ complement patterns: $1$ story, $1$ rejected & conflict predicate by hand; no abduction; no new entity\\
\bottomrule
\end{tabularx}
\\[2pt]\footnotesize $^{\ast}$ In the planning problem; the ontology declares $11$ individuals including the scheduled events and instants.
\end{table}

\noindent Table~\ref{tab:bench} summarises what each benchmark exercises. The bounds for $\mathrm{MHS}_k$ are worst-case bounds for a vocabulary-wide abducible set; a solver with a tighter set does less work. The table shows the parameters each method pays for, not run times.


\section{Implementation and Evaluation on Several Reasoners}
\label{sec:eval}

The contract of Section~\ref{sec:backend} makes four claims that can be tested. (1) The results of the tools do not depend on the reasoner, as long as it satisfies B1--B3. (2) The reasoners differ in cost, not in answers. (3) Incremental updates (B6) can be used by the tools. (4) Concurrent calls (B8) shorten the run time. This section describes the implementation, lists the reasoners that we considered, and tests these claims on all six benchmarks of Section~\ref{sec:examples} and on the scalable chain of Proposition~\ref{prop:depth}.

\subsection{The implementation}
\label{sec:impl}

All results of this paper are produced by one implementation, written in Python. It consists of the reasoner interface, the abducer \DDA{} (Section~\ref{sec:dda}), the planner \APT{} with \textsc{Generate}, \textsc{Enrich}, \textsc{Repair} and the temporal projection (Sections~\ref{sec:apt} and~\ref{sec:dynamics}), and a reader for benchmarks. A benchmark is an OWL~2 file in functional syntax; its rules are the annotations \texttt{pie:StateChanger} of the action classes, written as \texttt{CONSTRUCT} or \texttt{DELETE}/\texttt{INSERT} templates with a \texttt{WHERE} body that may contain \texttt{FILTER NOT EXISTS}. The reader splits the file into the terminology, the initial state and the rules. For adversarial synthesis (Section~\ref{sec:adversarial}) the implementation accepts patterns as input and generates them from the generic family and from the bounded miner of Proposition~\ref{prop:miner}, restricted to class literals over one variable; the extractors of Table~\ref{tab:patterns} and justifications are not part of it. On the taxes benchmark the generic family has four patterns and yields one story; the miner with $k=2$ finds six patterns, among them the disjointness pattern $\{\sn{TaxExempt}(x),\sn{TaxPayer}(x)\}$, and yields two stories, both by $[\ImportWheat(\sn{TraderJoe}),\ImportCar(\sn{TraderJoe})]$ (177 consistency checks).

Atoms are class assertions, negated class assertions and role assertions over named individuals. The vocabulary $\Omega$ is built from the atomic classes and the intersections, unions and existential restrictions that occur in the terminology. An action instance carries one individual per event role of its rule, and all rules of an action fire at the same instant on the same pre-state. Initial states contain no events. The reasoner is used only through consistency checks (B1--B3): entailment is the inconsistency of the knowledge base extended with the negated assertion, so B4 and B5 are obtained from the contract. Minimisation is the exact \textsc{Minimise}; the faster variant with a derivation record (B7) and incremental updates (B6) are not used for the results, so the implementation is the variant $\DDA'$ of Section~\ref{sec:price}. Parallelism and caching are those of Section~\ref{sec:parcache}.

\paragraph{Created individuals are not a reasoner ability.} Rules create individuals: the letter and the account of \sn{BankAccount}, the events and instants of a schedule, and the placeholders of \APT{}. The tools generate these as fresh names, outside the reasoner, and the reasoner receives ordinary assertions about them. It needs no ability to create individuals; B2 and B3 suffice, with the one proviso that $\op{Add}$ accepts assertions about names that do not occur in the terminology, which all five reasoners that we ran do. The benchmarks exercise this: 13 of the 27 distinct checks of the first \sn{BankAccount} goal and 4 of the 18 of the second contain an individual created by a rule, and all reasoners answer them like HermiT (Tables~\ref{tab:reasoners} and~\ref{tab:incr}).

\paragraph{Availability.} The implementation, all benchmark ontologies, the scripts that produce every table and trace of this paper, and a test suite that checks the results stated in the paper are available at \url{https://github.com/AmHsnSharafi/apt-tdkg} under the MIT licence.

\subsection{Reasoners considered}
\label{sec:considered}

A reasoner qualifies if it is sound and complete for the language in use (Section~\ref{sec:backend}), so we looked for reasoners for $\SROIQ$ and for incremental reasoners. Table~\ref{tab:considered} lists the candidates, what each offers for B6, and the outcome. Five were run: HermiT \cite{glimm2014hermit} 1.4 and JFact \cite{jfact,tsarkov2006fact} 4.0.2 (built from source) through the OWL API 4 \cite{horridge2011owlapi}, Pellet \cite{sirin2007pellet} 2.3.1 through the OWL API 3, Openllet \cite{openllet} 2.5.0, a maintained fork of Pellet, through the OWL API 5, and Konclude \cite{steigmiller2014konclude} v0.7.0-1139 (built from source) through OWLlink. Four candidates could not be used, for the reasons in the table. Openllet 2.6.5 is newer than the version that we ran, but its release archive lacks the libraries that it needs from Maven Central, which the environment cannot reach.

\begin{table}[t]
\centering\small
\caption{Reasoners considered. B6: what the reasoner offers for incremental updates. ``Run'' means that it was run on all benchmarks.}
\label{tab:considered}
\begin{tabularx}{\textwidth}{@{}P{2.7cm}P{2.6cm}P{3.1cm}X@{}}
\toprule
\textbf{Reasoner} & \textbf{Language} & \textbf{B6} & \textbf{Outcome}\\ \midrule
HermiT & $\SROIQ(\mathcal D)$ & none (new copy per check) & run; the reference for correctness\\
Konclude & $\SROIQ(\mathcal D)$ & \texttt{Tell}/\texttt{Retract} revisions & run; fresh mode correct, incremental mode not (Section~\ref{sec:eval-incr})\\
Pellet & $\SROIQ(\mathcal D)$ & incremental consistency checking, additions and removals & run; additions correct, removals not\\
Openllet & $\SROIQ(\mathcal D)$ & as Pellet & run; same behaviour as Pellet\\
JFact & $\SROIQ(\mathcal D)$ & none & run\\
Sequoia \cite{tena2019sequoia} & $\SROIQ$ without datatypes; consequence-based & none & not obtained: binary host unreachable; the source version does not support nominals and ABox assertions\\
ELK \cite{kazakov2013incremental} & $\mathcal{EL}$ & additions and deletions & cannot express the negated atoms that the tools assert\\
RDFox \cite{motik2019maintenance} & OWL~2~RL (datalog) & additions and deletions & incomplete for the Family ontology, which is outside OWL~2~RL; needs a licence\\
FaCT++ & $\SROIQ(\mathcal D)$ & planned, not implemented & sources not obtainable\\
\bottomrule
\end{tabularx}
\end{table}

\begin{table}[t]
\centering\small
\caption{Abilities (B1)--(B9) of Section~\ref{sec:backend} as provided by the five reasoners that were run, through the interfaces that the implementation uses. ``By B3'' means that the tools compute the ability from entailment checks, so that the reasoner provides nothing native. (B8) is provided by running separate reasoner instances or processes and was measured with HermiT and Pellet only. (B9) is not provided: the tools hold each state as a set of assertions and load it into a reasoner.}
\label{tab:abilities}
\setlength{\tabcolsep}{4pt}
\begin{tabular}{@{}lccccP{2.5cm}cc@{}}
\toprule
Reasoner & (B1)--(B3) & (B4) & (B5) & (B6) incremental updates & (B7) derivation record & (B8) & (B9)\\ \midrule
HermiT & yes & by B3 & by B3 & no & no & yes & no\\
JFact & yes & by B3 & by B3 & no & no & yes & no\\
Pellet & yes & by B3 & by B3 & additions reliable, removals not & no & yes & no\\
Openllet & yes & by B3 & by B3 & as Pellet & no & yes & no\\
Konclude & yes & by B3 & by B3 & revisions, not reliable & no & yes & no\\
\bottomrule
\end{tabular}
\end{table}

No reasoner that we ran reports what an update newly derives (B5) or stores derivations (B7) through these interfaces, so the tools run as $\DDA'$ with black-box supports on all five (Section~\ref{sec:backend}).

\subsection{Agreement and run time}
\label{sec:eval-times}

Table~\ref{tab:reasoners} runs the same tools on the benchmarks with each reasoner. Pellet and Openllet were run once with a new reasoner per check and once with incremental additions (Section~\ref{sec:eval-incr}). The sets of results and the numbers of distinct consistency checks are identical for all reasoners on all ten tasks: the results do not depend on the reasoner, as B1--B3 require. This includes the two \sn{BankAccount} goals, whose checks contain created individuals, and \sn{DerivedGate}, whose precondition holds only by entailment.

\begin{table}[t]
\centering\footnotesize\setlength{\tabcolsep}{3.2pt}
\caption{Run time in seconds of the same tool on five reasoner systems in seven set-ups (two processor cores; Konclude with a fresh knowledge base per check; ``incr.'': additions are incremental, see Table~\ref{tab:incr}). ``Scalable chain'' is the family of Proposition~\ref{prop:depth} with $n$ levels and $m$ additional individuals; it is not the Family benchmark. All set-ups perform the same numbers of distinct checks and return identical results.}
\label{tab:reasoners}
\begin{tabular}{@{}lrrrrrrrr@{}}
\toprule
\textbf{Task} & \textbf{checks} & \textbf{HermiT} & \textbf{Konclude} & \textbf{Pellet} & \textbf{Pellet incr.} & \textbf{Openllet} & \textbf{Openllet incr.} & \textbf{JFact}\\ \midrule
Door abduction & 150 & 1.9 & 1.0 & 1.0 & 1.1 & 3.4 & 4.8 & 1.8\\
Family relations, $L=3$ & 4133 & 12.5 & 26.4 & 8.0 & 9.6 & 24.4 & 37.0 & 10.3\\
Taxes story & 36 & 1.2 & 1.3 & 0.6 & 0.8 & 1.9 & 2.5 & 1.4\\
BankAccount, with card & 27 & 1.3 & 1.4 & 0.6 & 0.8 & 2.0 & 2.3 & 1.0\\
BankAccount, without card & 18 & 1.1 & 2.7 & 1.5 & 0.6 & 4.2 & 2.9 & 1.5\\
DerivedGate & 12 & 1.1 & 1.8 & 0.5 & 0.6 & 1.4 & 1.7 & 1.0\\
LabAccess, $R=3$ & 117 & 1.8 & 1.8 & 0.8 & 1.2 & 3.1 & 3.9 & 1.7\\
Scalable chain, $n=3$, $m=0$ & 95 & 1.7 & 1.6 & 0.7 & 1.0 & 2.6 & 3.6 & 1.8\\
Scalable chain, $n=3$, $m=50$ & 3845 & 16.1 & 33.1 & 8.7 & 11.9 & 27.3 & 44.1 & 14.3\\
Scalable chain, $n=3$, $m=200$ & 15\,095 & 93.7 & 240.3 & 57.2 & 79.9 & 138.4 & 222.9 & 92.4\\
\bottomrule
\end{tabular}
\end{table}

On the seven small tasks (up to 150 checks) every reasoner needs between 0.5 and 4.8 seconds, almost all of it process start and parsing, and differences below one second are within the variation between runs. On the three largest tasks (several thousand checks) Pellet is the fastest, JFact is slightly faster than HermiT, HermiT needs 1.6 to 1.9 times as long as Pellet, Openllet 1.5 to 2 times as long as HermiT, and Konclude about 2.1 to 2.6 times as long as HermiT, partly because each of its checks loads the terminology into a new knowledge base, a cost that we did not separate from the reasoning time. The largest Konclude task took between 240 and 878 seconds in six runs on the same machine, and we do not know the cause. These times contain process start, parsing and communication and say little about the reasoning time on large ontologies (Section~\ref{sec:eval-scope}).

\subsection{Incremental updates (B6)}
\label{sec:eval-incr}

Konclude, Pellet and Openllet can keep one knowledge base and process changes to it. We implemented each in the form in which a check sends only the difference to the previous check, recorded the distinct consistency checks that the tools make on seven benchmarks, and replayed them in the order of the first request through each set-up (Table~\ref{tab:incr}). All 4493 recorded checks are answered by HermiT; the table counts the answers of another set-up that differ.

\begin{table}[t]
\centering\footnotesize\setlength{\tabcolsep}{3.2pt}
\caption{Answers that differ from HermiT's when the recorded checks of a benchmark are replayed. ``Incr.'' applies additions incrementally and starts again from the terminology for removals; ``raw'' also removes assertions incrementally. Bank$_1$ and Bank$_2$: the goals with and without card; Family: the Family benchmark with two observations.}
\label{tab:incr}
\begin{tabular}{@{}lrrrrrrr@{}}
\toprule
\textbf{Reasoner and mode} & \textbf{Bank$_1$} & \textbf{Bank$_2$} & \textbf{Gate} & \textbf{Taxes} & \textbf{Lab} & \textbf{Door} & \textbf{Family}\\
\textit{recorded checks} & 27 & 18 & 12 & 36 & 117 & 150 & 4133\\ \midrule
Pellet, fresh; incr. & 0 & 0 & 0 & 0 & 0 & 0 & 0\\
Pellet, raw & 0 & 0 & 0 & 0 & 1 & 1 & 10\\
Openllet, fresh; incr. & 0 & 0 & 0 & 0 & 0 & 0 & 0\\
Openllet, raw & 0 & 0 & 0 & 0 & 1 & 1 & 10\\
Konclude, fresh & 0 & 0 & 0 & 0 & 0 & 0 & 0\\
Konclude, incremental revisions & 6 & 4 & 4 & 10 & 18 & 0 & 6\\
\bottomrule
\end{tabular}
\end{table}

\emph{Konclude.} The incremental mode (a \texttt{Tell} or \texttt{Retract} of the difference, processed as a new revision) was not reliable in the version tested: every wrong answer reports ``consistent'' where HermiT reports an inconsistency, for example after adding $\sn{EligiblePerson}(\sn{amir})$, $\sn{GetLetter}(e_1)$ and $\neg\sn{Human}(\sn{amir})$ to a terminology with $\sn{EligiblePerson}\sqsubseteq\sn{Human}$. Konclude's option for incremental rebuilding ended the server process in our test, and we cannot say whether the cause is the build, our use of OWLlink or Konclude itself.

\emph{Pellet and Openllet.} Incremental consistency checking \cite{halaschek2006dynamic} answered every replayed check correctly while assertions were only added. Removing assertions incrementally gave wrong answers (``inconsistent'' where HermiT reports consistency) in 1 of 117, 1 of 150 and 10 of 4133 checks, in the cases that we inspected when the removed assertion was a union or an intersection. Openllet 2.5.0 behaves identically, so the problem was not fixed there. The implementation therefore applies a check incrementally only when it adds assertions to the previous check of the same process and the previous answer was ``consistent''; otherwise it loads the terminology and the assertions anew. With this rule all 4493 replayed checks are answered correctly.

The consequence for performance is that the incremental mode is rarely used. The checks of \DDA{} and of the rule bodies alternate between related requests, such as $K\cup\{\neg d_1\}$ and then $K\cup\{\neg d_2\}$, which remove one assertion and add another. On Family 646 of 4133 checks (16\,\%), on LabAccess 21 of 117 (18\,\%) and on the scalable chain with 50 further individuals 9 of 3845 were answered incrementally, and Table~\ref{tab:par-meas} shows that the incremental mode is no faster than the fresh one. Since an unreliable answer would invalidate plans and explanations, all other results of this paper use the fresh modes.

\subsection{Parallelism (B8)}
\label{sec:eval-par}

Table~\ref{tab:par-meas} reports the run time of three complete tasks with one, two and four reasoner processes on a machine with two processor cores, for HermiT, for Pellet with a new reasoner per check and for Pellet with incremental additions, in which each process keeps its knowledge base. The results are identical in all runs, as Proposition~\ref{prop:parallel} requires, and the numbers of distinct consistency checks do not change.

\begin{table}[t]
\centering\small
\caption{Run time in seconds with $p$ reasoner processes on two processor cores; every row returns the same results. The last column gives the numbers of checks that were answered incrementally at $p=1,2,4$; all other checks started from the terminology. Task names as in Table~\ref{tab:reasoners}.}
\label{tab:par-meas}
\begin{tabular}{@{}llrrrr@{}}
\toprule
\textbf{Task} & \textbf{Set-up} & $p=1$ & $p=2$ & $p=4$ & \textbf{incremental checks}\\ \midrule
Family, 4133 checks & HermiT & 15.6 & 20.2 & 30.3 & ---\\
 & Pellet, fresh & 8.9 & 11.4 & 17.8 & ---\\
 & Pellet, incr. & 9.9 & 14.0 & 24.0 & 646 / 649 / 437\\ \midrule
LabAccess, 117 checks & HermiT & 2.2 & 3.5 & 4.8 & ---\\
 & Pellet, fresh & 1.1 & 1.8 & 2.9 & ---\\
 & Pellet, incr. & 1.4 & 1.9 & 2.8 & 21 / 23 / 22\\ \midrule
Scalable chain, $n=3$, $m=50$, 3845 checks & HermiT & 16.8 & 21.3 & 26.2 & ---\\
 & Pellet, fresh & 10.1 & 11.8 & 17.8 & ---\\
 & Pellet, incr. & 14.6 & 18.0 & 22.7 & 9 / 9 / 10\\
\bottomrule
\end{tabular}
\end{table}

The run time \emph{increases} with the number of processes in every row. Each check takes a few milliseconds, which is of the order of the cost of dispatching it and of starting a further Java process, and two cores cannot run more than two checks at a time. Parallelism pays only when a single call is expensive compared with its dispatch, as it is for large ontologies; this is the regime of the bound $W/p+S$ of Section~\ref{sec:parcache}, which we have not measured. Table~\ref{tab:par-tools} shows that every tool of Table~\ref{tab:par} does issue simultaneous reasoner calls: with four processes the largest number of calls in flight at one time is four for each tool, and the results, including the traces recorded by \textsc{Enrich}, equal those of the sequential run.

\begin{table}[t]
\centering\small
\caption{Each tool of Table~\ref{tab:par} alone, with one and with four reasoner processes (two processor cores). ``Peak'' is the largest number of reasoner calls in flight at the same time. All results equal those of the sequential run.}
\label{tab:par-tools}
\setlength{\tabcolsep}{4pt}
\begin{tabular}{@{}lrrrrr@{}}
\toprule
\textbf{Tool, benchmark} & \textbf{distinct checks} & $p=1$ & $p=4$ & \textbf{peak at $p=4$} & \textbf{same result}\\ \midrule
Projection, bank with 15 further individuals & 129 & 2.8 & 7.0 & 4 & yes\\
\textsc{Generate}, bank with 15 further individuals & 102 & 3.0 & 6.9 & 4 & yes\\
\textsc{Enrich}, LabAccess, $R=3$ & 81 & 2.0 & 4.8 & 4 & yes\\
\textsc{Adv}, taxes (miner and stories) & 93 & 2.0 & 4.6 & 4 & yes\\
\DDA, Family, two observations, $L=3$ & 4133 & 14.8 & 29.3 & 4 & yes\\
\bottomrule
\end{tabular}
\end{table}

Parallelism and incremental updates, which interact (Section~\ref{sec:parcache}), do not help each other here. With more processes each process also sees fewer related consecutive checks, and the number of incremental checks falls (Family: 646, 649 and 437 at $p=1,2,4$). The combination can pay off only with a reasoner whose incremental updates cover removals reliably, and on inputs where a check costs much more than the exchange with the process; we could not test that.

\subsection{Scope of the evaluation}
\label{sec:eval-scope}

\paragraph{Execution environment and repetitions.} All runs of this paper were made in one Linux container (kernel 6.18) with two processor cores and 7~GB of memory, with OpenJDK 21 and Python 3.13. The tools are deterministic and use no random numbers, so the numbers of reasoner calls and the results do not vary between runs. Each run time in this section is from a single run; run times are not averaged, no measure of their variation is given and no statistical test is applied, because they are used to compare orders of magnitude between set-ups and not to rank them. The parameters of the tools are the bounds named in the tables (plan length $d$, depth $R$, hypothesis size $L$), which were set to the values that the benchmark needs, and no parameter search preceded the final runs.

The conclusions hold for what was run. The reasoners are those that we could obtain, build and run in one environment, in the versions named, and Openllet is a fork of Pellet, so the two are not independent; a statement such as ``no tested reasoner has reliable incremental updates for our request pattern'' is not a statement about all reasoners. Each reasoner is used through one interface and one request pattern, and the incremental modes were driven by our code, so a wrong answer shows a failure of the combination, which may lie in the reasoner, its build or our use. HermiT is the reference for correctness, and agreement on the recorded checks is evidence and not a proof of soundness. The benchmarks are small and the machine has two cores, so the times do not bound what parallelism or incremental reasoning can achieve on larger inputs.

\section{Discussion and Conclusion}
\label{sec:discussion}

We presented three tools for TDKGs that need from a $\SROIQ$ reasoner only copies, assertion updates, a consistency check and an entailment check. \emph{Temporal projection} plays the events of the timeline and returns the last state; it is the basic operation of a TDKG and the exact test of plans. \DDA{} explains observations by the negated new consequences of their negation. \APT{} alternates between planning and abduction: the assumptions of a plan are explained by \DDA, each explanation becomes the goal of a sub-plan, and the recursion continues to a depth bound. We proved that actions given as rules with a query part and delete/insert templates (SPARQL \texttt{DELETE/INSERT} being the usual realisation) define decidable transitions and a decidable projection; that indirect effects are exactly the entailments of the updated state and integrity violations show up as inconsistencies; that \DDA{} is sound, terminates, finds every single-assertion explanation and explains several observations jointly; that \APT{} is sound, that its recursion terminates, that bounded planning is decidable, and that plans with the fewest actions can be found by an exact search; and that adversarial plans can target any contradiction the reasoner detects. If the reasoner records derivations, minimality is checked without reasoner calls; independent calls run in parallel and repeated ones are cached without changing any result.

\paragraph{Reasoner independence and logical position.} The tools run on any reasoner and any representation of the knowledge base that supplies the contract of Section~\ref{sec:backend}. Five $\SROIQ$ reasoners, through the OWL API and OWLlink, return the same results (Section~\ref{sec:eval}); the speed of the tools follows the abilities of the reasoner, in particular incremental updates and derivation records, and their correctness does not. We did not run the tools on a triple store or a graph database, so independence of the data structure rests on the contract and has been tested for the reasoner side only. Change is treated without modal operators, on one linear timeline and without possible worlds (Section~\ref{sec:logic-position}). Decidability of projection is then a corollary of the decidability of $\SROIQ$ and needs no separate proof for each fragment of an action logic. The linear timeline places the tools next to linear-time temporal logics and apart from the branching-time and possible-worlds approaches to action; branching is confined to the search of the planner.

\paragraph{Limitations.}
\begin{itemize}[leftmargin=1.2em]
\item \emph{Bounded plans.} Decidability is shown for bounded plan length. For unbounded length together with creation of new individuals we make no claim (Section~\ref{sec:apt}).
\item \emph{Complexity.} The problem inherits $2$\textsc{NExpTime}-hardness (Theorem~\ref{thm:apt-hard}) and lies in $3$\textsc{ExpTime} (Theorem~\ref{thm:apt-upper}); the exact class is open. In practice the tools are as fast as the reasoner; for fixed parameters their number of reasoner calls is polynomial, except for the factor $(bg)^R$ of the recursion (Proposition~\ref{prop:rec}).
\item \emph{Scope of the evaluation.} We give analytical bounds, call counts of the implementation (Section~\ref{sec:scal}) and run times of ten tasks on five reasoners, one external abduction solver (CATS) on problems of three examples and for class-assertion hypotheses only (Section~\ref{sec:cmp-systems}), but no run times on large ontologies; the reasoners, versions, interfaces and request pattern are those listed in Section~\ref{sec:eval}, and the statements on incremental updates and parallelism concern the tested combinations only. The performance benefit of incremental updates, derivation records and parallelism on large inputs is argued, not measured.
\item \emph{Incompleteness.} \DDA{} is complete only for its reachable space (Proposition~\ref{prop:limits}); the fast minimality mode prefers basic atoms over their weaker consequences (Section~\ref{sec:fastmin}); the heuristic generator of \APT{} is sound but not complete, and the engine makes no claim about plans with the fewest actions (Section~\ref{sec:quality}). The repair step recovers plans that need a deleting action, but only for conflicts met in the projection of a proposed candidate.
\item \emph{Side conditions.} Negative conditions in rule bodies, such as \texttt{FILTER NOT EXISTS}, are supported as side conditions that the projection evaluates; \textsc{Generate} ignores them and may therefore propose plans that are rejected.
\item \emph{Expressiveness of the logical choices.} Without modal operators the logic cannot state that a property holds after all or some executions of a plan; with one timeline and deterministic rules it cannot represent actions with several possible outcomes, verification over all futures, sensing or knowledge of an agent (Section~\ref{sec:logic-position}).
\item \emph{Comparison with CATS.} The comparison is between $\DDA'$ with the weakest contract and CATS with the model access of JFact; the derived counts for a reasoner with B5 and B7 are not measured (Section~\ref{sec:cmp-systems}, Table~\ref{tab:phases}).
\item \emph{Update semantics.} The update is syntactic: deletion removes explicit assertions only, and integrity constraints are reported rather than repaired (Proposition~\ref{prop:scope}). Rules fired at the same instant are evaluated on the same pre-state, and actions cannot change the clock.
\item \emph{Cost of pattern checks.} \textsc{Brave} tries the subsets of the final ABox; this is exponential in its size and is practical because the final ABox of a story is small or can be restricted to a conflict.
\end{itemize}

\paragraph{Future work.} A reasoner for this setting that provides all abilities of Table~\ref{tab:capabilities} together; running the tools on it is the experiment that remains. The reasoner should provide: exact consistency and entailment for $\SROIQ$ (B1--B3), native pattern evaluation (B4), reports of the consequences an update derives (B5), reliable incremental updates (B6), derivation records (B7) and concurrent calls (B8) and layered copies (B9), with the tools implemented on top of it, so that the effect of B4--B7 and B9 can be measured one by one and parallelism can be studied together with incremental updates; evaluation on large ontologies and with further reasoners; decidable fragments for unbounded plans; ranking of explanations and plans by preference or probability; causal rules for non-definitional constraints; derivation records that return all justifications.

\appendix
\renewcommand{\thesection}{\Alph{section}}
\makeatletter
\renewcommand{\@seccntformat}[1]{\csname the#1\endcsname.\quad}
\makeatother
\section{Plans Computed for the Benchmarks}
\label{app:traces}

Table~\ref{tab:traces} lists the plans that the implementation (Section~\ref{sec:impl}) computes from the benchmark ontology files; the script \texttt{benchmarks/plans.py} of the repository reproduces the table. Action instances are written with their participants. Times are not shown: the schedule assigns fresh instants. The plans for the taxes benchmark agree with the specification of Example~\ref{ex:tax} and the plans for LabAccess with Table~\ref{tab:lab} (Example~\ref{ex:lab}). The initial state is the set of explicit assertions of the ontology that are not about actions or instants. For LabAccess all four supported plans found by the recursion are given: the first two need no assumption, the last two assume $\sn{Degreed}(\sn{Amir})$, which no action supports.

\begin{table}[h]
\centering\footnotesize
\caption{Plans computed by the implementation. $\Gamma$ is the set of residual assumptions.}
\label{tab:traces}
\begin{tabularx}{\textwidth}{@{}P{2.3cm}P{4.6cm}P{3.0cm}X@{}}
\toprule
\textbf{Benchmark} & \textbf{Execution sequence} & \textbf{Residual assumptions $\Gamma$} & \textbf{Initial state}\\ \midrule
Bank account, goal with card & \texttt{GetLetter}(\textit{Amir}), \texttt{OpenAccount}(\textit{Amir}) & $\emptyset$ & \textit{EligiblePerson(Amir)}, \textit{ProofOfAddress(pr)}, \textit{has(Amir, pr)}\\
Bank account, goal without card & \texttt{OpenAccount}(\textit{Amir}) & $\emptyset$ & as above\\
Derived gate & \texttt{Act\_IssueBadge}(\textit{Amir}), \texttt{Act\_EnterZone}(\textit{Amir}) & $\emptyset$ & \textit{RegisteredPerson(Amir)}\\
Secure door & \texttt{Act\_TurnKey}(\textit{FrontDoorLock}), \texttt{Act\_OperateHandle}(\textit{dor}) & \textit{HingedStructure(dor)}, \textit{WoodenStructure(dor)}, \textit{isInstalledOn(FrontDoorLock, dor)} & \textit{MechanicalLock(FrontDoorLock)}\\
Taxes (story, complement pattern $\gamma_1$) & \texttt{ImportWheat\-Action}(\textit{TraderJoe}), \texttt{ImportCar\-Action}(\textit{TraderJoe}) & $\emptyset$ & \textit{Trader(TraderJoe)}\\
Lab access (the lab) & \texttt{Act\_Certify}(\textit{Amir}), \texttt{Act\_Train}(\textit{Amir}), \texttt{Act\_IssuePass}(\textit{Amir}), \texttt{Act\_EnterLab}(\textit{Amir}) & $\emptyset$ & \textit{Registered(Amir)}\\
 & \texttt{Act\_Certify}(\textit{Amir}), \texttt{Act\_IssuePass}(\textit{Amir}), \texttt{Act\_Train}(\textit{Amir}), \texttt{Act\_EnterLab}(\textit{Amir}) & $\emptyset$ & as above\\
 & \texttt{Act\_Train}(\textit{Amir}), \texttt{Act\_IssuePass}(\textit{Amir}), \texttt{Act\_EnterLab}(\textit{Amir}) & \textit{Degreed(Amir)} & as above\\
 & \texttt{Act\_IssuePass}(\textit{Amir}), \texttt{Act\_Train}(\textit{Amir}), \texttt{Act\_EnterLab}(\textit{Amir}) & \textit{Degreed(Amir)} & as above\\
Repair demonstration (Section~\ref{sec:repair}) & \texttt{Act\_Unlatch}(\textit{Hatch}), \texttt{Act\_OpenIt}(\textit{Hatch}) & $\emptyset$ & \textit{Closed(Hatch)}; $\sn{Closed}$ and $\sn{Open}$ are disjoint\\
\bottomrule
\end{tabularx}
\end{table}

\section{Formal Statements for the Analytical Comparison}
\label{app:formal}

This appendix states precisely the definitions and results on which Section~\ref{sec:comparison} relies: hitting-set abduction (B.1) and classical planning with axioms (B.2). Table~\ref{tab:prov} gives the source of each statement; ``ours'' means that the statement is proved here or in the paper and is not taken from the literature.

\begin{table}[t]
\centering\small
\caption{Sources of the definitions and results of this appendix.}
\label{tab:prov}
\begin{tabularx}{\textwidth}{@{}P{4.2cm}X@{}}
\toprule
\textbf{Statement} & \textbf{Source}\\ \midrule
Def.~\ref{def:mhs} ($\mathrm{MHS}_k$) & the abduction problem and its criteria (consistency, entailment, minimality) follow \cite{pukancova2020,koopmann2021signature}; hitting sets follow \cite{reiter1987theory,greiner1989correction}; the size bound $k$ is ours\\
Prop.~\ref{prop:mhs} (cost of $\mathrm{MHS}_k$) & ours\\
Def.~\ref{def:aaa}, Prop.~\ref{prop:aaa} (tree of AAA) & the scheme is that of AAA \cite{pukancova2020,mrozek2018}, with Reiter's hitting-set argument \cite{reiter1987theory}; the formulation, the proof and the bound are ours\\
Prop.~\ref{prop:incomparable}, \ref{prop:ddaprime} & ours\\
Def.~\ref{def:cp} (CP) & planning with axioms, after \cite{thiebaux2005defense}\\
Prop.~\ref{prop:cp-np} & standard (compare \cite{bylander1994computational}); the proof is given for completeness\\
Def.~\ref{def:fragment}, \ref{def:comp}, Prop.~\ref{prop:compile}, \ref{prop:coincide} & ours; a Horn fragment whose closure is computed by Datalog-style axioms\\
Def.~\ref{def:translation}, Thm.~\ref{thm:no-compile} & ours; use Theorem~\ref{thm:apt-hard} and the nondeterministic time hierarchy theorem \cite{cook1973hierarchy}\\
\bottomrule
\end{tabularx}
\end{table}

\subsection{Hitting-set abduction}

\begin{definition}[Hitting-set abducer $\mathrm{MHS}_k$]\label{def:mhs}
Let $\KB$ be a consistent KB, $O$ a finite set of observations, $\mathit{Abd}$ a finite set of assertions (the \emph{abducibles}) and $k\ge0$. $\mathrm{MHS}_k(\KB,O,\mathit{Abd})$ is the set of all subset-minimal $H\subseteq\mathit{Abd}$ with $|H|\le k$ such that $\KB\cup H$ is consistent and entails every $o\in O$. (A subset-minimal $H$ is one for which no proper subset of $H$ has these two properties.)
\end{definition}

$\mathrm{MHS}_k$ is the specification of what a hitting-set solver returns. Reiter's hitting-set tree \cite{reiter1987theory,greiner1989correction} computes it by exploring candidate sets and asking a reasoner about them. We analyse first the exhaustive form, which is the largest search that any implementation needs, then the tree of AAA, which uses models to explore fewer candidates; the solver CATS adds the merging of explanations to that tree (Section~\ref{sec:cmp-systems}).

\begin{proposition}[Cost of $\mathrm{MHS}_k$]\label{prop:mhs}
$\mathrm{MHS}_k$ can be computed with at most $\sum_{j\le k}\binom{|\mathit{Abd}|}{j}\le(|\mathit{Abd}|+1)^k$ candidate tests, each consisting of one update $\op{Add}$ and one consistency check and $|O|$ entailment checks on the updated copy. The bound does not depend on $\TB$.
\end{proposition}
\begin{proof}
Test every subset $H\subseteq\mathit{Abd}$ with $|H|\le k$, in order of increasing size, skipping the sets that contain an already accepted set; the accepted sets are exactly the subset-minimal ones. A test needs a copy of $\KB$, the update $\op{Add}(\KB',H)$, one consistency check and $|O|$ entailment checks. For the bound, a set of at most $k$ elements of $\mathit{Abd}$ corresponds to a $k$-tuple over $\mathit{Abd}\cup\{\bot\}$ (list its elements and pad with $\bot$); this map is injective, so there are at most $(|\mathit{Abd}|+1)^k$ such sets.
\end{proof}

\begin{definition}[Model-guided hitting-set tree, the scheme of AAA]\label{def:aaa}
Let $\KB$ be a consistent KB, $o$ an observation with $\KB\nentails o$, $\mathit{Abd}$ a finite set of assertions and $k\ge0$. For a model $M$ let $\mathrm{anti}(M)=\{a\in\mathit{Abd}\mid M\not\models a\}$, the abducibles that are false in $M$. The tree has as nodes candidate sets $c\subseteq\mathit{Abd}$ with $|c|\le k$; its root is $\emptyset$, and nodes are processed by increasing size. To process $c$, add $c$ to a copy of $\KB\cup\{\neg o\}$ and check consistency.
\begin{enumerate}[label=(\arabic*),nosep]
\item If the check fails, $c$ is a \emph{solution}. It is returned if $\KB\cup c$ is consistent and no returned set is a subset of $c$; a solution has no children.
\item Otherwise the reasoner provides a model $M$ of $\KB\cup\{\neg o\}\cup c$, and if $|c|<k$ the children are the sets $c\cup\{a\}$ with $a\in\mathrm{anti}(M)$; a child is omitted if it was generated before or contains a returned set (an implementation may also omit children that are inconsistent with $\KB$).
\end{enumerate}
\end{definition}

\begin{proposition}[Correctness and cost of the tree]\label{prop:aaa}
For $\KB\nentails o$ the tree of Definition~\ref{def:aaa} returns exactly $\mathrm{MHS}_k(\KB,\{o\},\mathit{Abd})$. It has at most $(|\mathit{Abd}|+1)^k$ nodes; processing a node needs one update and at most one consistency check (none if a stored model of $\KB\cup\{\neg o\}$ satisfies $c$, which is Reiter's model reuse), and reading one model, and each solution needs one further consistency check of $\KB\cup c$. AAA additionally keeps only relevant solutions ($c\not\models o$), a filter on the returned sets that does not change the cost. The reasoner must provide a model after a successful check, an ability outside B1--B3.
\end{proposition}
\begin{proof}
\emph{Soundness.} A returned $c$ has $\KB\cup\{\neg o\}\cup c$ inconsistent, so $\KB\cup c\models o$, and $\KB\cup c$ is consistent. Suppose a proper subset of $c$ were an explanation; then so is a subset-minimal one $E'\subsetneq c$ with $|E'|<k$. By the completeness argument below, $E'$ is a node of smaller size, processed earlier and returned; then $c$ contains a returned set and is not returned. \emph{Completeness.} Let $E\in\mathrm{MHS}_k(\KB,\{o\},\mathit{Abd})$. We show by induction on $i\le|E|$ that some node $c\subseteq E$ with $|c|=i$ is generated. For $i=0$ it is the root. Let $c\subseteq E$ be a node with $|c|<|E|$. If $\KB\cup\{\neg o\}\cup c$ were inconsistent, then $\KB\cup c\models o$ and $\KB\cup c$ is consistent (it is contained in $\KB\cup E$); by minimality of $E$ we get $c=E$, a contradiction. So the check succeeds and the reasoner gives a model $M\models\KB\cup\{\neg o\}\cup c$. Since $\KB\cup E\models o$ and $M\models\KB\cup\{\neg o\}$, we have $M\not\models E$, so some $e\in E$ has $M\not\models e$, that is $e\in\mathrm{anti}(M)$, and $e\notin c$ because $M\models c$. The child $c\cup\{e\}\subseteq E$ is generated: it contains no returned set $E'$ because $E'\subseteq E$ would give $E'=E$ by minimality, and $E$ is returned only at its own node. At $i=|E|$ the node $E$ is a solution, and it is returned. \emph{Cost.} Nodes are distinct subsets of $\mathit{Abd}$ of at most $k$ elements, whose number is bounded as in Proposition~\ref{prop:mhs}.
\end{proof}

\begin{proposition}[Relation to \DDA]\label{prop:incomparable}
(a) Let $\mathit{Abd}\subseteq\Omega$, where $\Omega$ is the vocabulary of Section~\ref{sec:dda} with the observation $o$ and $\Omega$ closed under negation. Every $H=\{a\}\in\mathrm{MHS}_1(\KB,\{o\},\mathit{Abd})$ with $a\ne o$ is an explanation returned by \DDA{} with $L=1$. (b) For explanations of two or more atoms neither method contains the other: the explanation $\{R(a,b),B(b)\}$ of Proposition~\ref{prop:limits}(a) is returned by $\mathrm{MHS}_2$ and missed by \DDA, and the explanation $\{(\exists R.B)(a)\}$ of Proposition~\ref{prop:limits}(b) is returned by \DDA{} and cannot be expressed over atomic abducibles.
\end{proposition}
\begin{proof}
(a) Let $a\in\Omega$ with $\KB\cup\{a\}$ consistent and entailing $o$, and $\KB\not\models o$ (otherwise the explanation is empty). Then $\KB\cup\{\neg o\}\models\neg a$, and $\KB\not\models\neg a$ because $\KB\cup\{a\}$ is consistent. So $\neg a$ is a new consequence of the negated observation, $\neg\neg a=a\in\textsc{Suggest}(\KB,o)$, and $\{a\}$ is a successor of $H_0=\{o\}$ after one round. It satisfies (E), (C) and (M); (R) holds since $a\ne o$. Hence $\textsc{Minimise}$ returns it. (b) is Proposition~\ref{prop:limits}.
\end{proof}

\begin{proposition}[\DDA{} and the tree of AAA]\label{prop:aaa-dda}
Let $\KB\nentails o$ and let $\mathit{Abd}=\Omega$ be closed under negation. (a) For every model $M$ of $\KB\cup\{\neg o\}$, $\Sugg(o)\subseteq\mathrm{anti}(M)$; hence every suggestion is a child of the root of the tree of Definition~\ref{def:aaa}. (b) Every suggestion $\{a\}$ is an explanation (Lemma~\ref{lem:sugg}) without a consistency check, whereas a child of the root that is not a suggestion needs one and may fail to be a solution.
\end{proposition}
\begin{proof}
(a) Let $a=\neg\delta\in\Sugg(o)$. Then $\KB\cup\{\neg o\}\models\delta$, so $M\models\delta$, so $M\not\models a$ and $a\in\mathrm{anti}(M)$. (b) The first part is Lemma~\ref{lem:sugg}. For the second, let $\TB=\{A\sqsubseteq C\}$, $\AB=\emptyset$, $o=C(x)$ and $\mathit{Abd}=\{A(x),\neg A(x),B(x),\neg B(x),C(x),\neg C(x)\}$. A model $M$ with $M\models\neg C(x),\neg A(x),B(x)$ has $\mathrm{anti}(M)=\{A(x),\neg B(x),C(x)\}$, and $\Sugg(o)=\{A(x),C(x)\}$. The child $\{\neg B(x)\}$ is not a suggestion, and $\KB\cup\{\neg C(x),\neg B(x)\}$ is consistent, so it is not a solution.
\end{proof}

\begin{proposition}[Operations of the relaxed variant $\DDA'$]\label{prop:ddaprime}
$\DDA'$ computes $\textsc{Suggest}$ by entailment tests on $\Omega$ (without ability B5) and minimality by Algorithm~\ref{alg:minimise} (without ability B7). It returns the same explanations as \DDA. For $s$ distinct suggestion computations it needs $s$ updates and at most $(s+1)\,|\Omega|$ entailment checks for the suggestions, and at most $N_H\,2^{h}(n+1)$ checks for minimality, where $N_H$ is the number of reached hypotheses, $h$ the largest hypothesis and $n=|O'|$.
\end{proposition}
\begin{proof}
By Algorithm~\ref{alg:suggest}, the set $\Delta$ of new consequences is $\{\delta\in\Omega\mid\op{Entails}(\KB',\delta)\}\setminus\{\delta\in\Omega\mid\op{Entails}(\KB,\delta)\}$. Computing it by $|\Omega|$ tests on the updated copy $\KB'$ and $|\Omega|$ tests on $\KB$ gives, by (B3), the set that B5 would report. The tests on $\KB$ are shared by all $s$ computations, which gives $s|\Omega|+|\Omega|$ checks. The bound for minimality is that of Algorithm~\ref{alg:minimise}: at most $2^{h}$ subsets per hypothesis, each tested by an update and $n$ entailment checks, plus one consistency check.
\end{proof}

\subsection{Classical planning with axioms}

\begin{definition}[Classical planning problem with axioms]\label{def:cp}
A \emph{planning problem with axioms} (CP) is a tuple $\langle\mathcal O,\mathcal P_b,\mathcal P_d,\mathcal X,s_0,G,\mathcal A\rangle$ where
\begin{itemize}[leftmargin=*,itemsep=1pt,topsep=2pt]
\item $\mathcal O$ is a finite set of objects, and $\mathcal P_b$, $\mathcal P_d$ are disjoint finite sets of \emph{basic} and \emph{derived} predicates, each with an arity;
\item $\mathcal X$ is a finite set of \emph{axioms} $q(\bar x)\leftarrow\beta_1,\dots,\beta_m$ where $q\in\mathcal P_d$ and the $\beta_i$ are atoms over $\mathcal P_b\cup\mathcal P_d$ whose variables include those of $q(\bar x)$;
\item a \emph{state} is a finite set of ground atoms over $\mathcal P_b$ and objects of $\mathcal O$; $s_0$ is the initial state; the \emph{closure} $\mathrm{cl}(s)$ is the least set of ground atoms that contains $s$ and is closed under the axioms (it is computed in polynomial time by applying the axioms until nothing new is derived);
\item $G$ is a finite set of ground atoms (the goal);
\item $\mathcal A$ is a finite set of \emph{actions}; an action $a$ has parameters $\bar y$, a \emph{precondition} (a finite set of atoms over $\mathcal P_b\cup\mathcal P_d$ with variables among $\bar y$), and \emph{add} and \emph{delete} lists (atoms over $\mathcal P_b$ with variables among $\bar y$).
\end{itemize}
An instance $a(\bar o)$ of an action (parameters replaced by objects) is \emph{applicable} in $s$ if its instantiated precondition is a subset of $\mathrm{cl}(s)$; applying it gives $(s\setminus\mathit{Del})\cup\mathit{Add}$ (deletions first). A \emph{plan of length $m$} is a sequence of $m$ action instances, each applicable in the state produced by its predecessors, whose final state $s$ satisfies $G\subseteq\mathrm{cl}(s)$. The \emph{closed-world assumption} holds: an atom that is not in $\mathrm{cl}(s)$ is false.
\end{definition}

\begin{proposition}[Plan existence for CP]\label{prop:cp-np}
Deciding whether a CP has a plan of length at most $d$, with $d$ given in unary, is in \textsc{NP}.
\end{proposition}
\begin{proof}
Guess at most $d$ action instances; there are polynomially many instances of each action. Simulate them: applicability, application and the final goal test need the closure of at most $d{+}1$ states, each computed in polynomial time.
\end{proof}

\begin{definition}[The compilable fragment]\label{def:fragment}
Let $\mathcal V_{ev}$ be the \emph{event vocabulary} of a TDKG: its action classes, $\sn{hasTime}$, $\sn{before}$, $\sn{TimeInstant}$ and the roles that relate events to their participants. A \emph{planning instance in the fragment} $\mathcal F$ consists of a KB $\KB_0=\langle\TB,\AB_0\rangle$, rules $\RS$, a goal $G$ and a bound $d$ such that
\begin{enumerate}[label=(F\arabic*),leftmargin=*,itemsep=1pt,topsep=2pt]
\item every axiom of $\TB$ has one of the forms $A\sqsubseteq B$, $A_1\sqcap A_2\sqsubseteq B$, $\exists R.\top\sqsubseteq C$, $\top\sqsubseteq\forall R.C$, $R\sqsubseteq S$, with $A,B,C,A_i$ class names and $R,S$ role names; and $\AB_0$ consists of atomic assertions and contains no events;
\item if a class or role of $\mathcal V_{ev}$ occurs in the premise of an axiom, then the conclusion of that axiom is in $\mathcal V_{ev}$ as well (the terminology never derives facts about the world from events);
\item every action class has exactly one rule; every rule has a body of atomic class and role atoms, only atomic deletion and insertion templates, no new variables and no side conditions;
\item the goal $G$ is a finite set of ground atomic assertions.
\end{enumerate}
\end{definition}

\begin{definition}[Translation $\mathrm{Comp}$]\label{def:comp}
For an instance of $\mathcal F$ let $\mathcal O$ be the set of individuals of $\AB_0$, of the rules and of $G$. Let $\mathcal P_b$ contain the class names (unary) and role names (binary) that are not in $\mathcal V_{ev}$, and let $\mathcal P_d$ contain a derived copy $X^{*}$ of every $X\in\mathcal P_b$. The axioms are
\[
X^{*}(\bar x)\leftarrow X(\bar x)\ \ \text{for every }X\in\mathcal P_b,
\]
and, for every axiom of $\TB$ whose conclusion is not in $\mathcal V_{ev}$,
\[
\begin{gathered}B^{*}(x)\leftarrow A^{*}(x);\quad B^{*}(x)\leftarrow A_1^{*}(x),A_2^{*}(x);\quad C^{*}(x)\leftarrow R^{*}(x,y);\\ C^{*}(y)\leftarrow R^{*}(x,y);\quad S^{*}(x,y)\leftarrow R^{*}(x,y)\end{gathered}
\]
for $A\sqsubseteq B$, $A_1\sqcap A_2\sqsubseteq B$, $\exists R.\top\sqsubseteq C$, $\top\sqsubseteq\forall R.C$ and $R\sqsubseteq S$, respectively. Each rule $r$ becomes an action whose parameters are the participants of $r$, whose precondition is the set of non-event atoms of its body with every predicate $X$ replaced by $X^{*}$, and whose delete and add lists are $D_r$ and $I_r$ (over $\mathcal P_b$). The initial state is $\AB_0$, and the goal is $G$ with every predicate $X$ replaced by $X^{*}$.
\end{definition}

\begin{proposition}[Closure equals entailment]\label{prop:compile}
Let $\TB$ satisfy (F1) and let $\AB$ be a set of atomic assertions over $\mathcal P_b$. Then $\langle\TB,\AB\rangle$ is consistent, and for every ground atomic assertion $X(\bar a)$ over named individuals, $\langle\TB,\AB\rangle\models X(\bar a)$ iff $X^{*}(\bar a)\in\mathrm{cl}(\AB)$ for the axioms of Definition~\ref{def:comp}.
\end{proposition}
\begin{proof}
All axioms of (F1) are Horn clauses and all assertions are atoms. Let $\mathcal I$ be the interpretation whose domain is the set of named individuals and in which $X$ has the extension $\{\bar a\mid X^{*}(\bar a)\in\mathrm{cl}(\AB)\}$. It satisfies every assertion of $\AB$ (by the first kind of axiom) and every axiom of $\TB$ (each is satisfied because the corresponding Datalog axiom is closed in $\mathrm{cl}(\AB)$). Hence the KB is consistent, and every entailed atom holds in $\mathcal I$, that is, belongs to $\mathrm{cl}(\AB)$. Conversely, each Datalog axiom is sound for its source axiom, so by induction on the length of the derivation every atom $X^{*}(\bar a)\in\mathrm{cl}(\AB)$ corresponds to an entailed assertion.
\end{proof}

\begin{proposition}[\APT{} and CP coincide on $\mathcal F$]\label{prop:coincide}
Let an instance of $\mathcal F$ be given, and let $\pi=a_1\dots a_m$ with $m\le d$ be a sequence of action instances whose participants are individuals of $\mathcal O$. Then $(\pi;\emptyset)$ is a valid plan (Definition~\ref{def:valid}) iff $\pi$ is a plan of length $m$ of $\mathrm{Comp}(\text{instance})$.
\end{proposition}
\begin{proof}
Let $\KB^\pi$ be the scheduled KB, with the event $e_i$ of $a_i$ at a fresh instant $t_i$, and let $s_i$ be the state of the CP after $a_1\dots a_i$. We show by induction on $i$ that the non-event part of the explicit ABox of the projected state $\KB_i^\pi$ equals $s_i$, and that its event part contains only atoms of $\mathcal V_{ev}$. For $i=0$ this is (F1). Because of (F2) the entailed atoms outside $\mathcal V_{ev}$ of $\KB_{i}^\pi$ are those of $\mathrm{cl}(s_i)$ (Proposition~\ref{prop:compile}, applied to the part outside $\mathcal V_{ev}$; the event part cannot contribute). Every state is consistent, so (V1) holds. At instant $t_i$ only $e_i$ is present, so the only rule that can fire is the unique rule $r_i$ of the class of $e_i$ (F3), and it fires with participants $\beta_i$ iff its non-event body atoms are entailed, that is, iff the starred atoms belong to $\mathrm{cl}(s_{i-1})$: this is applicability of $a_i$, and gives (V2). The step then deletes $D_{r_i}$ and inserts $I_{r_i}$, as the CP does (no new variables by (F3)); this gives the induction step. Finally $\op{Select}(\KB_m^\pi,G)\ne\emptyset$ iff $G^{*}\subseteq\mathrm{cl}(s_m)$ by Proposition~\ref{prop:compile}, which is (V3). Conversely the same equivalences show that a CP plan is valid.
\end{proof}

\begin{definition}[Compact translation]\label{def:translation}
A \emph{translation} is a function $f$ that maps every planning instance $I$ (a KB $\KB_0$, rules, a goal and the bound $d$) to a CP with horizon $d$ such that $I$ has a valid plan in $\mathrm{Cand}_d$ with $\Gamma=\emptyset$ iff $f(I)$ has a plan of length at most $d$. It is \emph{compact} if $f$ is computable in time $2^{p(|I|)}$ for some polynomial $p$.
\end{definition}

\begin{theorem}[No compact translation]\label{thm:no-compile}
There is no compact translation.
\end{theorem}
\begin{proof}
Suppose $f$ is a compact translation. Let $L\in2\textsc{NExpTime}$. By Theorem~\ref{thm:apt-hard} there is a polynomial-time reduction mapping each input $x$ to a planning instance $I_x$ with $d=1$ such that $x\in L$ iff $I_x$ has a valid plan. The CP $f(I_x)$ has size at most $2^{p(|I_x|)}$, which is $2^{q(|x|)}$ for a polynomial $q$, and $x\in L$ iff $f(I_x)$ has a plan of length at most $1$. By Proposition~\ref{prop:cp-np} this is decidable in nondeterministic time polynomial in the size of $f(I_x)$, hence in $\textsc{NTIME}(2^{q'(|x|)})$ including the time to compute $f(I_x)$. Therefore $2\textsc{NExpTime}\subseteq\textsc{NExpTime}$, which contradicts the nondeterministic time hierarchy theorem \cite{cook1973hierarchy}.
\end{proof}

\subsection{What a classical planner needs for each benchmark}

Table~\ref{tab:cp-bench} lists, for each benchmark of Section~\ref{sec:examples}, what a CP encoding needs and how \APT{} treats the same feature. ``In $\mathcal F$'' means that the benchmark satisfies Definition~\ref{def:fragment}.

\begin{table}[t]
\centering\small
\caption{What a classical planner with axioms (CP) needs for each benchmark, compared with \APT.}
\label{tab:cp-bench}
\begin{tabularx}{\textwidth}{@{}P{2.1cm}P{5.0cm}X@{}}
\toprule
\textbf{Benchmark} & \textbf{CP encoding} & \textbf{In \APT}\\ \midrule
DerivedGate & in $\mathcal F$; the subsumption $\sn{BadgeHolder}\sqsubseteq\sn{AuthorizedPerson}$ is written a second time as an axiom & axiom read from the ontology; the precondition of \EnterZone{} holds by entailment\\
LabAccess & in $\mathcal F$; the definition and two subsumptions become axioms, and the derived preconditions are evaluated directly & recursion between abduction and planning resolves the derived preconditions; the same plans (Prop.~\ref{prop:coincide})\\
BankAccount & fixed object pool for letters and accounts; closed-world negation for rule~1 & creation of individuals; the side condition is evaluated by projection\\
SecureDoor & unknown facts must be enumerated as optional ``assume'' actions & assumptions $\Gamma$ are reported and explained\\
Taxes & hand-written conflict predicate; no explanation of $\neg\sn{TaxPayer}$; fixed agent & contradiction pattern from the vocabulary; abduction; placeholder agent\\
Family & not a planning problem & abduction only (\DDA)\\
\bottomrule
\end{tabularx}
\end{table}

\end{document}